\documentclass[10pt]{article}

\RequirePackage{amsthm}
\let\amsproof\proof
\let\endamsproof\endproof

\usepackage[authoryear]{fardlab}

\let\proof\amsproof
\let\endproof\endamsproof

\usepackage{needspace}
\preto\section{\needspace{5\baselineskip}}
\preto\subsection{\needspace{4\baselineskip}}

\usepackage{colortbl}
\usepackage{wrapfig}
\usepackage{listings}
\usepackage{cleveref}
\usepackage{pifont}
\usepackage{etoc}

\usepackage{amsmath,amsfonts,bm}

\def\eqref#1{equation~\ref{#1}}

\def\1{\bm{1}}

\DeclareMathAlphabet{\mathsfit}{\encodingdefault}{\sfdefault}{m}{sl}
\SetMathAlphabet{\mathsfit}{bold}{\encodingdefault}{\sfdefault}{bx}{n}

\newcommand{\KL}{D_{\mathrm{KL}}}

\makeatletter
\@for\fard@env:=theorem,proposition,lemma,corollary,definition,assumption,remark\do{\expandafter\let\csname\fard@env\endcsname\relax\expandafter\let\csname end\fard@env\endcsname\relax\expandafter\let\csname c@\fard@env\endcsname\relax}
\makeatother
\theoremstyle{plain}

\newtheorem{proposition}{Proposition}[section]
\newtheorem{lemma}{Lemma}[section]
\newtheorem{corollary}{Corollary}[section]
\theoremstyle{definition}
\newtheorem{definition}{Definition}[section]
\newtheorem{assumption}{Assumption}[section]
\theoremstyle{remark}

\theoremstyle{plain}

\crefname{theorem}{Theorem}{Theorems}
\Crefname{theorem}{Theorem}{Theorems}
\crefname{proposition}{Proposition}{Propositions}
\Crefname{proposition}{Proposition}{Propositions}
\crefname{lemma}{Lemma}{Lemmas}
\Crefname{lemma}{Lemma}{Lemmas}
\crefname{corollary}{Corollary}{Corollaries}
\Crefname{corollary}{Corollary}{Corollaries}
\crefname{definition}{Definition}{Definitions}
\Crefname{definition}{Definition}{Definitions}
\crefname{assumption}{Assumption}{Assumptions}
\Crefname{assumption}{Assumption}{Assumptions}
\crefname{remark}{Remark}{Remarks}
\Crefname{remark}{Remark}{Remarks}
\crefname{figure}{Figure}{Figures}
\crefname{table}{Table}{Tables}
\crefname{section}{Section}{Sections}
\crefname{subsection}{Section}{Sections}
\crefname{subsubsection}{Section}{Sections}
\crefname{appendix}{Appendix}{Appendices}
\crefname{subappendix}{Appendix}{Appendices}
\crefname{subsubappendix}{Appendix}{Appendices}

\definecolor{tolGrey}{HTML}{DDDDDD}
\definecolor{tolBlue}{HTML}{77AADD}
\definecolor{tolRose}{HTML}{CC6677}
\definecolor{tolPurple}{HTML}{AA4499}

\tcbset{thmbox/.style={
    enhanced, arc=2pt, boxrule=0.4pt, colframe=black,
    left=6pt, right=6pt, top=4pt, bottom=4pt,
    before skip=8pt, after skip=8pt}}

\tcolorboxenvironment{definition}{thmbox, colback=tolGrey!35}
\tcolorboxenvironment{theorem}{thmbox, colback=tolBlue!22}
\tcolorboxenvironment{proposition}{thmbox, colback=tolBlue!22}
\tcolorboxenvironment{lemma}{thmbox, colback=tolBlue!22}
\tcolorboxenvironment{corollary}{thmbox, colback=tolBlue!22}
\tcolorboxenvironment{assumption}{thmbox, colback=tolRose!18}
\tcolorboxenvironment{remark}{thmbox, colback=tolPurple!14}
\tcbset{casebox/.style={
    enhanced, arc=2pt, boxrule=0.4pt, colframe=black!75,
    colback=tolGrey!25, colbacktitle=tolBlue!45, coltitle=black,
    fonttitle=\bfseries\small, width=\textwidth, top=1mm, bottom=1mm,
    attach boxed title to top left={xshift=3mm, yshift*=-\tcboxedtitleheight/2},
    boxed title style={arc=1pt, boxrule=0.4pt, colframe=black!75},
    before upper={\vspace{1mm}}}}

\newcommand{\TV}{d_{\mathrm{TV}}}

\definecolor{improvebg}{RGB}{198,239,206}
\definecolor{worsebg}{RGB}{242,201,201}
\definecolor{improvetxt}{RGB}{34,110,54}
\definecolor{worsetxt}{RGB}{160,40,40}
\definecolor{greytext}{gray}{0.55}
\definecolor{restcoral}{HTML}{E76F51}
\newcommand{\basetxt}{\textcolor{greytext}{Base}}
\newcommand{\gaincell}[1]{\cellcolor{improvebg}{\textbf{#1}}}
\newcommand{\losscell}[1]{#1}
\newcommand{\savecell}[1]{\gaincell{#1}}

\title{Principled Thoughts for Latent Recursive LLM Systems}
\shorttitle{Principled Thoughts for Latent Recursive LLM Systems}

\author{Fahd Seddik\textsuperscript{1}\enspace Fatemeh Fard\textsuperscript{1}}
\affil{1}{FARD Lab, University of British Columbia, Okanagan, Canada}
\makeatletter
\def\fard@authornotes{Correspondence: \href{mailto:fahd.seddik@ubc.ca}{fahd.seddik@ubc.ca}, \href{mailto:fatemeh.fard@ubc.ca}{fatemeh.fard@ubc.ca}}
\makeatother

\projecturl[fard-lab.github.io/REST]{https://fard-lab.github.io/REST}

\begin{document}
\etocdepthtag.toc{main}
\maketitle

\begin{abstract}
Large language models can reason in continuous space instead of decoded text, by recurring on their own hidden states or by passing those states between agents, while training supervises only the Cross-Entropy (CE) of the final decoded answer and does not constrain the thought. Theoretical and empirical analyses establish and confirm four failures of CE-only training that lead to a lower probability of the correct answer such as collapsing thoughts across distinct questions and retaining irrelevant information. We introduce \textbf{REST} (REpresentation-Supervised Thoughts), a training objective that turns four properties of a valid thought representation (causality, minimality, separability, and stability) into differentiable losses added to CE.
We instantiate it in latent single-agent and multi-agent systems, without architectural changes or added parameters at inference. Across 7 benchmarks spanning mathematics, science, medicine, and code generation, with the same training data, compute, and latent budget, REST increases accuracy over CE-only training across agent settings and model sizes by up to 7.5 percentage points and convergence on a final answer by 30\%.
Furthermore, REST thoughts encode more of what is required to achieve the correct answer, and decoding them better recovers the intended output of the agent, which makes latent communication easier to interpret.

\end{abstract}

\begin{figure}[h!]
\vspace{-1.0\baselineskip}
\centering
\includegraphics[width=\linewidth]{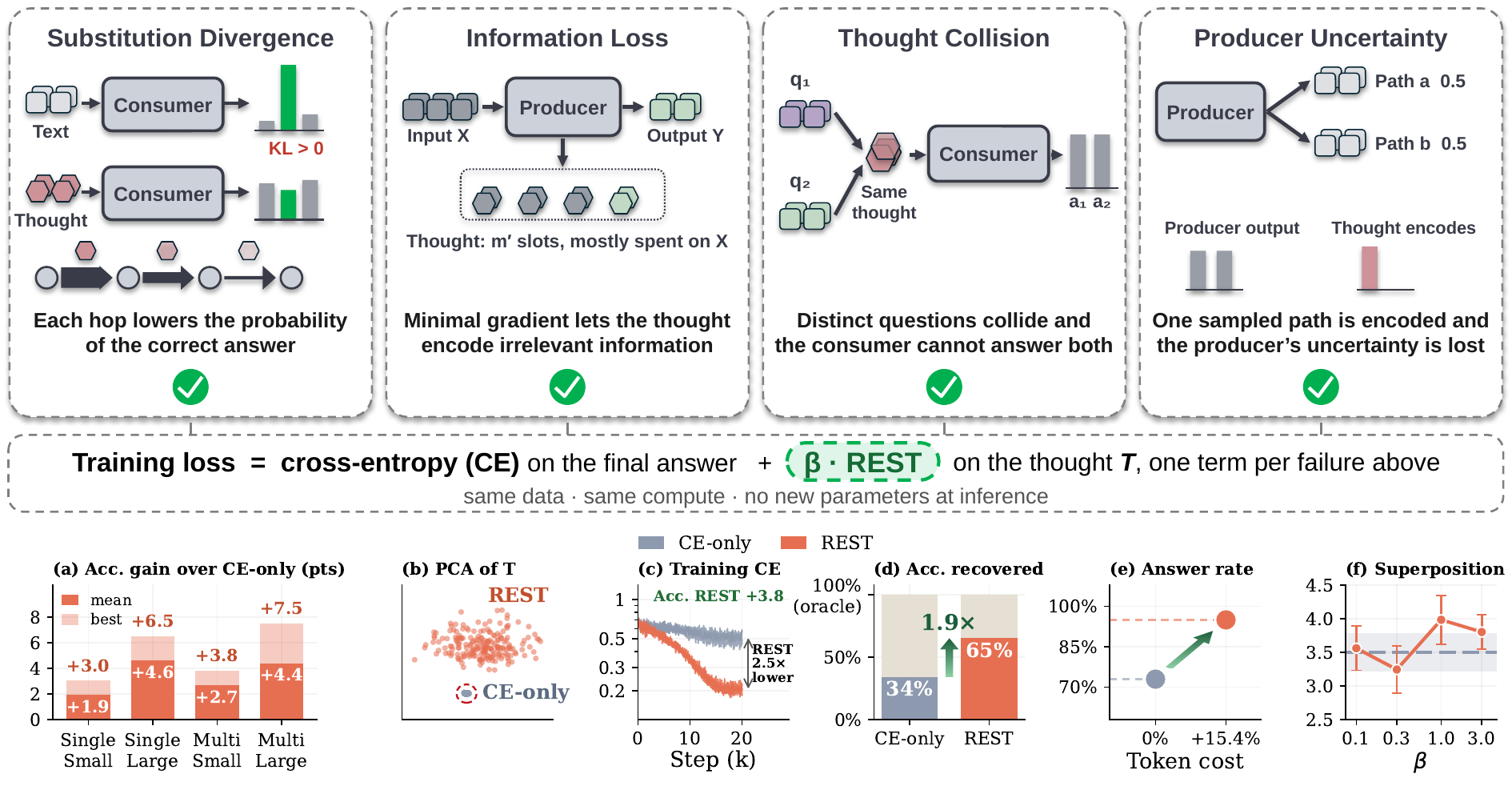}
\caption{\textbf{REST.} The top row shows the four failures of Cross-Entropy (CE) only training of the thought. REST adds a $\beta$-weighted loss on to the CE, with one term per failure, under the same data and compute. (a) Accuracy gain for single- and multi-agent settings across Small (1--2B) and Large (3--4B) models. (b)--(f) REST thoughts spread apart instead of collapsing, reach a lower training CE and a higher accuracy, recover more accuracy when they replace the oracle text, reach a final answer more often, and keep their superposition.}
\label{fig:teaser}
\end{figure}

\section{Introduction}
\label{sec:introduction}

Chain-of-Thought (CoT) prompting enables Large Language Models (LLMs) to solve complex problems by generating each intermediate reasoning step as explicit text \citep{wei2022cotprompting}.
However, expressing every step through tokens of a discrete vocabulary confines reasoning to paths that can be written in words \citep{li2025implicitsurvey}.
The resulting reasoning traces are also long, which raises inference cost and leads models to allocate excessive computation to simple problems \citep{chen2025overthinking}.

To resolve these limitations, a growing line of work reasons directly in the continuous latent space of LLMs rather than through decoded text \citep{hao2026coconut,zhang2025softthinking,butt2026soft,shen2025selfdistillcot,wei2025simcot,sheshanarayana2026thinking}. Within a single model, this is realized by training the model to consume its own hidden states as the next reasoning step instead of a token embedding \citep{hao2026coconut,shen2025selfdistillcot,li2026persistentlatent}. From one model to another, hidden states or KV caches have instead been used for inter-agent communication \citep{liu2024cacheaugmentation,zheng2025thoughtcommunication,fu2026cachetocache}.
Recent systems unify both, with multiple agents reasoning and communicating in latent space \citep{zou2026latentcollaboration} and repeating this collaboration over several rounds of recursion \citep{zou2026recursivemultiagentsystems}. The main objective is usually Cross-Entropy (CE) of the final decoded answer. As these systems scale in agents and rounds, each later agent builds on the thoughts it receives, and error compounds through the remainder of the system. Existing representation-level supervision targets the text the thought replaces, which CODI does by matching the hidden states that text induces \citep{shen2025selfdistillcot} or SIM-COT by decoding that text back from the thought \citep{wei2025simcot}. However, supervising only on the target text does not dictate the relation between two thoughts nor the distribution of outputs the model can generate. Both requirements are among the four properties that define a valid thought representation \citep{seddik2026formalizinglatentthoughtsaxioms}.

A valid thought representation should preserve what is needed to reproduce the output (Causality), discard what is irrelevant to that output (Minimality), remain distinguishable across semantically distinct inputs (Separability), and reflect the distribution of possible outputs rather than one sample from it (Stability). An audit of existing latent representations finds that every candidate violates at least one property. The audit evaluates existing methods, and does not perform training to target these properties. Training with CE only does not enforce them either, since a model receives no signal as long as the correct answer maintains high probability (\cref{sec:price}).

\begin{figure}[h!]
    \centering
    \includegraphics[width=\columnwidth]{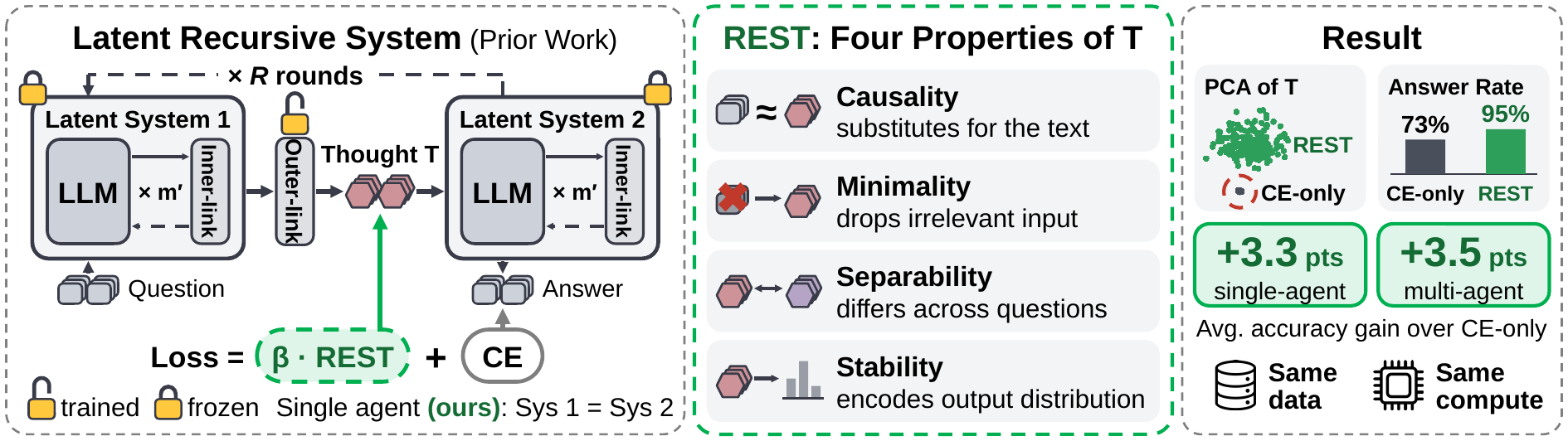}
    \caption{\textbf{Overview of REST.} Latent systems pass a thought $\mathbf{T}$ through a trained outer link (left). REST adds a $\beta$-weighted loss built from the four properties of $\mathbf{T}$ to the CE objective (middle). CE-only thoughts collapse and REST improves accuracy under the same data and compute (right).}
    \label{fig:overview}
\end{figure}

We introduce \textbf{REpresentation-Supervised Thought(s) (REST)}. REST translates each of these four properties into a differentiable loss term (\cref{app:theory:diff}) and adds it to the total loss alongside CE. Each term is weighted by a hyperparameter $\beta$, added on top of the existing CE loss. Training then targets the thought representation as well as the decoded answer (\cref{fig:overview}). REST is studied on a latent recursive system in which the base LLMs are frozen \citep{zou2026recursivemultiagentsystems}. We extend this construction from the multi-agent setting in which it was proposed to a single-agent setting, where a model recurs on its own hidden states. In addition, no changes to architecture or inference procedure are required, allowing it to be added to existing latent recursive systems without modification.

REST increases downstream accuracy across MATH500 \citep{lightman2024stepverify}, GPQA-Diamond \citep{rein2024gpqabench}, MedQA \citep{jin2021disease}, AIME2025/2026 \citep{zhang2025aimeexam,dekoninck2026matharena}, LiveCodeBench-v6 \citep{jain2025livecodebench}, and MBPP+ \citep{liu2023evalplus}, spanning mathematical, scientific, and code-generation domains. Averaged over all REST configurations, it improves accuracy over CE only by +3.5 points in the multi-agent setting and +3.3 in the single-agent setting, under a matched training budget, the same training data, and the same latent budget. The best configurations reach +7.5 and +6.5 points, respectively. Through our analysis, we confirm that the trained thoughts exhibit the properties that were targeted, which in turn lower-bound the probability of the correct answer (\cref{app:ce}). REST decodes 15.4\% more tokens on average at inference, which we attribute to a higher rate of convergence on a final answer (\cref{sec:analysis}).

Our contributions can be summarized as follows:
\begin{itemize}
    \item We identify four failures of CE-only thoughts that lead to a lower probability of the correct answer (\cref{sec:price}). In trained systems, CE-only thoughts collapse across distinct questions and retain irrelevant information (\cref{sec:analysis}).
    \item We introduce REST, which translates four theoretically motivated properties of thought representations \citep{seddik2026formalizinglatentthoughtsaxioms} into differentiable loss terms added to the CE objective, and we extend latent recursive systems \citep{zou2026recursivemultiagentsystems} from the multi-agent setting to a single agent that recurs on its own hidden states.
    \item REST improves accuracy over the CE baseline across mathematical, scientific, medical, and code-generation benchmarks in both the single-agent and multi-agent settings, under matched training and latent budgets, and its gains remain robust across loss weights.
\end{itemize}

\section{Background}
\label{sec:background}

\textbf{Setup and Notation.}
Sets and operators are denoted with ($\mathcal{V}$, $\mathcal{R}_{\mathrm{in}}$, $\mathcal{R}_\psi$), random
variables, matrices, and scalar-valued functions in upper case ($Y$, $H$, $J$), and scalars, indices,
token sequences, and single vectors in lower case ($m$, $t$, $u$, $e_i$), with bold for the
thought $\mathbf{T}$ and vectors derived from it ($\mathbf{t}$).
A next-token distribution is lower case ($p$, $q$) while the sequence is the
matching upper case ($P$, $Q$).

We adopt recursive multi-agent systems \citep{zou2026recursivemultiagentsystems}, and extend their approach to the single-agent setting to study REST. Let $f_\theta(\cdot)$ denote an auto-regressive transformer model with vocabulary $\mathcal{V}$ and hidden size $d_h$.
Given a token sequence $u$ with input embeddings
$E(u) = [e_1, \dots, e_m] \in \mathbb{R}^{m \times d_h}$, a forward pass yields last-layer hidden
states $H = [h_1, \dots, h_m] \in \mathbb{R}^{m \times d_h}$ and, at each position $i$, the
next-token distribution $\mathrm{softmax}(h_i W_{\mathrm{out}})$ over $\mathcal{V}$.
A multi-agent system $\mathcal{S}$ comprises $N$ agents $\mathcal{A} = \{A_1, \dots, A_N\}$, where
each $A_i$ corresponds to $f_{\theta_i}$, where every $\theta_i$ is frozen.

Hidden states re-enter embedding space through two links \citep{zou2026recursivemultiagentsystems}.
An inner link $\mathcal{R}_{\mathrm{in}}$ maps an agent's hidden state back to its own input space, which
lets $A_i$ take a latent step instead of decoding a token.
It is trained per agent by aligning $\mathcal{R}_{\mathrm{in}}(H)$ onto the embeddings of the
ground-truth text.
An outer link $\mathcal{R}_\psi$ maps the output of $\mathcal{R}_{\mathrm{in}}$ into the input space of another agent.
$\mathcal{R}_{\mathrm{in}}$'s objective only aligns hidden states to the input embedding space, whereas all
outer links are trained through CE on the target answer text at the final agent.

A latent thought of a producer $A_i \in \mathcal{A}$ to a consumer $A_j$ composes both links as
$\mathbf{T} = \mathcal{R}_\psi(\mathcal{R}_{\mathrm{in}}(H_u)) \in \mathbb{R}^{m' \times d_h}$, where
$H_u \in \mathbb{R}^{m \times d_h}$ denotes the producer's last-layer hidden states at the $m$
positions of its own output $u \in \mathcal{V}^{m}$, and $m'$ is a latent steps budget for transferred thoughts.

\begin{definition}[\textbf{Transfer}]\label{def:transfer}
For a consumer $A_j \in \mathcal{A}$, let $E_{\mathrm{pre}}$ and $E_{\mathrm{post}}$ denote the embedding blocks of a prompt
that surround the transferred content, both independent of $\psi$.
For every block $Z$ of $d_h$-dimensional embedding vectors, the consumer
performs a forward pass on $[\,E_{\mathrm{pre}} \,;\, Z \,;\, E_{\mathrm{post}} \,;\, E(v)]$ where $v$ is the teacher-forced target answer for the consumer. The transfer can be through latent thoughts for $Z = \mathbf{T}$ and through textual embeddings for $Z = E(u)$.
\end{definition}

Training and inference differ in how $H_u$ is obtained.
During training, $u$ is the producer's ground-truth output, and one teacher-forced pass over the
producer's prompt concatenated with $u$ yields $H_u$ at $u$'s positions.
During inference no $u$ is available in advance, so the producer starts from its prompt alone and
takes $m'$ latent steps.
At each step, $\mathcal{R}_{\mathrm{in}}$ maps the last position's hidden state into the next input
embedding, and $H_u = [h_1, \dots, h_{m'}]$ is the sequence of hidden states over these steps.

\textbf{Roles and Hops.}
The system chains a planner $A_1$, a refiner $A_2$ and a solver $A_3$ (\cref{fig:overview},
\cref{tab:agent-config}).
The planner proposes an approach to the question, the refiner revises that approach, and the solver
produces the answer.
Each agent takes $m'$ latent steps through $\mathcal{R}_{\mathrm{in}}$, and $\mathcal{R}_\psi$ maps
the resulting hidden states into the next agent's input space.
Every handoff of this kind is one transfer of \cref{def:transfer}.

\textbf{Recursion.}
Once the solver finishes, its thought returns to the planner and the chain repeats for a further
round \citep{zou2026recursivemultiagentsystems}.
Let $R$ denote the number of rounds, let $\mathbf{T}^{(r,k)}$ denote the thought that $A_k$ receives
at round $r$, and let $\mathcal{H}$ collect the pairs $(r,k)$ at which a transfer occurs.
The planner of the first round is the one agent that conditions on the question without a
transferred thought.
Every later agent conditions on both its own input context and the thought it receives.
The solver of the final round decodes $v$ as text, and every earlier hop stays in latent space.

\subsection{Limitations of Answer-Level Supervision}
\label{sec:price}

\textbf{Substitution Divergence.}
Let $P(\cdot \mid Z)$ denote the distribution the consumer assigns to complete targets under a
transferred block $Z$.
A consumer answers under either block of \cref{def:transfer}.
Let $\mathrm{CE}_{\mathrm{lat}}$ and $\mathrm{CE}_{\mathrm{txt}}$ denote the per-token CE
it incurs on a target $v$ under each one.
Let
$\Lambda(v) = \mathrm{CE}_{\mathrm{lat}}(v) - \mathrm{CE}_{\mathrm{txt}}(v)$
denote the cost of substituting the thought for the producer's text.
We can represent this as a divergence between the two transfers over the targets the producer's text induces,
\begin{equation}\label{eq:bg-substitution}
    \mathbb{E}_{v \sim P(\cdot \mid E(u))}\big[ \Lambda(v) \big]
    = \frac{1}{n}\, \KL\big( P(\cdot \mid E(u)) \,\big\|\, P(\cdot \mid \mathbf{T}) \big) ,
\end{equation}
where $n$ is the length of $v$ (\cref{app:ce:causality}).
This divergence can be zero only when both transfers lead the consumer agent to the same distribution.
The probability the consumer assigns to the correct target is reduced by a factor of
$\exp(-n \Lambda(v))$.
A longer target is therefore penalized more for the same divergence per token. This effect compounds for each hop $A_i \to A_j$ in the pipeline.
Empirically, CE-only does not reach the accuracy of the oracle text from the producer agent (\cref{fig:analysis-panels}(b)).

\textbf{Information Loss.}
When a consumer agent solves a question, it may solve a large part on its own (\cref{tab:base-model}), and on those
questions CE would be close to its minimum. Gradients updating the thought for those examples would be small compared to others. Therefore, the output from a producer agent may not be preserved as $\psi$ may not learn from such examples. The consumer agent is then left with less information than when
given the producer's generated output (\cref{prop:ce-minimality-chain}). In addition, since the transfer is further limited to a budget of $m'$ positions, $\psi$ may learn to encode irrelevant information due to such examples since their supervision signal is too weak.
Empirically, CE-only encodes more content from the producer agent's input than its output (\cref{fig:analysis-panels}(a)).

\textbf{Thought Collision.}
If training resulted in two different examples having colliding thoughts, then a consumer agent must
answer both with nearly the same distribution even if the two target texts are different.
This imposes a lower bound on their average CE (\cref{prop:ce-separability-floor}).
Under total collapse, where the two thoughts are identical, the optimal solution would be assigning equal probability to both targets, which would mean it is random chance.
CE evaluates every example on its own target.
Summing over examples does not change this, since each term depends on one thought only.
A collision raises the loss on both examples, and CE cannot distinguish this from two
difficult examples.
Empirically, CE-only collapses distinct questions together (\cref{fig:separability-pca}).

\textbf{Producer Uncertainty.}
Generated text from a producer used as supervision targets would not encode the level of certainty for that output. At each step, it may be choosing between alternative solutions of equal
probability. CE supervises only the consumer's final answer, and it does not require the thought to
encode this certainty. The loss would be higher for a thought that doesn't encode the possible output distribution from a producer agent (\cref{prop:ce-stability-variance}).

\begin{figure}[h!]
    \centering
    \includegraphics[width=\linewidth]{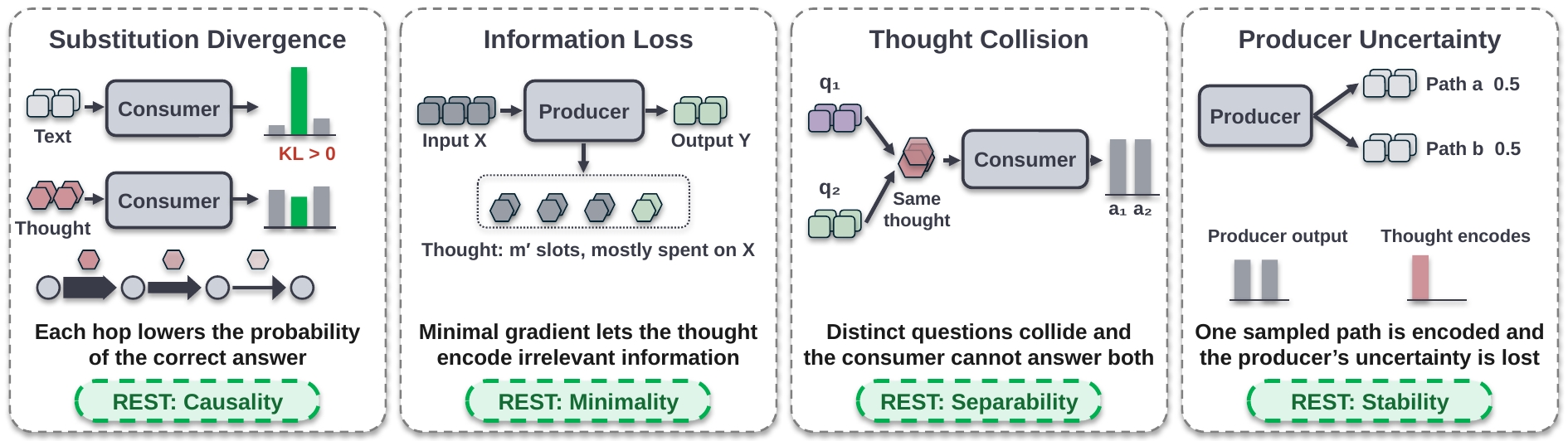}
    \caption{\textbf{Why CE is not enough.} Four failures of CE thoughts and the property that addresses each.}
    \label{fig:failures}
\end{figure}

\section{REST: Representation-Supervised Thought(s)}
\label{sec:method}

Every agent base LLM $f_{\theta_i}$ stays frozen alongside $\mathcal{R}_{\mathrm{in}}$. The outer link $\mathcal{R}_\psi$ is the component that is trained with CE, thus we keep $\mathcal{R}_{\mathrm{in}}$ frozen and only train $\mathcal{R}_\psi$ to isolate the effect of REST. CE is kept as part of the total loss function to maintain legible text outputs when answering. For each term, we start with the definition of the property and arrive at the loss function below (proofs in \cref{app:theory:diff}).

\subsection{Principled Thoughts}
\label{sec:thoughts}

Let $n$ denote the length of the target $v$ of \cref{def:transfer}, let
$v_{<t} = (v_1, \dots, v_{t-1})$ denote its prefix at position $t \in \{1, \dots, n\}$, and let
$p(\cdot \mid v_{<t}, Z)$ denote the consumer agent's next-token distribution at that position under the
transferred block $Z$.

\textbf{Causality} requires that $\mathbf{T}$ hold the required information about the text generated by a producer agent when given to a
consumer agent without altering what a consumer agent would have predicted. We enforce this property by penalizing the divergence
between the two transfers of \cref{def:transfer},
\begin{equation}\label{eq:loss-causality}
    \mathcal{L}_{\mathrm{caus}}(\psi)
    = \frac{1}{n} \sum_{t=1}^{n}
      \KL\Big( p\big(\cdot \mid v_{<t}, E(u)\big) \,\Big\|\, p\big(\cdot \mid v_{<t}, \mathbf{T}\big) \Big).
\end{equation}

\textbf{Minimality} requires that $\mathbf{T}$ should remove irrelevant information that was present in the input of the producer agent while maintaining relevant information relative to its output.
Let $X$ and $Y$ denote the producer's input and output as random variables, the latter realized by
the sequence $u$ (\cref{def:transfer}), and let $\mathrm{CE}(Y \mid \mathbf{T})$ and
$\mathrm{CE}(X \mid Y, \mathbf{T})$ denote the CE terms calculated for the consumer agent,
\begin{equation}\label{eq:loss-minimality}
    \mathcal{L}_{\mathrm{min}}(\psi)
    = \lambda_1 \, \mathrm{CE}\big( Y \mid \mathbf{T} \big)
      - \lambda_2 \, \mathrm{CE}\big( X \mid Y, \mathbf{T} \big) ,
\end{equation}
where $\lambda_1, \lambda_2 > 0$ are scalars. Since a thought representation $\mathbf{T}$ from the producer agent may encode information from its input and output, this loss term would penalize a representation that would encode irrelevant information from its input relative to the information present in its output.

\textbf{Separability} requires that two thought representations for semantically distinct outputs should be distinguishable or separable to represent that they encode semantically distinct information. In order to apply this on $\mathbf{T}$, which can be a sequence of vectors, we utilize an attention pool to produce one vector for each sequence. Let $\mathbf{t}$ and $\mathbf{t}_1, \dots, \mathbf{t}_K$ in $\mathbb{R}^{d_h}$ denote the thoughts pooled over the $m'$ positions of the current example and of the $K$ training examples that precede it, respectively, and let $\mathrm{sim}(a, b)$ denote the cosine similarity.
We penalize the similarity of $\mathbf{t}$ to each $\mathbf{t}_k$,
\begin{equation}\label{eq:loss-separability}
    \mathcal{L}_{\mathrm{sep}}(\psi, \omega)
    = \log \sum_{k=1}^{K}
      \exp\!\Big( \mathrm{sim}\big( \mathbf{t}, \mathbf{t}_k \big) \big/ \tau \Big) ,
\end{equation}
where $\tau > 0$ is a scalar for temperature and $\omega$ denotes the parameters of the attention pooling.

\textbf{Stability} requires encoding the output distribution rather than one sequence.
Sampling that distribution would require many outputs at every stage of the system. Estimating the entropy on the other hand can track the property without bias (\cref{app:theory:stability}).
Let $q(\cdot \mid u_{<s})$ denote the producer's next-token distribution at position
$s \in \{1, \dots, m\}$ of its own output $u$, let $\mathbb{H}$ denote Shannon entropy, and let
\begin{equation}\label{eq:entropy-rate-estimate}
    \widehat{\mathbb{H}}(u)
    = \frac{1}{m} \sum_{s=1}^{m} \mathbb{H}\big( q(\cdot \mid u_{<s}) \big)
\end{equation}
denote the producer's mean predictive entropy along $u$.
We enforce the property through the squared error between \cref{eq:entropy-rate-estimate} and an estimate of its value through a learned probe that receives $\mathbf{T}$,
\begin{equation}\label{eq:loss-stability}
    \mathcal{L}_{\mathrm{stab}}(\psi, \omega)
    = \Big( g_\omega(\mathbf{T}) - \widehat{\mathbb{H}}(u) \Big)^{2} ,
\end{equation}
where $g_\omega$ is an attention pooling over the $m'$ positions of $\mathbf{T}$ with an
affine map from $\mathbb{R}^{d_h}$ to $\mathbb{R}$.

\subsection{Multi-Agent Systems}
\label{sec:multi-agent}

REST adds one term at every transfer of \cref{def:transfer} and leaves the CE of
the system unchanged,
\begin{equation}\label{eq:objective}
    \mathcal{L}(\psi, \omega)
    = \mathcal{L}_{\mathrm{CE}}\big(v \mid \mathbf{T}^{(R,N)}\big)
      + \beta \, \frac{1}{|\mathcal{H}|} \sum_{(r,k) \in \mathcal{H}}
        \mathcal{L}_{\mathrm{prop}}\big(\psi, \omega; \mathbf{T}^{(r,k)}\big) ,
\end{equation}
where $\mathcal{L}_{\mathrm{prop}} \in \{ \mathcal{L}_{\mathrm{caus}}, \mathcal{L}_{\mathrm{min}},
\mathcal{L}_{\mathrm{sep}}, \mathcal{L}_{\mathrm{stab}} \}$, and $\mathbf{T}^{(R,N)}$ is the thought that the solver of the final round
receives. $\omega$ is present only in $\mathcal{L}_{\mathrm{sep}}$ and
$\mathcal{L}_{\mathrm{stab}}$.
The property term is averaged over each transfer, therefore receiving a gradient from its own term in addition to the final CE.

\subsection{Single-Agent Systems}
\label{sec:single-agent}

\begin{wrapfigure}{r}{0.37\linewidth}
\centering
\vspace{-2.0\baselineskip}
\includegraphics[width=\linewidth]{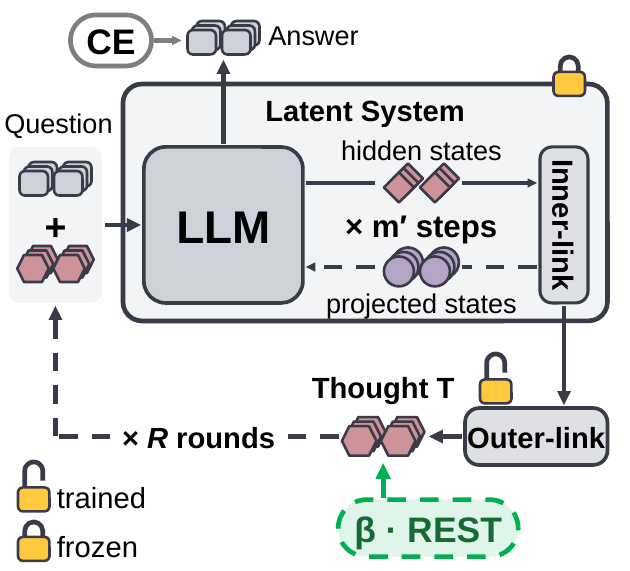}
\caption{\textbf{Single agent}. $\mathcal{R}_{\mathrm{in}}$ runs at each of the $m'$ steps, and $\mathcal{R}_\psi$ once per round.}
\label{fig:selfloop}
\vspace{-1.0\baselineskip}
\end{wrapfigure}

It is important to note that \cref{def:transfer} permits $A_i = A_j$, which corresponds to our extension to the single-agent setting. REST extends to a single agent reasoning in latent space, rather than being restricted to
a pipeline of several agents.

\textbf{Self-Loop.}
A single solver acts as both the producer and the consumer, with no planner and no refiner
(\cref{fig:selfloop}).
The solver first runs on its own input context and computes last-layer hidden states.
These states pass through the frozen inner link $\mathcal{R}_{\mathrm{in}}$ and the trained outer
link $\mathcal{R}_\psi$, which gives the thought $\mathbf{T}$ of \cref{sec:background}.
The solver then runs a second time, conditioned on both its own input context and $\mathbf{T}$.
Each further round repeats the second pass, and the hidden states it computes form the next round's
thought.
\Cref{eq:objective} applies at $N = 1$.
The two links act at different points of the system.
$\mathcal{R}_{\mathrm{in}}$ runs at every one of the $m'$ latent steps and lets the solver take a
latent step instead of decoding a token, whereas $\mathcal{R}_\psi$ runs once per round.
Removing $\mathcal{R}_{\mathrm{in}}$ therefore removes latent reasoning itself rather than one
trainable component.
Keeping it frozen leaves $\psi$ as the only parameter that varies between CE and REST,
and between the single- and multi-agent settings.

\section{Evaluation}
\label{sec:experiments}

\subsection{Experimental Setup}

\begin{wraptable}{R}{0.35\linewidth}
\centering
\vspace{-1.5\baselineskip} 
\scriptsize
\setlength{\tabcolsep}{3pt}
\caption{Light and Scaled systems.}
\label{tab:agent-config}
\begin{tabular}{l|l|l}
\toprule
\textbf{System} & \textbf{Role} & \textbf{Model} \\
\midrule

\multirow{3}{*}{Light}
& Planner & Qwen3-1.7B  \\
& Refiner & Llama-3.2-1B-Instruct \\
& Solver  & Qwen2.5-Math-1.5B-Instruct \\
\midrule

\multirow{3}{*}{Scaled}
& Planner & Gemma-3-4B-it \\
& Refiner & Llama-3.2-3B-Instruct \\
& Solver  & Qwen3.5-4B \\
\bottomrule
\end{tabular}
\vspace{-1.0\baselineskip} 
\end{wraptable}

\textbf{Systems and Baselines.} We evaluate REST on open-weight LLMs across the Qwen \citep{qwen2025qwentwofive,yang2024qwen25mathtechnicalreportmathematical,yang2025qwenthree,qwenteam2026qwenthreefive}, Llama \citep{grattafiori2024llamaherd}, and Gemma \citep{gemmateam2025gemmathree} model families for heterogeneous agent collaborations. \Cref{tab:agent-config} lists the model assigned to each role. For baseline comparisons, we evaluate CODI \citep{shen2025selfdistillcot} and SIM-CoT \citep{wei2025simcot} adapted as loss terms added to CE, and the frozen LLMs as a text baseline (\cref{tab:base-model}). CE only is the unmodified latent recursive system that REST builds on, and Best pair combines causality and minimality. Detailed implementations and hyperparameters are in \cref{app:baseline} and \cref{app:hyperparams}.

\textbf{Data.}
We adopt Sequential-Math \citep{zou2026recursivemultiagentsystems} for training, constructed by rewriting question-answer pairs curated from s1K \citep{muennighoff2025simplescaling}, m1K \citep{huang2026medicalscaling} into role-specific texts.
Evaluation covers four domains, mathematics with MATH500 \citep{lightman2024stepverify}, AIME2025 \citep{zhang2025aimeexam}, and AIME2026 \citep{dekoninck2026matharena}, science with GPQA-Diamond \citep{rein2024gpqabench}, medicine with MedQA \citep{jin2021disease}, and code generation with MBPP+ \citep{liu2023evalplus} and LiveCodeBench-v6 \citep{jain2025livecodebench}. Additional details in \cref{app:datasets}.

\textbf{Training and Inference.}
For training, we freeze all agent parameters and update only $\psi$ and $\omega$. The training objective adds a weighted term to the loss function. AdamW is used under a cosine learning rate schedule. During inference, we follow each model's official recommended sampling settings, and a lower temperature for code generation than for other reasoning tasks. We average values across three training seeds. The Avg.\ Change column reports the average change against the CE-only baseline across the benchmarks of one system, and each system is reported at the recursion round that performs best for it, with the baseline taken at that same round for fair comparison.

\subsection{Single-Agent Evaluation}

\begin{table}[h!]
\centering \small
\caption{Single-agent, Light vs Scaled, at $r=1$ and $r=3$ respectively, each row at its best $\beta$.}
\label{tab:single-agent-main}
\resizebox{\linewidth}{!}{%
\arrayrulecolor{black!30}
\begin{tabular}{l|c|c>{\columncolor{tolGrey!40}}cc>{\columncolor{tolGrey!40}}cc>{\columncolor{tolGrey!40}}c|c>{\columncolor{tolGrey!40}}c|c>{\columncolor{tolGrey!40}}c|c>{\columncolor{tolGrey!40}}c|cc}
\arrayrulecolor{restcoral}\specialrule{2.0pt}{0pt}{2pt}\arrayrulecolor{black!30}
\textbf{Method} & \textbf{Metric} & \multicolumn{2}{c}{\textbf{Math500}} & \multicolumn{2}{c}{\textbf{AIME2025}} & \multicolumn{2}{c|}{\textbf{AIME2026}} & \multicolumn{2}{c|}{\textbf{GPQA-D}} & \multicolumn{2}{c|}{\textbf{MedQA}} & \multicolumn{2}{c|}{\textbf{Code Gen.}} & \multicolumn{2}{c}{\textbf{Avg.\ Change}} \\
\cmidrule(lr){3-4}\cmidrule(lr){5-6}\cmidrule(lr){7-8}\cmidrule(lr){9-10}\cmidrule(lr){11-12}\cmidrule(lr){13-14}\cmidrule(lr){15-16}
 & & \textbf{Light} & \textbf{Scaled} & \textbf{Light} & \textbf{Scaled} & \textbf{Light} & \textbf{Scaled} & \textbf{Light} & \textbf{Scaled} & \textbf{Light} & \textbf{Scaled} & \textbf{Light} & \textbf{Scaled} & \textbf{Light} & \textbf{Scaled} \\
\midrule
 & Acc. & 70.6 & 80.9 & 27.8 & 65.6 & 12.2 & 72.2 & 27.3 & 63.3 & 27.1 & 79.2 & 27.7 & 35.5 & \basetxt & \basetxt \\
\multirow{-2}{*}{CE only} & Token & 557 & 862 & 905 & 7226 & 1016 & 6967 & 911 & 1709 & 1177 & 739 & 477 & 1289 & \basetxt & \basetxt \\
\arrayrulecolor{restcoral}\specialrule{0.8pt}{0pt}{0pt}\arrayrulecolor{black!30}
\rowcolor{restcoral!40}\multicolumn{16}{c}{\textbf{\textsc{REST (ours), single property}}} \\
\arrayrulecolor{restcoral}\specialrule{0.8pt}{0pt}{0pt}\arrayrulecolor{black!30}
 & Acc. & 72.4 & 80.4 & 23.3 & 78.9 & 16.7 & 83.3 & 23.9 & 63.5 & 29.7 & 79.4 & 32.9 & 40.2 & \gaincell{$\uparrow 1.0$} & \gaincell{$\uparrow 4.8$} \\
\multirow{-2}{*}{Causality} & Token & 551 & 1008 & 907 & 8677 & 1008 & 8165 & 851 & 2216 & 1113 & 955 & 570 & 1662 & \savecell{$-0.9\%$} & $+20.7\%$ \\
\midrule
 & Acc. & 72.1 & 81.1 & 30.0 & 78.9 & 20.0 & 86.7 & 27.6 & 63.8 & 29.7 & 83.0 & 31.5 & 42.3 & \gaincell{$\uparrow 3.0$} & \gaincell{$\uparrow 6.5$} \\
\multirow{-2}{*}{Minimality} & Token & 546 & 1053 & 906 & 10256 & 936 & 10390 & 872 & 2526 & 1075 & 1162 & 561 & 1850 & \savecell{$-2.9\%$} & $+44.9\%$ \\
\midrule
 & Acc. & 73.1 & 80.2 & 26.7 & 76.7 & 17.8 & 83.3 & 28.8 & 61.1 & 27.7 & 81.7 & 30.2 & 39.4 & \gaincell{$\uparrow 1.9$} & \gaincell{$\uparrow 4.3$} \\
\multirow{-2}{*}{Separability} & Token & 551 & 1056 & 918 & 8657 & 987 & 8033 & 901 & 2556 & 1083 & 1087 & 488 & 1764 & \savecell{$-2.3\%$} & $+23.2\%$ \\
\midrule
 & Acc. & 72.1 & 79.1 & 27.8 & 83.3 & 15.6 & 88.3 & 28.1 & 64.4 & 29.8 & 79.8 & 29.1 & 40.4 & \gaincell{$\uparrow 1.6$} & \gaincell{$\uparrow 6.4$} \\
\multirow{-2}{*}{Stability} & Token & 538 & 1049 & 921 & 9626 & 944 & 9174 & 845 & 2464 & 1051 & 1071 & 645 & 1783 & \savecell{$-2.0\%$} & $+33.9\%$ \\
\arrayrulecolor{restcoral}\specialrule{0.8pt}{0pt}{0pt}\arrayrulecolor{black!30}
\rowcolor{restcoral!40}\multicolumn{16}{c}{\textbf{\textsc{REST (ours), composition of properties}}} \\
\arrayrulecolor{restcoral}\specialrule{0.8pt}{0pt}{0pt}\arrayrulecolor{black!30}
 & Acc. & 69.7 & 79.4 & 23.3 & 80.0 & 16.7 & 86.7 & 26.1 & 57.6 & 29.6 & 79.0 & 33.5 & 38.6 & \gaincell{$\uparrow 1.0$} & \gaincell{$\uparrow 4.1$} \\
\multirow{-2}{*}{Best pair} & Token & 542 & 1065 & 876 & 10503 & 1039 & 10728 & 826 & 2605 & 1083 & 1179 & 553 & 1632 & \savecell{$-2.4\%$} & $+47.5\%$ \\
\midrule
 & Acc. & 72.4 & 78.4 & 33.3 & 73.3 & 16.7 & 76.7 & 23.7 & 60.6 & 31.0 & 79.3 & 32.8 & 37.0 & \gaincell{$\uparrow 2.9$} & \gaincell{$\uparrow 1.4$} \\
\multirow{-2}{*}{All properties} & Token & 569 & 981 & 934 & 11053 & 984 & 10338 & 953 & 1744 & 1130 & 1001 & 555 & 1375 & $+1.6\%$ & $+41.0\%$ \\
\arrayrulecolor{restcoral}\specialrule{2.0pt}{2pt}{0pt}\arrayrulecolor{black}
\end{tabular}%
}
\end{table}

\Cref{tab:single-agent-main} shows every property term improving accuracy over the CE-only objective in this setting.
The effect can be attributed to the thought rather than collaboration among agents. A practitioner is also more likely to already be running a single model system when considering costs. In the Light system, the token count is lower than CE, while it is higher in the multi-agent setting. Among the four properties, minimality delivers the largest gain in the single-agent setting.
In the Light system, minimality exceeds the baseline on every task while decoding fewer tokens overall, indicating that its accuracy gain does not come from longer generation.
This is consistent with the structure of the self-loop, in which the producer's input contains the same question and instructions that the consumer already conditions on when receiving $\mathbf{T}$.
As a result, any input content encoded in $\mathbf{T}$ is redundant and occupies latent positions that would otherwise carry the producer's output.

\subsection{Multi-Agent Evaluation}

\begin{table}[h!]
\centering \small
\caption{Multi-agent, Light vs Scaled, at $r=1$ and $r=3$ respectively, each row at its best $\beta$.}
\label{tab:multi-agent-main}
\resizebox{\linewidth}{!}{%
\arrayrulecolor{black!30}
\begin{tabular}{l|c|c>{\columncolor{tolGrey!40}}cc>{\columncolor{tolGrey!40}}cc>{\columncolor{tolGrey!40}}c|c>{\columncolor{tolGrey!40}}c|c>{\columncolor{tolGrey!40}}c|c>{\columncolor{tolGrey!40}}c|cc}
\arrayrulecolor{restcoral}\specialrule{2.0pt}{0pt}{2pt}\arrayrulecolor{black!30}
\textbf{Method} & \textbf{Metric} & \multicolumn{2}{c}{\textbf{Math500}} & \multicolumn{2}{c}{\textbf{AIME2025}} & \multicolumn{2}{c|}{\textbf{AIME2026}} & \multicolumn{2}{c|}{\textbf{GPQA-D}} & \multicolumn{2}{c|}{\textbf{MedQA}} & \multicolumn{2}{c|}{\textbf{Code Gen.}} & \multicolumn{2}{c}{\textbf{Avg.\ Change}} \\
\cmidrule(lr){3-4}\cmidrule(lr){5-6}\cmidrule(lr){7-8}\cmidrule(lr){9-10}\cmidrule(lr){11-12}\cmidrule(lr){13-14}\cmidrule(lr){15-16}
 & & \textbf{Light} & \textbf{Scaled} & \textbf{Light} & \textbf{Scaled} & \textbf{Light} & \textbf{Scaled} & \textbf{Light} & \textbf{Scaled} & \textbf{Light} & \textbf{Scaled} & \textbf{Light} & \textbf{Scaled} & \textbf{Light} & \textbf{Scaled} \\
\midrule
 & Acc. & 71.1 & 86.6 & 22.2 & 80.0 & 17.8 & 60.0 & 24.9 & 59.1 & 28.9 & 79.7 & 32.5 & 33.5 & \basetxt & \basetxt \\
\multirow{-2}{*}{CE only} & Token & 550 & 1252 & 849 & 10485 & 884 & 5420 & 749 & 2572 & 868 & 829 & 604 & 1029 & \basetxt & \basetxt \\
\arrayrulecolor{restcoral}\specialrule{0.8pt}{0pt}{0pt}\arrayrulecolor{black!30}
\rowcolor{restcoral!40}\multicolumn{16}{c}{\textbf{\textsc{REST (ours), single property}}} \\
\arrayrulecolor{restcoral}\specialrule{0.8pt}{0pt}{0pt}\arrayrulecolor{black!30}
 & Acc. & 77.8 & 86.8 & 27.8 & 80.0 & 22.2 & 86.7 & 28.6 & 54.9 & 30.0 & 78.7 & 33.9 & 39.1 & \gaincell{$\uparrow 3.8$} & \gaincell{$\uparrow 4.6$} \\
\multirow{-2}{*}{Causality} & Token & 601 & 1158 & 894 & 8576 & 978 & 8278 & 888 & 1568 & 1049 & 996 & 643 & 1924 & $+12.2\%$ & $+4.2\%$ \\
\midrule
 & Acc. & 77.1 & 85.8 & 30.0 & 86.7 & 22.2 & 83.3 & 27.6 & 62.6 & 26.8 & 83.0 & 32.0 & 39.0 & \gaincell{$\uparrow 3.1$} & \gaincell{$\uparrow 6.9$} \\
\multirow{-2}{*}{Minimality} & Token & 605 & 1204 & 920 & 9873 & 1052 & 9558 & 849 & 2528 & 1158 & 1275 & 575 & 1875 & $+14.6\%$ & $+21.9\%$ \\
\midrule
 & Acc. & 74.7 & 82.2 & 25.6 & 73.3 & 17.8 & 65.6 & 25.6 & 52.5 & 30.9 & 78.7 & 31.1 & 31.6 & \gaincell{$\uparrow 1.4$} & \losscell{$\downarrow 2.5$} \\
\multirow{-2}{*}{Separability} & Token & 564 & 1191 & 838 & 7039 & 923 & 6940 & 758 & 1940 & 1116 & 1189 & 1650 & 1280 & $+29.9\%$ & \savecell{$-9.3\%$} \\
\midrule
 & Acc. & 74.0 & 85.6 & 27.8 & 80.0 & 16.7 & 80.0 & 26.8 & 62.6 & 31.0 & 75.0 & 35.1 & 34.2 & \gaincell{$\uparrow 2.3$} & \gaincell{$\uparrow 3.1$} \\
\multirow{-2}{*}{Stability} & Token & 551 & 1105 & 881 & 7588 & 976 & 6497 & 785 & 1908 & 1205 & 941 & 781 & 1215 & $+15.0\%$ & \savecell{$-10.8\%$} \\
\arrayrulecolor{restcoral}\specialrule{0.8pt}{0pt}{0pt}\arrayrulecolor{black!30}
\rowcolor{restcoral!40}\multicolumn{16}{c}{\textbf{\textsc{REST (ours), composition of properties}}} \\
\arrayrulecolor{restcoral}\specialrule{0.8pt}{0pt}{0pt}\arrayrulecolor{black!30}
 & Acc. & 77.6 & 87.0 & 25.6 & 83.3 & 15.6 & 86.7 & 29.0 & 60.1 & 30.1 & 85.0 & 33.9 & 41.7 & \gaincell{$\uparrow 2.4$} & \gaincell{$\uparrow 7.5$} \\
\multirow{-2}{*}{Best pair} & Token & 604 & 1227 & 916 & 10281 & 978 & 9828 & 886 & 2629 & 1100 & 1259 & 613 & 2134 & $+13.2\%$ & $+26.7\%$ \\
\midrule
 & Acc. & 76.2 & 86.6 & 28.9 & 78.9 & 21.1 & 86.7 & 27.6 & 63.6 & 28.7 & 81.7 & 33.3 & 42.1 & \gaincell{$\uparrow 3.1$} & \gaincell{$\uparrow 6.8$} \\
\multirow{-2}{*}{All properties} & Token & 608 & 1253 & 900 & 11033 & 1022 & 11275 & 879 & 2673 & 1042 & 1373 & 569 & 2151 & $+11.5\%$ & $+37.9\%$ \\
\arrayrulecolor{restcoral}\specialrule{2.0pt}{2pt}{0pt}\arrayrulecolor{black}
\end{tabular}%
}
\end{table}

\Cref{tab:multi-agent-main} reports the full planner, refiner, and solver system.
Every setting improves accuracy over CE-only in the Light system. This trend continues for the Scaled system as well except separability, where it reduces the number of tokens instead.
In the Light system, REST with causality or minimality matches or surpasses the strongest frozen agent in \cref{tab:base-model} on all three mathematical benchmarks, while the CE-only system underperforms it on all three. Unlike in the Light single-agent setting, REST increases tokens in the multi-agent system (more in \cref{sec:analysis}). Beyond the best configuration, the $r=1$ sweep in \cref{tab:property-sweep} shows that REST does not rely on careful weight tuning, as both causality and minimality improve the average accuracy at every $\beta$.

\subsection{Analysis}
\label{sec:analysis}

\begin{wrapfigure}{R}{0.3\linewidth}
\centering
\vspace{-2.0\baselineskip}
\includegraphics[width=\linewidth]{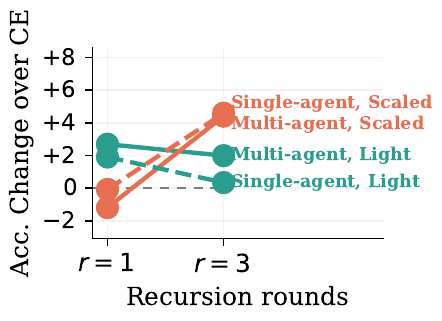}
\caption{Acc.\ Change vs $r$}
\label{fig:depth-scale}
\vspace{-1.0\baselineskip}
\end{wrapfigure}

\textbf{Scale and Recursion Depth.}
\Cref{fig:depth-scale} averages the change in accuracy over CE-only at each recursion round.
Additional rounds increase the gain of REST for the Scaled agents and reduce it for the Light agents, in both the single-agent and the multi-agent systems.
Larger agents therefore make use of additional rounds of latent exchange, whereas smaller agents obtain their gain from a single round.
\Cref{app:per-round} provides the full tables at both rounds.

\textbf{Auxiliary-Loss Baselines.}
\Cref{tab:baseline-comparison} compares REST against CODI and SIM-CoT adapted as loss terms.
REST retains an advantage over the two in accuracy on both systems, and in token overhead on the Light system.
Both baselines supervise the latent state against a text the producer would have written, whereas REST constrains a property of the thought.
CODI requires the consumer's hidden state to match the state observed under text, while causality requires only that the consumer predict the same continuation.
The comparison isolates the effect of an auxiliary term that targets the representation.

\begin{table}[h!]
\centering \small
\caption{REST against auxiliary-loss baselines, Light vs Scaled.}
\label{tab:baseline-comparison}
\resizebox{\linewidth}{!}{%
\arrayrulecolor{black!30}
\begin{tabular}{l|c|c>{\columncolor{tolGrey!40}}cc>{\columncolor{tolGrey!40}}cc>{\columncolor{tolGrey!40}}c|c>{\columncolor{tolGrey!40}}c|c>{\columncolor{tolGrey!40}}c|c>{\columncolor{tolGrey!40}}c|cc}
\arrayrulecolor{restcoral}\specialrule{2.0pt}{0pt}{2pt}\arrayrulecolor{black!30}
\textbf{Method} & \textbf{Metric} & \multicolumn{2}{c}{\textbf{Math500}} & \multicolumn{2}{c}{\textbf{AIME2025}} & \multicolumn{2}{c|}{\textbf{AIME2026}} & \multicolumn{2}{c|}{\textbf{GPQA-D}} & \multicolumn{2}{c|}{\textbf{MedQA}} & \multicolumn{2}{c|}{\textbf{Code Gen.}} & \multicolumn{2}{c}{\textbf{Avg.\ Change}} \\
\cmidrule(lr){3-4}\cmidrule(lr){5-6}\cmidrule(lr){7-8}\cmidrule(lr){9-10}\cmidrule(lr){11-12}\cmidrule(lr){13-14}\cmidrule(lr){15-16}
 & & \textbf{Light} & \textbf{Scaled} & \textbf{Light} & \textbf{Scaled} & \textbf{Light} & \textbf{Scaled} & \textbf{Light} & \textbf{Scaled} & \textbf{Light} & \textbf{Scaled} & \textbf{Light} & \textbf{Scaled} & \textbf{Light} & \textbf{Scaled} \\
\midrule
 & Acc. & 71.1 & 87.7 & 22.2 & 76.7 & 17.8 & 86.7 & 24.9 & 63.1 & 28.9 & 80.2 & 32.5 & 39.0 & \basetxt & \basetxt \\
\multirow{-2}{*}{CE only} & Token & 550 & 1018 & 849 & 9009 & 884 & 8488 & 749 & 2379 & 868 & 1137 & 604 & 1507 & \basetxt & \basetxt \\
\midrule
 & Acc. & 76.9 & 85.9 & 30.0 & 83.3 & 22.2 & 83.3 & 28.8 & 61.8 & 27.3 & 80.3 & 32.0 & 39.5 & \gaincell{$\uparrow 3.3$} & \gaincell{$\uparrow 0.1$} \\
\multirow{-2}{*}{CODI ($\beta=20$)} & Token & 598 & 1035 & 933 & 9344 & 987 & 9153 & 861 & 2293 & 1101 & 1097 & 682 & 1602 & $+14.6\%$ & $+4.2\%$ \\
\midrule
 & Acc. & 76.7 & 82.0 & 27.8 & 77.8 & 21.1 & 78.9 & 26.9 & 54.5 & 28.9 & 79.0 & 34.9 & 35.0 & \gaincell{$\uparrow 3.2$} & \losscell{$\downarrow 4.4$} \\
\multirow{-2}{*}{SIM-CoT ($\lambda_{\text{step}}=0.3$)} & Token & 598 & 728 & 929 & 9854 & 1004 & 8797 & 914 & 943 & 1202 & 624 & 549 & 1276 & $+15.4\%$ & \savecell{$-5.6\%$} \\
\midrule
 & Acc. & 77.8 & 88.0 & 27.8 & 83.3 & 22.2 & 90.0 & 28.6 & 67.2 & 30.0 & 81.3 & 33.9 & 42.3 & \gaincell{$\uparrow 3.8$} & \gaincell{$\uparrow 3.1$} \\
\multirow{-2}{*}{\textbf{REST (ours)}} & Token & 601 & 1044 & 894 & 10309 & 978 & 9956 & 888 & 2287 & 1049 & 1120 & 643 & 2036 & $+12.2\%$ & $+13.7\%$ \\
\arrayrulecolor{restcoral}\specialrule{2.0pt}{2pt}{0pt}\arrayrulecolor{black}
\end{tabular}%
}

\vspace{2pt}
\end{table}

\begin{figure}[h!]
\centering
\includegraphics[width=\linewidth]{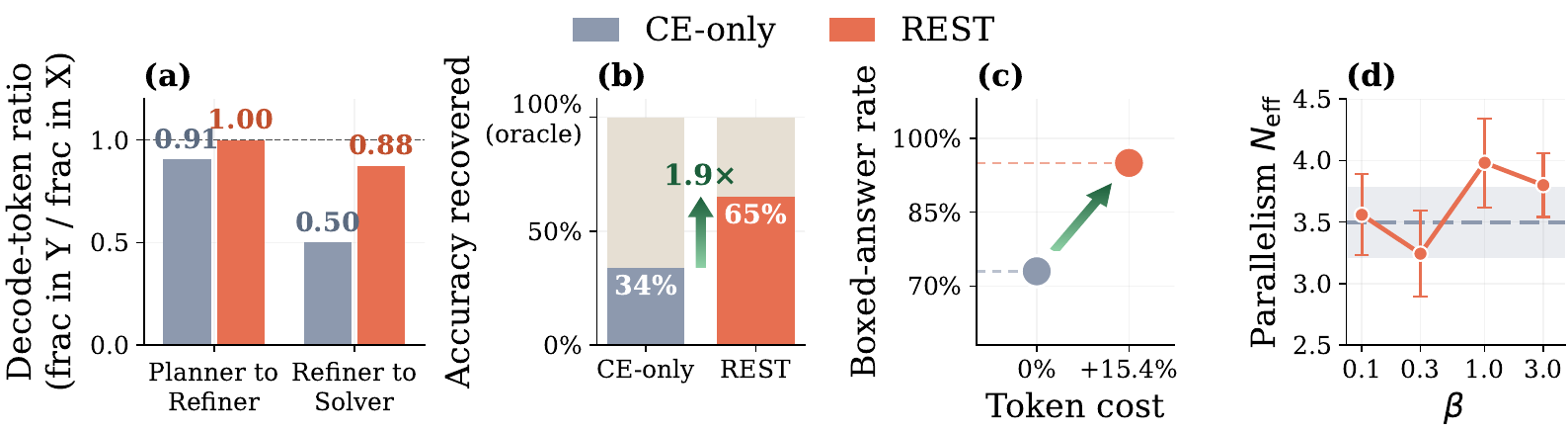}
\caption{\textbf{REST against CE-only.} (a) Decoded thoughts compared to output / input. (b) Replacing oracle text with thought. (c) Answer rate vs token usage. (d) Effective Superposition.}
\label{fig:analysis-panels}
\end{figure}

\begin{figure}[!h]
    \centering
    \includegraphics[width=\linewidth]{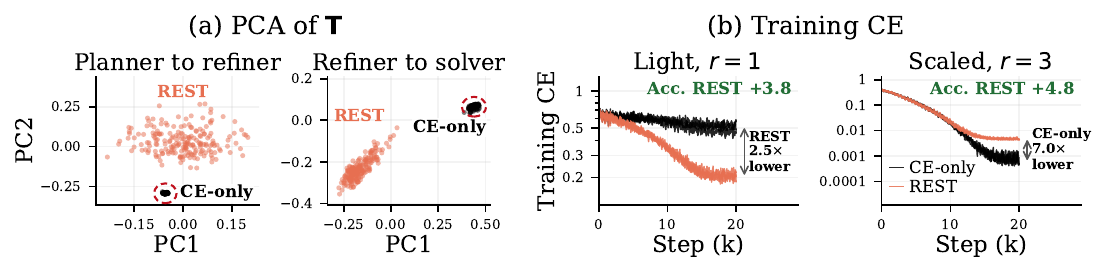}
    \caption{\textbf{REST against CE-only.} (a) PCA for $\mathbf{T}$. CE thoughts collapse into dense clusters, while REST spreads thoughts apart. (b) Training CE under causality. More details in Appendix~\ref{app:results}.}
    \label{fig:separability-pca}
\end{figure}

\textbf{Encoding Producer Output.}
\Cref{fig:analysis-panels}(a) measures how much of a thought's decoded content belongs to the producer's own output, against how much belongs to the prompt it received.
Under REST the thought encodes the producer's plan, or the refined plan. CE has relatively more content from the input encoded, which might be irrelevant to the next agent. This is exactly the behavior targeted by \cref{eq:loss-minimality}.
The effect grows as the system moves further along the chain of latent thoughts.

\textbf{Preserved Information.}
\Cref{fig:analysis-panels}(b) compares the performance of the solver when given the oracle refined plan compared to the thought representation.
REST reaches more of this accuracy than CE. Therefore, REST contains relatively more useful information to the solver than CE.

\textbf{Token Cost.} This information comes at a token cost.
Although we observe an increase in tokens when using REST (+15.4\% on average), this leads to the model being more likely to generate a final answer. We analyze the boxed-answer rate which is the percentage of examples the model converges on an answer. As illustrated in \cref{fig:analysis-panels}(c), REST reaches a 95\% boxed-answer rate compared to 73\% for CE. This indicates that the additional tokens are spent towards reaching a final answer.

\textbf{Superposition.}
\Cref{fig:analysis-panels}(d) measures how many candidate reasoning paths a thought supports at once, a metric we adapt from \citet{deng2025latentsft} (\cref{eq:neff}).
It decodes $\mathbf{T}$ through the consumer's vocabulary, scores each candidate refined plan by the decoded tokens it receives, and reports the exponentiated entropy of the resulting posterior over candidates.
Higher $N_{\text{eff}}$ indicates higher effective superposition.
REST maintains superposition and slightly increases with increasing the weight $\beta$.

\textbf{CE Thoughts.}
\Cref{fig:separability-pca}(a) projects $\mathbf{T}$ for 200 random samples. Thoughts from CE are gathered into dense clusters, hence semantically distinct questions would have close representations that would confuse a consumer agent. After inspection of the CE-only clusters, the texts are not related to each other (\cref{app:results}). Therefore, the clustering represents collapse with no shared content.
REST penalizes this behavior through separability across a continuous region.
This is the failure in \cref{sec:price}, where difficult examples would be indistinguishable from a collapse. Lower training CE does not achieve better accuracy (\cref{fig:separability-pca}(b)).
CE-only plateaus above causality for Light, where it does not memorize, while for Scaled it memorizes the training set and still underperforms REST in accuracy.
Therefore, only minimizing CE would make the thought representation underconstrained.

\section{Related Work}
\label{sec:related}

\begin{wrapfigure}{R}{0.35\linewidth}
\centering
\vspace{-2.0\baselineskip}
\includegraphics[width=\linewidth]{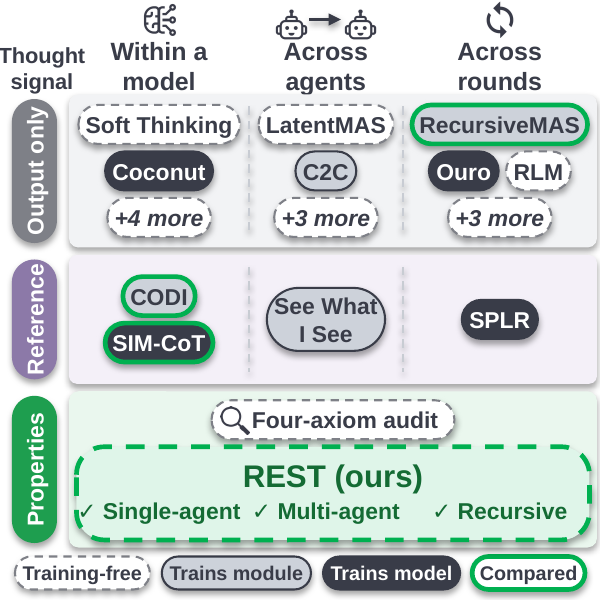}
\caption{Positioning of REST.}
\label{fig:related}
\vspace{-1.0\baselineskip}
\end{wrapfigure}

\textbf{Reasoning in Latent Space.}
Recent work treats the continuous latent space of an LLM as the medium of reasoning rather than the decoded text \citep{chen2026latentcotsurvey, zhu2025latentreasoningsurvey}.
Within a single model, hidden states re-enter the input stream as the next reasoning step \citep{zhang2025softthinking, tan2025latentcompression, butt2026soft, sheshanarayana2026thinking, li2026persistentlatent}.
Across models, hidden representations and KV caches carry information between agents in place of text \citep{liu2024cacheaugmentation, xu2025softchain, zheng2025thoughtcommunication, fu2026cachetocache, ye2025crosscontextkv, zou2026latentcollaboration}.
A further line reuses the same computation over several rounds to deepen latent reasoning \citep{geiping2025recurrentdepth, zhu2026loopedlm, jolicoeurmartineau2025tinyrecursion, bae2025recursionmixture}, which \citet{zou2026recursivemultiagentsystems} raise to the system level by looping heterogeneous agents through a learned link.
Most of these systems rely on the cross-entropy of the final decoded answer for training, and typically leave the latent thoughts themselves unconstrained.

\textbf{Supervising Latent Thoughts.}
The signal that trains a latent state is the answer following the latent steps \citep{hao2026coconut}, a self-distilled explicit chain of thought \citep{shen2025selfdistillcot}, or an auxiliary decoder aligning every latent step with its explicit counterpart \citep{wei2025simcot}.
Each of these signals is defined against a reference the latent state must reproduce, either a text or another model's own cache \citep{chen2026iseeknowi}.
Audits of the resulting representations report high accuracy without latent reasoning, collapsed candidate solutions, and unstable trajectories \citep{cui2026how, rizvi-martel2026the, sahoo2026when}.
A recent framework states four properties a valid thought representation should satisfy, and audits existing methods against them \citep{seddik2026formalizinglatentthoughtsaxioms}.
REST is the first to turn the theoretically motivated properties of a valid thought representation into loss terms for a latent recursive LLM system (Figure~\ref{fig:related}). Additional related work in \cref{app:related}.

\section{Conclusion}
\label{sec:conclusion}

We introduce REST, a training objective that supervises the thoughts of latent recursive LLM systems through four properties of a valid thought representation.
Our objective translates causality, minimality, separability, and stability into differentiable loss terms added to CE.
These terms require no architectural changes or added parameters at inference.
Theoretical analyses show that the failures these properties prevent lead to a lower probability of the correct answer.
Empirically, thoughts trained only with CE collapse across distinct questions and retain irrelevant information.
Evaluations across mathematical, scientific, medical, and code-generation benchmarks demonstrate that REST improves accuracy over CE in both single- and multi-agent settings.
Overall, REST establishes a principled way to supervise latent thoughts beyond the final decoded answer.

\textbf{Implications.}
The improvements from an auxiliary term in both settings indicate that CE-only insufficiently constrains latent thoughts, and that latent recursive systems benefit from supervising the thought representation.
The property that yields the largest gain further depends on the consumer, as causality leads when the thought passes to a different agent and minimality leads when a model recurs on its own thought.
Since the collapse and input content observed under CE are not reflected in the CE objective, final-answer accuracy provides an incomplete view of latent systems, and inspection of the thought representation should accompany it.
Finally, decoded REST thoughts contain relatively more of the producer's output than of its prompt, suggesting that property supervision can make latent communication easier to audit as agents increasingly communicate outside of text.

\textbf{Limitations.}
Due to the substantial computational cost, we do not sweep all combinations of properties and hyperparameters, as the number of required runs grows multiplicatively with each addition.
Nevertheless, causality and minimality improve over CE at every weight we evaluate (\cref{tab:property-sweep}). A cheap approximation to Stability is opted for  due to compute limitations but still leads to performance gain. Another limitation is that REST trains only the outer link while the base LLMs and the inner link remain frozen, a design choice to isolate the contribution of each property term.
Extending REST to the training of the inner link or the agent is left to future work.

\bibliography{references}

\appendix
\newpage
\etocdepthtag.toc{app}
\etocsettagdepth{main}{none}
\etocsettagdepth{app}{subsubsection}
\etocsetnexttocdepth{subsubsection}
\tableofcontents
\newpage

\section{Theoretical Analysis}
\label{app:theory}

\subsection{From Axioms to Differentiable Surrogates}
\label{app:theory:diff}

This subsection derives each differentiable loss term used in our work from the functional properties of \citep{seddik2026formalizinglatentthoughtsaxioms}. Starting from the definition of each property, we arrive at the loss term used in \cref{sec:method}.

\subsubsection{Causality}
\label{app:theory:causality}

Chaining the per-position distributions of \cref{sec:thoughts} yields the distribution that the
consumer assigns to a complete target,
\begin{equation}\label{eq:seq-law}
    P(v \mid Z) = \prod_{t=1}^{n} p\big(v_t \mid v_{<t}, Z\big),
\end{equation}
where truncating the product at $t-1$ gives the distribution $P(v_{<t} \mid Z)$ it assigns to a
prefix.
Only $p(\cdot \mid v_{<t}, \mathbf{T})$ and $P(\cdot \mid \mathbf{T})$ depend on $\psi$, since a
textual transfer consists of embeddings independent of $\psi$ and every $\theta_i$ and
$\mathcal{R}_{\mathrm{in}}$ are frozen.

In the notation of \cref{eq:seq-law}, the Causality axiom requires that both transfers induce the
same distribution over complete targets, and the divergence
\begin{equation}\label{eq:causality-axiom}
    \KL\big(P(\cdot \mid E(u)) \,\big\|\, P(\cdot \mid \mathbf{T})\big)
\end{equation}
quantifies the extent to which they do not.
Evaluating \cref{eq:causality-axiom} ranges over every continuation in $\mathcal{V}^{n}$ and
therefore requires autoregressive sampling, whereas one teacher-forced pass yields a single
continuation at all $n$ positions simultaneously.
\Cref{eq:loss-causality} is \cref{eq:causality-axiom} in differentiable form, and two differences
separate them.
The surrogate averages a per-position divergence, and it takes its prefixes from the dataset.
This is contrary to a divergence over sequences and to the continuations a textual transfer would
produce.
Since we can factorize the divergence, the per-position average does not introduce error.

\begin{proposition}[\textbf{Exact factorization}]\label{prop:causality-factorization}
For every transfer,
\begin{equation}\label{eq:causality-factorization}
    \KL\big(P(\cdot \mid E(u)) \,\big\|\, P(\cdot \mid \mathbf{T})\big)
    = \sum_{t=1}^{n} \mathbb{E}_{v_{<t} \sim P(\cdot \mid E(u))}
      \Big[ \KL\big( p(\cdot \mid v_{<t}, E(u)) \,\big\|\, p(\cdot \mid v_{<t}, \mathbf{T}) \big) \Big].
\end{equation}
\end{proposition}

\begin{proof}[Proof of \cref{prop:causality-factorization}]
Expanding the divergence over complete targets,
\begin{align*}
    \KL\big(P(\cdot \mid E(u)) \,\big\|\, P(\cdot \mid \mathbf{T})\big)
    &= \sum_{v \in \mathcal{V}^{n}} P\big(v \mid E(u)\big)
       \log \frac{P\big(v \mid E(u)\big)}{P\big(v \mid \mathbf{T}\big)} \\
    &\stackrel{\eqref{eq:seq-law}}{=} \sum_{v \in \mathcal{V}^{n}} P\big(v \mid E(u)\big)
       \sum_{t=1}^{n} \log \frac{p\big(v_t \mid v_{<t}, E(u)\big)}{p\big(v_t \mid v_{<t}, \mathbf{T}\big)} ,
\end{align*}
where the second line writes the logarithm of a product as a sum of logarithms.
Both sums are finite, hence exchanging them gives
\begin{equation}\label{eq:causality-exchanged}
    \KL\big(P(\cdot \mid E(u)) \,\big\|\, P(\cdot \mid \mathbf{T})\big)
    = \sum_{t=1}^{n} \; \sum_{v \in \mathcal{V}^{n}} P\big(v \mid E(u)\big)
      \log \frac{p\big(v_t \mid v_{<t}, E(u)\big)}{p\big(v_t \mid v_{<t}, \mathbf{T}\big)} .
\end{equation}
Fix a position $t$ and evaluate its inner sum.
Its summand depends on $v$ only through $v_{\leq t}$.
Therefore, summing over every continuation past $t$ reduces it to a sum over $\mathcal{V}^{t}$.
The prefix factorization
\begin{equation}\label{eq:causality-prefix-split}
    P\big( v_{\leq t} \mid E(u) \big)
    = P\big( v_{<t} \mid E(u) \big) \, p\big( v_t \mid v_{<t}, E(u) \big)
\end{equation}
then splits that sum over the prefix $v_{<t}$ and the final token $v_t = w$,
\begin{align*}
    &\sum_{v \in \mathcal{V}^{n}} P\big(v \mid E(u)\big)
      \log \frac{p\big(v_t \mid v_{<t}, E(u)\big)}{p\big(v_t \mid v_{<t}, \mathbf{T}\big)} \\
    &\qquad = \sum_{v_{\leq t} \in \mathcal{V}^{t}} P\big(v_{\leq t} \mid E(u)\big)
      \log \frac{p\big(v_t \mid v_{<t}, E(u)\big)}{p\big(v_t \mid v_{<t}, \mathbf{T}\big)} \\
    &\qquad = \sum_{v_{<t} \in \mathcal{V}^{t-1}} P\big(v_{<t} \mid E(u)\big)
      \sum_{w \in \mathcal{V}} p\big(w \mid v_{<t}, E(u)\big)
      \log \frac{p\big(w \mid v_{<t}, E(u)\big)}{p\big(w \mid v_{<t}, \mathbf{T}\big)} \\
    &\qquad = \sum_{v_{<t} \in \mathcal{V}^{t-1}} P\big(v_{<t} \mid E(u)\big) \,
      \KL\big( p(\cdot \mid v_{<t}, E(u)) \,\big\|\, p(\cdot \mid v_{<t}, \mathbf{T}) \big) \\
    &\qquad = \mathbb{E}_{v_{<t} \sim P(\cdot \mid E(u))}
      \Big[ \KL\big( p(\cdot \mid v_{<t}, E(u)) \,\big\|\, p(\cdot \mid v_{<t}, \mathbf{T}) \big) \Big] ,
\end{align*}
where the third equality recognizes the inner sum over $w$ as a divergence.
The fourth recognizes the outer sum as an expectation over the prefixes a textual transfer produces.
Substituting this equality into \cref{eq:causality-exchanged} at every position yields
\cref{eq:causality-factorization}.
\end{proof}

The second difference therefore remains, which is the distribution the prefixes $v_{<t}$ of
\cref{eq:causality-factorization} are sampled from.
This is the only step at which the surrogate is not the exact property, since training samples those
prefixes from the dataset rather than from the consumer's own generations.
Two conditions bound its effect.

\begin{assumption}[\textbf{Bounded disagreement}]\label{ass:bounded-disagreement}
There exists a finite $B$ with
\begin{equation}\label{eq:bounded-disagreement}
    \Big| \log p\big( w \mid v_{<t}, E(u) \big) - \log p\big( w \mid v_{<t}, \mathbf{T} \big) \Big|
    \;\leq\; B
\end{equation}
at every position $t$, every prefix $v_{<t}$, every token $w \in \mathcal{V}$, and every $\psi$ in
the optimization domain.
\end{assumption}

Bounding the consumer's logits $z$ in absolute value by $M$ lower-bounds every next-token
probability,
\begin{equation*}
    p\big( w \mid v_{<t}, Z \big)
    = \frac{\exp( z_w )}{\sum_{w' \in \mathcal{V}} \exp( z_{w'} )}
    \;\geq\; \frac{e^{-M}}{|\mathcal{V}| \, e^{M}}
    \;=\; \frac{e^{-2M}}{|\mathcal{V}|} ,
\end{equation*}
hence every log-probability lies in an interval of length
\begin{equation*}
    \log 1 - \log \frac{e^{-2M}}{|\mathcal{V}|} = 2M + \log|\mathcal{V}| .
\end{equation*}
Every difference of two of them therefore satisfies
\begin{equation*}
    \Big| \log p\big( w \mid v_{<t}, E(u) \big) - \log p\big( w \mid v_{<t}, \mathbf{T} \big) \Big|
    \;\leq\; 2M + \log|\mathcal{V}| ,
\end{equation*}
at which \cref{ass:bounded-disagreement} holds with $B = 2M + \log|\mathcal{V}|$.\footnote{This
constant is the worst case over the whole vocabulary.}
\Cref{eq:bounded-disagreement} bounds the corresponding divergence by the same constant,
\begin{equation}\label{eq:bounded-disagreement-kl}
    \KL\big( p(\cdot \mid v_{<t}, E(u)) \,\big\|\, p(\cdot \mid v_{<t}, \mathbf{T}) \big)
    = \sum_{w} p\big( w \mid v_{<t}, E(u) \big)
      \log \frac{p( w \mid v_{<t}, E(u) )}{p( w \mid v_{<t}, \mathbf{T} )}
    \;\leq\; B ,
\end{equation}
where the inequality averages \cref{eq:bounded-disagreement} under
$p(\cdot \mid v_{<t}, E(u))$.

\begin{assumption}[\textbf{Reference faithfulness}]\label{ass:reference-faithfulness}
Let $D$ denote the distribution from which the training data draws the target $v$, conditioned on the same
$E_{\mathrm{pre}}$ and $E_{\mathrm{post}}$ that the consumer receives in \cref{def:transfer}.
There exists $\varepsilon \in [0, 1]$ with $\TV\big(D, P(\cdot \mid E(u))\big) \leq \varepsilon$.
\end{assumption}

\Cref{ass:reference-faithfulness} requires that the consumer, once given the producer's text
explicitly, reproduce what the training data records.
The training trajectory in general comes from a different model than the frozen consumer, and
$\varepsilon$ therefore quantifies a distillation gap between the two.

We here state a standard result on averaging a bounded quantity under two distributions, for
completeness.
It will later let us replace the distribution that weights each position.

\begin{lemma}[\textbf{Total variation controls a bounded average}]\label{lem:tv-averaging}
Averaging a quantity whose range fits in an interval of length $B$ under two distributions $\pi$ and
$\pi'$ on a common finite space alters the value by at most $B$ times their total variation
distance,
\begin{equation}\label{eq:tv-averaging}
    \big| \mathbb{E}_{\pi}[f] - \mathbb{E}_{\pi'}[f] \big|
    \;\leq\; B \cdot \TV(\pi, \pi') .
\end{equation}
\end{lemma}

\begin{proof}[Proof of \cref{lem:tv-averaging}]
Let $c$ denote the midpoint of the interval containing $f$'s range, at which
\begin{equation*}
    \big| f - c \big| \;\leq\; \frac{B}{2} .
\end{equation*}
Then
\begin{equation*}
    \big| \mathbb{E}_{\pi}[f] - \mathbb{E}_{\pi'}[f] \big|
    = \Big| \sum (\pi - \pi')(f - c) \Big|
    \;\leq\; \frac{B}{2} \sum |\pi - \pi'|
    \;=\; B \cdot \TV(\pi, \pi') ,
\end{equation*}
where subtracting $c$ leaves the difference unchanged because $\sum(\pi - \pi') = 0$.
\end{proof}

\begin{proposition}[\textbf{The surrogate tracks the axiom}]\label{prop:causality-tracking}
Under \cref{ass:bounded-disagreement,ass:reference-faithfulness}, for every $\psi$ in the
optimization domain,
the surrogate averaged over the training target,
\begin{equation*}
    \mathbb{E}_{v \sim D}\big[ \mathcal{L}_{\mathrm{caus}}(\psi) \big]
    = \mathbb{E}_{v \sim D}\bigg[ \frac{1}{n} \sum_{t=1}^{n}
      \KL\big( p(\cdot \mid v_{<t}, E(u)) \,\big\|\, p(\cdot \mid v_{<t}, \mathbf{T}) \big) \bigg] ,
\end{equation*}
satisfies
\begin{equation}\label{eq:causality-tracking}
    \Big|\, \mathbb{E}_{v \sim D}\big[ \mathcal{L}_{\mathrm{caus}}(\psi) \big]
    \;-\; \frac{1}{n}\, \KL\big(P(\cdot \mid E(u)) \,\big\|\, P(\cdot \mid \mathbf{T})\big) \Big|
    \;\leq\; B \varepsilon .
\end{equation}
\end{proposition}

\begin{proof}[Proof of \cref{prop:causality-tracking}]
Denote the divergence at a single position by
\begin{equation*}
    \Delta_t(v_{<t})
    = \KL\big( p(\cdot \mid v_{<t}, E(u)) \,\big\|\, p(\cdot \mid v_{<t}, \mathbf{T}) \big)
    \;\in\; [0, B] ,
\end{equation*}
where the lower end holds of every divergence and the upper end follows from
\cref{eq:bounded-disagreement-kl}. In this notation \cref{eq:loss-causality} expands to
\begin{equation*}
    \mathcal{L}_{\mathrm{caus}}(\psi)
    = \frac{1}{n} \sum_{t=1}^{n}
      \KL\big( p(\cdot \mid v_{<t}, E(u)) \,\big\|\, p(\cdot \mid v_{<t}, \mathbf{T}) \big)
    = \frac{1}{n} \sum_{t=1}^{n} \Delta_t(v_{<t}) ,
\end{equation*}
and \cref{prop:causality-factorization} becomes
\begin{equation*}
    \frac{1}{n}\, \KL\big(P(\cdot \mid E(u)) \,\big\|\, P(\cdot \mid \mathbf{T})\big)
    = \frac{1}{n} \sum_{t=1}^{n} \mathbb{E}_{P(\cdot \mid E(u))}\big[ \Delta_t \big] ,
\end{equation*}
where both are expressed through the same $\Delta_t$ and differ only in how each position is
weighted. The two weightings are the distributions of $v_{<t}$ under $D$ and under $P(\cdot \mid E(u))$, each
obtained through the projection
\begin{equation*}
    \pi_t : \mathcal{V}^{n} \to \mathcal{V}^{t-1} ,
    \qquad
    \pi_t(v) = v_{<t} .
\end{equation*}
Projecting cannot increase total variation distance,
\begin{align*}
    &\TV\Big( D \circ \pi_t^{-1}, \; P(\cdot \mid E(u)) \circ \pi_t^{-1} \Big) \\
    &\qquad \;\leq\; \TV\big( D, \, P(\cdot \mid E(u)) \big)
    && \text{data processing} \\
    &\qquad \;\leq\; \varepsilon
    && \text{\cref{ass:reference-faithfulness}} ,
\end{align*}
and \cref{lem:tv-averaging} with $f = \Delta_t$ yields
\begin{equation*}
    \Big| \mathbb{E}_{D}\big[ \Delta_t \big]
        - \mathbb{E}_{P(\cdot \mid E(u))}\big[ \Delta_t \big] \Big|
    \;\leq\; B \varepsilon
    \qquad \text{at every position } t .
\end{equation*}

Summing the $n$ positions and applying the triangle inequality,
\begin{align*}
    &\Big| \, n\, \mathbb{E}_{v \sim D}\big[ \mathcal{L}_{\mathrm{caus}}(\psi) \big]
      - \KL\big(P(\cdot \mid E(u)) \,\big\|\, P(\cdot \mid \mathbf{T})\big) \Big| \\
    &\qquad = \bigg| \sum_{t=1}^{n} \Big( \mathbb{E}_{D}\big[ \Delta_t \big]
      - \mathbb{E}_{P(\cdot \mid E(u))}\big[ \Delta_t \big] \Big) \bigg| \\
    &\qquad \leq \sum_{t=1}^{n} \Big| \mathbb{E}_{D}\big[ \Delta_t \big]
      - \mathbb{E}_{P(\cdot \mid E(u))}\big[ \Delta_t \big] \Big|
    \;\leq\; n B \varepsilon ,
\end{align*}
and dividing by $n$ yields \cref{eq:causality-tracking}.
\end{proof}

The division by $n$ in \cref{eq:loss-causality} renders \cref{eq:causality-tracking} independent of
the target length.
The $n$ per-position bounds accumulate to $n B \varepsilon$, and the normalization then removes the
factor.
Within a transfer that division is a constant free of $\psi$, hence the minimizer is the same either
way.
\Cref{prop:causality-tracking} compares the two quantities, and the following corollary establishes
that minimizing one minimizes the other up to a bounded excess.

\begin{corollary}[\textbf{Surrogate minimizers are near-optimal for the axiom}]\label{cor:causality-minimizer}
Under \cref{ass:bounded-disagreement,ass:reference-faithfulness}, let $\hat{\psi}$ minimize
$\mathbb{E}_{v \sim D}[ \mathcal{L}_{\mathrm{caus}}(\psi) ]$ over the optimization domain and let
$\psi^{\star}$ minimize \cref{eq:causality-axiom} over the same domain.
Then \cref{eq:causality-axiom} at $\hat{\psi}$ exceeds its minimum by at most $2 n B \varepsilon$.
\end{corollary}

\begin{proof}[Proof of \cref{cor:causality-minimizer}]
Denote the two objectives by
\begin{equation*}
    F(\psi) = \frac{1}{n}\, \KL\big(P(\cdot \mid E(u)) \,\big\|\, P(\cdot \mid \mathbf{T})\big),
    \qquad
    G(\psi) = \mathbb{E}_{v \sim D}\big[ \mathcal{L}_{\mathrm{caus}}(\psi) \big] ,
\end{equation*}
where $|F(\psi) - G(\psi)| \leq B \varepsilon$ at every $\psi$ by \cref{prop:causality-tracking}.
Chaining the two objectives at $\hat{\psi}$ and at $\psi^{\star}$,
\begin{align*}
    F(\hat{\psi})
    &\;\leq\; G(\hat{\psi}) + B\varepsilon
    && \text{\cref{prop:causality-tracking}} \\
    &\;\leq\; G(\psi^{\star}) + B\varepsilon
    && \hat{\psi} \text{ minimizes } G \\
    &\;\leq\; F(\psi^{\star}) + 2 B \varepsilon
    && \text{\cref{prop:causality-tracking}} .
\end{align*}
Multiplying by $n$,
\begin{equation*}
    n\, F(\hat{\psi}) \;\leq\; n\, F(\psi^{\star}) + 2 n B \varepsilon ,
\end{equation*}
where $n F(\psi)$ is \cref{eq:causality-axiom}.
\end{proof}

\subsubsection{Minimality}
\label{app:theory:minimality}

The Minimality axiom requires that the thought minimize the Information Bottleneck Lagrangian at a
trade-off weight $\beta_{\mathrm{IB}}$ \citep{seddik2026formalizinglatentthoughtsaxioms}, formed from the mutual
information $I(\cdot\,; \cdot)$. This weight is internal to the Minimality derivation and is distinct
from the $\beta$ of \cref{eq:objective}, which weights a property term against the CE objective.
\begin{equation}\label{eq:minimality-axiom}
    J(\beta_{\mathrm{IB}}) = I\big( X ; \mathbf{T} \big) - \beta_{\mathrm{IB}}\, I\big( \mathbf{T} ; Y \big) ,
\end{equation}
retaining little of the input and much of the output.
Evaluating \cref{eq:minimality-axiom} requires distributions that training would have to estimate,
whereas each cross-entropy of \cref{eq:loss-minimality} follows from one forward pass of the
consumer.
The two weights of \cref{eq:loss-minimality} enter the results below only through
\begin{equation}\label{eq:minimality-beta}
    \beta_{\mathrm{IB}} = 1 + \frac{\lambda_1}{\lambda_2} ,
    \qquad\text{equivalently}\qquad
    \lambda_2 (\beta_{\mathrm{IB}} - 1) = \lambda_1 .
\end{equation}
The ratio $\lambda_1 / \lambda_2$ covers $(0, \infty)$, hence \cref{eq:minimality-beta} covers every
$\beta_{\mathrm{IB}} > 1$.
Therefore, \cref{eq:loss-minimality} represents \cref{eq:minimality-axiom} at every such $\beta_{\mathrm{IB}}$,
with that ratio setting the trade-off weight.
The symmetric weight $\beta_{\mathrm{IB}} = 2$ corresponds to $\lambda_1 = \lambda_2$.

\Cref{eq:loss-minimality} is \cref{eq:minimality-axiom} in differentiable form.
\Cref{prop:minimality-decomposition} rewrites the axiom in entropies,
\cref{prop:minimality-substitution} replaces those entropies by the cross-entropies the consumer
incurs, and the two conditions that follow bound the residual terms of that replacement.

\begin{proposition}[\textbf{Decomposition of the axiom}]\label{prop:minimality-decomposition}
Under \cref{eq:minimality-beta},
\begin{equation}\label{eq:minimality-decomposition}
    \lambda_2 J(\beta_{\mathrm{IB}})
    = \lambda_1 \mathbb{H}\big( Y \mid \mathbf{T} \big)
      - \lambda_2 \mathbb{H}\big( X \mid Y, \mathbf{T} \big)
      - \lambda_2 I\big( Y ; \mathbf{T} \mid X \big)
      + C ,
\end{equation}
where $C = \lambda_2 \mathbb{H}(X \mid Y) - \lambda_1 \mathbb{H}(Y)$ is independent of $\mathbf{T}$.
\end{proposition}

\begin{proof}[Proof of \cref{prop:minimality-decomposition}]
The chain rule for mutual information accumulates the information $\mathbf{T}$ contains about the pair
$(X, Y)$ in either order,
\begin{equation*}
    I\big( X ; \mathbf{T} \big) + I\big( Y ; \mathbf{T} \mid X \big)
    = I\big( \mathbf{T} ; Y \big) + I\big( X ; \mathbf{T} \mid Y \big) ,
\end{equation*}
and solving for $I(X ; \mathbf{T})$ and substituting into \cref{eq:minimality-axiom} gives
\begin{equation*}
    J(\beta_{\mathrm{IB}})
    = (1 - \beta_{\mathrm{IB}})\, I\big( \mathbf{T} ; Y \big)
      + I\big( X ; \mathbf{T} \mid Y \big)
      - I\big( Y ; \mathbf{T} \mid X \big) .
\end{equation*}
Multiplying by $\lambda_2 > 0$ and applying $\lambda_2 (\beta_{\mathrm{IB}} - 1) = \lambda_1$ to the first term,
\begin{equation*}
    \lambda_2 J(\beta_{\mathrm{IB}})
    = - \lambda_1 I\big( \mathbf{T} ; Y \big)
      + \lambda_2 I\big( X ; \mathbf{T} \mid Y \big)
      - \lambda_2 I\big( Y ; \mathbf{T} \mid X \big) .
\end{equation*}
The first two expand into entropies,
\begin{equation*}
    I\big( \mathbf{T} ; Y \big) = \mathbb{H}(Y) - \mathbb{H}\big( Y \mid \mathbf{T} \big) ,
    \qquad
    I\big( X ; \mathbf{T} \mid Y \big)
    = \mathbb{H}\big( X \mid Y \big) - \mathbb{H}\big( X \mid Y, \mathbf{T} \big) ,
\end{equation*}
and substituting them gives
\begin{equation*}
    \lambda_2 J(\beta_{\mathrm{IB}})
    = \lambda_1 \mathbb{H}\big( Y \mid \mathbf{T} \big)
      - \lambda_2 \mathbb{H}\big( X \mid Y, \mathbf{T} \big)
      - \lambda_2 I\big( Y ; \mathbf{T} \mid X \big)
      + \underbrace{\lambda_2 \mathbb{H}\big( X \mid Y \big) - \lambda_1 \mathbb{H}(Y)}_{C} .
\end{equation*}
\end{proof}

The axiom's original decomposition discards $I(Y ; \mathbf{T} \mid X)$, which equals zero when
$\mathbf{T}$ is only a function of $X$.
\Cref{prop:minimality-decomposition} retains it, since
$\mathbf{T} = \mathcal{R}_\psi\big(\mathcal{R}_{\mathrm{in}}(H_u)\big)$ takes the
hidden states of a pass over the producer's output as its input and therefore depends on $Y$.

\begin{proposition}[\textbf{The axiom in terms of the loss}]\label{prop:minimality-substitution}
Let
$\gamma_1 = \mathrm{CE}(Y \mid \mathbf{T}) - \mathbb{H}(Y \mid \mathbf{T})$ and
$\gamma_2 = \mathrm{CE}(X \mid Y, \mathbf{T}) - \mathbb{H}(X \mid Y, \mathbf{T})$
denote the two decoder gaps.
Then under \cref{eq:minimality-beta},
\begin{align}\label{eq:minimality-substitution}
    \lambda_2 J(\beta_{\mathrm{IB}})
    &= \underbrace{
         \lambda_1 \mathrm{CE}\big( Y \mid \mathbf{T} \big)
         - \lambda_2 \mathrm{CE}\big( X \mid Y, \mathbf{T} \big)
       }_{\mathcal{L}_{\mathrm{min}}(\psi) \text{ of } \cref{eq:loss-minimality}}
       \;+\; C \nonumber \\
    &\quad - \lambda_2 I\big( Y ; \mathbf{T} \mid X \big)
       - \big( \lambda_1 \gamma_1 - \lambda_2 \gamma_2 \big) .
\end{align}
\end{proposition}

\begin{proof}[Proof of \cref{prop:minimality-substitution}]
Gibbs' inequality lower-bounds a cross-entropy by the entropy it approximates,
\begin{equation*}
    \mathrm{CE}\big( Y \mid \mathbf{T} \big) \geq \mathbb{H}\big( Y \mid \mathbf{T} \big) ,
    \qquad
    \mathrm{CE}\big( X \mid Y, \mathbf{T} \big)
    \geq \mathbb{H}\big( X \mid Y, \mathbf{T} \big) ,
\end{equation*}
hence $\gamma_1 \geq 0$ and $\gamma_2 \geq 0$, and their definitions rearrange to
\begin{equation*}
    \mathbb{H}\big( Y \mid \mathbf{T} \big) = \mathrm{CE}\big( Y \mid \mathbf{T} \big) - \gamma_1 ,
    \qquad
    \mathbb{H}\big( X \mid Y, \mathbf{T} \big)
    = \mathrm{CE}\big( X \mid Y, \mathbf{T} \big) - \gamma_2 .
\end{equation*}
Substituting both into the entropy pair of \cref{eq:minimality-decomposition},
\begin{align*}
    \lambda_1 \mathbb{H}\big( Y \mid \mathbf{T} \big)
      - \lambda_2 \mathbb{H}\big( X \mid Y, \mathbf{T} \big)
    &= \lambda_1 \Big( \mathrm{CE}\big( Y \mid \mathbf{T} \big) - \gamma_1 \Big)
       - \lambda_2 \Big( \mathrm{CE}\big( X \mid Y, \mathbf{T} \big) - \gamma_2 \Big) \\
    &= \lambda_1 \mathrm{CE}\big( Y \mid \mathbf{T} \big)
       - \lambda_2 \mathrm{CE}\big( X \mid Y, \mathbf{T} \big)
       - \big( \lambda_1 \gamma_1 - \lambda_2 \gamma_2 \big) .
\end{align*}
\Cref{eq:minimality-decomposition} gives the same entropy pair as
\begin{equation*}
    \lambda_1 \mathbb{H}\big( Y \mid \mathbf{T} \big)
      - \lambda_2 \mathbb{H}\big( X \mid Y, \mathbf{T} \big)
    = \lambda_2 J(\beta_{\mathrm{IB}}) + \lambda_2 I\big( Y ; \mathbf{T} \mid X \big) - C ,
\end{equation*}
and equating the two right sides and solving for $\lambda_2 J(\beta_{\mathrm{IB}})$ gives
\cref{eq:minimality-substitution}.
\end{proof}

\Cref{eq:minimality-substitution} is exact, and two conditions bound the terms by
which \cref{eq:loss-minimality} exceeds the axiom's objective.

\begin{assumption}[\textbf{Bounded decoder gap}]\label{ass:decoder-gap}
There exists a finite $\Gamma$ with
$| \lambda_1 \gamma_1 - \lambda_2 \gamma_2 | \leq \Gamma$ at every $\psi$ in the optimization domain.
\end{assumption}

Gibbs' inequality constrains the first gap,
\begin{equation*}
    \gamma_1 = \mathrm{CE}\big( Y \mid \mathbf{T} \big) - \mathbb{H}\big( Y \mid \mathbf{T} \big)
    \;\geq\; 0 ,
\end{equation*}
and minimizing \cref{eq:loss-minimality} decreases $\gamma_1$ toward zero once the decoder
approaches the true conditional distribution of $Y$ given $\mathbf{T}$.
The second gap admits no such bound.
The coefficient of $\mathrm{CE}(X \mid Y, \mathbf{T})$ in \cref{eq:loss-minimality} is $-\lambda_2$,
hence minimizing the surrogate increases that cross-entropy.
Any upper bound on $\gamma_2$ derived from that cross-entropy increases as well, while
$\mathbb{H}(X \mid Y, \mathbf{T})$ remains unconstrained.
\Cref{ass:decoder-gap} therefore requires the consumer to decode with comparable quality under both
prompts of \cref{eq:loss-minimality}.
Rearranging it,
\begin{equation*}
    \gamma_1 \;\geq\; \frac{\lambda_2 \gamma_2 - \Gamma}{\lambda_1} ,
\end{equation*}
constrains $\gamma_2$ by $\gamma_1$ at every $\psi$.
\Cref{eq:minimality-beta} turns the ordering of the two weights into an ordering of $\beta_{\mathrm{IB}}$,
\begin{equation*}
    \lambda_2 < \lambda_1
    \quad\Longleftrightarrow\quad
    \frac{\lambda_1}{\lambda_2} > 1
    \quad\Longleftrightarrow\quad
    \beta_{\mathrm{IB}} > 2 ,
\end{equation*}
at which the constrained direction has the larger weight.

\begin{assumption}[\textbf{Bounded realization leakage}]\label{ass:realization-leakage}
There exists a finite $\kappa$ with
$I(Y ; \mathbf{T} \mid X) \leq \kappa$ at every $\psi$ in the optimization domain.
\end{assumption}

\Cref{ass:realization-leakage} bounds what $\mathbf{T}$ retains of the particular output beyond what
the input explains, and it holds with $\kappa = 0$ when $\mathbf{T}$ is only a
function of $X$.
Solving \cref{eq:minimality-substitution} for the surrogate,
\begin{equation}\label{eq:minimality-surrogate-solved}
    \mathcal{L}_{\mathrm{min}}(\psi)
    = \lambda_2 J(\beta_{\mathrm{IB}}) - C
      + \lambda_2 I\big( Y ; \mathbf{T} \mid X \big)
      + \big( \lambda_1 \gamma_1 - \lambda_2 \gamma_2 \big) ,
\end{equation}
gives the term a positive coefficient, hence minimizing \cref{eq:loss-minimality} reduces it
alongside the axiom's objective rather than opposed to it.

\begin{proposition}[\textbf{The surrogate tracks the axiom}]\label{prop:minimality-tracking}
Under \cref{ass:decoder-gap,ass:realization-leakage}, at every $\psi$ in the optimization domain,
\begin{equation}\label{eq:minimality-tracking}
    \Big| \, \lambda_2 J(\beta_{\mathrm{IB}}) - \big( \mathcal{L}_{\mathrm{min}}(\psi) + C \big) \Big|
    \;\leq\; \Gamma + \lambda_2 \kappa .
\end{equation}
\end{proposition}

\begin{proof}[Proof of \cref{prop:minimality-tracking}]
Rearranging \cref{eq:minimality-substitution},
\begin{equation*}
    \lambda_2 J(\beta_{\mathrm{IB}}) - \big( \mathcal{L}_{\mathrm{min}}(\psi) + C \big)
    = - \lambda_2 I\big( Y ; \mathbf{T} \mid X \big)
      - \big( \lambda_1 \gamma_1 - \lambda_2 \gamma_2 \big) ,
\end{equation*}
and bounding the two terms in turn,
\begin{align*}
    &\Big| \, \lambda_2 J(\beta_{\mathrm{IB}}) - \big( \mathcal{L}_{\mathrm{min}}(\psi) + C \big) \Big| \\
    &\qquad \;\leq\; \lambda_2 \Big| I\big( Y ; \mathbf{T} \mid X \big) \Big|
      + \Big| \lambda_1 \gamma_1 - \lambda_2 \gamma_2 \Big|
    && \text{triangle inequality} \\
    &\qquad \;\leq\; \lambda_2 \kappa + \Gamma
    && \text{\cref{ass:decoder-gap,ass:realization-leakage}} .
\end{align*}
\end{proof}

\Cref{prop:minimality-tracking} compares the two quantities, and the following corollary establishes
that minimizing one minimizes the other up to a bounded excess.

\begin{corollary}[\textbf{Surrogate minimizers are near-optimal for the axiom}]\label{cor:minimality-minimizer}
Under \cref{ass:decoder-gap,ass:realization-leakage}, let $\hat{\psi}$ minimize
\cref{eq:loss-minimality} over the optimization domain and let $\psi^{\star}$ minimize
\cref{eq:minimality-axiom} over the same domain.
Then
\begin{equation}\label{eq:minimality-excess}
    J(\beta_{\mathrm{IB}})\big|_{\hat{\psi}}
    \;\leq\; J(\beta_{\mathrm{IB}})\big|_{\psi^{\star}} + \frac{2 (\Gamma + \lambda_2 \kappa)}{\lambda_2} .
\end{equation}
\end{corollary}

\begin{proof}[Proof of \cref{cor:minimality-minimizer}]
Denote the two objectives by
\begin{equation*}
    F(\psi) = \lambda_2 J(\beta_{\mathrm{IB}}) ,
    \qquad
    G(\psi) = \lambda_1 \mathrm{CE}\big( Y \mid \mathbf{T} \big)
      - \lambda_2 \mathrm{CE}\big( X \mid Y, \mathbf{T} \big) + C ,
\end{equation*}
where $|F(\psi) - G(\psi)| \leq \Gamma + \lambda_2 \kappa$ at every $\psi$ by
\cref{prop:minimality-tracking}.
Since $C$ and $\lambda_2$ are constant in $\psi$, $\hat{\psi}$ minimizes $G$ and $\psi^{\star}$
minimizes $F$.
Similar to how we showed \cref{cor:causality-minimizer}, we chain the two objectives at $\hat{\psi}$
and at $\psi^{\star}$ to obtain
\begin{equation*}
    F(\hat{\psi}) \;\leq\; F(\psi^{\star}) + 2 (\Gamma + \lambda_2 \kappa) .
\end{equation*}
Substituting $F(\psi) = \lambda_2 J(\beta_{\mathrm{IB}})$ and dividing by $\lambda_2 > 0$,
\begin{equation*}
    J(\beta_{\mathrm{IB}})\big|_{\hat{\psi}}
    \;\leq\; J(\beta_{\mathrm{IB}})\big|_{\psi^{\star}}
      + \frac{2 (\Gamma + \lambda_2 \kappa)}{\lambda_2} .
\end{equation*}
\end{proof}

\subsubsection{Separability}
\label{app:theory:separability}

The Separability axiom requires a projection of bounded capacity under which the thoughts of two
semantically disjoint inputs stay apart,
\begin{equation}\label{eq:separability-axiom}
    d\big( \varphi(\mathbf{T}), \varphi(\mathbf{T}') \big) > \delta
    \qquad \text{for some } \varphi \in \mathcal{H} ,
\end{equation}
where $\mathcal{H}$ is the class of admissible projections and $d$ is a metric on the semantic space.
Normalizing the pooled thought of \cref{eq:loss-separability} to unit length,
\begin{equation}\label{eq:sphere-map}
    \varphi_\omega(\mathbf{T}) = \mathbf{t} \big/ \|\mathbf{t}\| ,
\end{equation}
gives a member of such a class.
Expanding the squared distance between two of its values,
\begin{align}
    \big\| \varphi_\omega(\mathbf{T}) - \varphi_\omega(\mathbf{T}') \big\|^{2}
    &= \big\| \varphi_\omega(\mathbf{T}) \big\|^{2}
       + \big\| \varphi_\omega(\mathbf{T}') \big\|^{2}
       - 2 \big\langle \varphi_\omega(\mathbf{T}), \, \varphi_\omega(\mathbf{T}') \big\rangle
       \nonumber \\
    &= 2 - 2\, \frac{\big\langle \mathbf{t}, \, \mathbf{t}' \big\rangle}
       {\|\mathbf{t}\| \, \|\mathbf{t}'\|} \nonumber \\
    &= 2 - 2\, \mathrm{sim}\big( \mathbf{t}, \mathbf{t}' \big) ,
    \label{eq:sphere-identity}
\end{align}
where the second equality applies the unit norm that \cref{eq:sphere-map} imposes on both values.
We adopt the distance of \cref{eq:sphere-identity} as the metric $d$ of
\cref{eq:separability-axiom}, following the
realization of the semantic metric by cosine similarity in
\citet{seddik2026formalizinglatentthoughtsaxioms}.
\Cref{eq:loss-separability} is \cref{eq:separability-axiom} in differentiable form.
The axiom asserts an inequality, hence the two are linked by a threshold, and the axiom is valid
whenever the surrogate falls below it.

\begin{proposition}[\textbf{Lower bound on the margin}]\label{prop:separability-margin}
For every $\psi$, every $\omega$, and every preceding thought $\mathbf{T}_k$ with pooled vector
$\mathbf{t}_k$, $k \in \{1, \dots, K\}$, as in \cref{sec:thoughts},
\begin{equation}\label{eq:separability-margin}
    \big\| \varphi_\omega(\mathbf{T}) - \varphi_\omega(\mathbf{T}_k) \big\|^{2}
    \;\geq\; 2 - 2 \tau\, \mathcal{L}_{\mathrm{sep}}(\psi, \omega)
    \qquad \text{at every } k \in \{1, \dots, K\},\ K \geq 1 .
\end{equation}
\end{proposition}

\begin{proof}[Proof of \cref{prop:separability-margin}]
Every term of the sum in \cref{eq:loss-separability} is positive, hence each single term is at most
the whole sum,
\begin{equation*}
    \exp\!\Big( \mathrm{sim}\big( \mathbf{t}, \mathbf{t}_k \big) \big/ \tau \Big)
    \;\leq\; \sum_{j=1}^{K}
      \exp\!\Big( \mathrm{sim}\big( \mathbf{t}, \mathbf{t}_j \big) \big/ \tau \Big)
    \qquad \text{at every } k .
\end{equation*}
Taking logarithms,
\begin{equation*}
    \frac{1}{\tau}\, \mathrm{sim}\big( \mathbf{t}, \mathbf{t}_k \big)
    \;\leq\; \log \sum_{j=1}^{K}
      \exp\!\Big( \mathrm{sim}\big( \mathbf{t}, \mathbf{t}_j \big) \big/ \tau \Big)
    \;=\; \mathcal{L}_{\mathrm{sep}}(\psi, \omega) ,
\end{equation*}
and multiplying by $\tau > 0$,
\begin{equation}\label{eq:separability-sim-bound}
    \mathrm{sim}\big( \mathbf{t}, \mathbf{t}_k \big)
    \;\leq\; \tau\, \mathcal{L}_{\mathrm{sep}}(\psi, \omega)
    \qquad \text{at every } k .
\end{equation}
Substituting \cref{eq:separability-sim-bound} into \cref{eq:sphere-identity},
\begin{equation*}
    \big\| \varphi_\omega(\mathbf{T}) - \varphi_\omega(\mathbf{T}_k) \big\|^{2}
    = 2 - 2\, \mathrm{sim}\big( \mathbf{t}, \mathbf{t}_k \big)
    \;\geq\; 2 - 2 \tau\, \mathcal{L}_{\mathrm{sep}}(\psi, \omega) ,
\end{equation*}
which is \cref{eq:separability-margin}.
\end{proof}

\Cref{eq:separability-sim-bound} holds at all $K$ indices at once, hence the surrogate controls the
largest similarity rather than their average.
The right side of \cref{eq:separability-margin} is positive if and only if
\begin{equation*}
    \tau\, \mathcal{L}_{\mathrm{sep}}(\psi, \omega) \;<\; 1 .
\end{equation*}
\Cref{eq:separability-axiom} requires that margin only of inputs whose semantic supports are
disjoint, and one condition supplies that precondition at the pairs the surrogate acts on.

\begin{assumption}[\textbf{Distinct inputs are semantically disjoint}]\label{ass:preceding-disjointness}
The $K$ preceding thoughts originate from training examples distinct from the current one, and two
distinct examples induce disjoint semantic supports.
\end{assumption}

The semantic support of an example covers its whole reasoning trajectory together with its
conclusion, hence two distinct inputs violate \cref{ass:preceding-disjointness} only when both their
trajectories and their conclusions agree.
\Cref{ass:preceding-disjointness} makes every pair the surrogate acts on semantically disjoint, hence
each pair falls under the clause proved here rather than under the axiom's converse clause, which
applies to semantically convergent inputs.

\begin{corollary}[\textbf{Sufficient condition for the axiom's margin}]\label{cor:separability-axiom}
Under \cref{ass:preceding-disjointness}, at every $\delta \in (0, 2]$,
\begin{equation}\label{eq:separability-threshold}
    \underbrace{
      \log \sum_{k=1}^{K}
        \exp\!\Big( \mathrm{sim}\big( \mathbf{t}, \mathbf{t}_k \big) \big/ \tau \Big)
    }_{\mathcal{L}_{\mathrm{sep}}(\psi, \omega) \text{ of } \cref{eq:loss-separability}}
    \;\leq\; \frac{2 - \delta^{2}}{2 \tau}
    \qquad\Longrightarrow\qquad
    \big\| \varphi_\omega(\mathbf{T}) - \varphi_\omega(\mathbf{T}_k) \big\| \;\geq\; \delta
    \quad \text{at every } k ,
\end{equation}
and \cref{eq:separability-axiom} therefore holds at every $\delta' \in (0, \delta)$ with $\varphi = \varphi_\omega$.
\end{corollary}

\begin{proof}[Proof of \cref{cor:separability-axiom}]
Multiplying the hypothesis of \cref{eq:separability-threshold} by $2\tau$ and rearranging gives
\begin{equation*}
    2 - 2 \tau\, \mathcal{L}_{\mathrm{sep}}(\psi, \omega) \;\geq\; \delta^{2} ,
\end{equation*}
and chaining this with \cref{prop:separability-margin} lower-bounds every squared distance by
$\delta^{2}$,
\begin{equation*}
    \big\| \varphi_\omega(\mathbf{T}) - \varphi_\omega(\mathbf{T}_k) \big\|^{2}
    \;\geq\; 2 - 2 \tau\, \mathcal{L}_{\mathrm{sep}}(\psi, \omega)
    \;\geq\; \delta^{2}
    \qquad \text{at every } k .
\end{equation*}
Taking square roots yields the conclusion of \cref{eq:separability-threshold}.
\Cref{ass:preceding-disjointness} supplies the precondition of \cref{eq:separability-axiom} at each
pair, and $\varphi_\omega$ belongs to $\mathcal{H}$ by \cref{eq:sphere-map}.
\end{proof}

\Cref{cor:separability-axiom} states the axiom's own requirement as a threshold on the surrogate,
hence minimizing \cref{eq:loss-separability} maximizes the margin obtainable from $\mathcal{H}$.

\subsubsection{Stability}
\label{app:theory:stability}

Chaining the per-position distributions of \cref{eq:entropy-rate-estimate} yields the distribution
the producer assigns to a complete output,
\begin{equation}\label{eq:producer-law}
    Q(u) = \prod_{s=1}^{m} q\big(u_s \mid u_{<s}\big) ,
\end{equation}
and $Q$ is the distribution of the producer's output $Y$.
The axiom comprises a lexical-invariance clause and a mode-collapse clause, and
\cref{eq:loss-stability} targets the latter, which requires that $\mathbf{T}$ encode the entropy
\begin{equation}\label{eq:stability-axiom}
    \mathbb{H}(Y) = - \sum_{u \in \mathcal{V}^{m}} Q(u) \log Q(u)
\end{equation}
rather than a single realization drawn from $Q$.
Evaluating \cref{eq:stability-axiom} ranges over every output in $\mathcal{V}^{m}$ and therefore
requires autoregressive sampling, whereas one teacher-forced pass yields all $m$ next-token
distributions along a single output.
\Cref{eq:loss-stability} is \cref{eq:stability-axiom} in differentiable form, and two differences
separate its target \cref{eq:entropy-rate-estimate} from the axiom.
The estimate averages per-position entropies along one output, and it evaluates those entropies
along the output in the dataset.
This is contrary to an entropy over sequences and to outputs from the producer itself.
Since we can factorize the entropy, the per-position average does not introduce error.

\begin{proposition}[\textbf{Exact factorization}]\label{prop:stability-factorization}
For every producer,
\begin{equation}\label{eq:stability-factorization}
    \mathbb{H}(Y)
    = \sum_{s=1}^{m} \mathbb{E}_{u_{<s} \sim Q}
      \Big[ \mathbb{H}\big( q(\cdot \mid u_{<s}) \big) \Big].
\end{equation}
\end{proposition}

\begin{proof}[Proof of \cref{prop:stability-factorization}]
Similar to how we showed \cref{prop:causality-factorization}, we substitute \cref{eq:producer-law},
exchange the two finite sums, and split each inner sum over the prefix $u_{<s}$ and the final token
$u_s = w$.
The inner sum over $w$ is an entropy and the outer sum is an expectation over the prefixes the
producer assigns, hence at every position $s$ we can obtain
\begin{equation*}
    - \sum_{u \in \mathcal{V}^{m}} Q(u) \log q\big(u_s \mid u_{<s}\big)
    = \mathbb{E}_{u_{<s} \sim Q}
      \Big[ \mathbb{H}\big( q(\cdot \mid u_{<s}) \big) \Big] .
\end{equation*}
Summing over the $m$ positions yields \cref{eq:stability-factorization}.
\end{proof}

The second difference therefore remains, which is the distribution the prefixes $u_{<s}$ of
\cref{eq:stability-factorization} are drawn from.

\begin{proposition}[\textbf{Unbiased estimate}]\label{prop:stability-unbiased}
If $u \sim Q$, then
\begin{equation}\label{eq:stability-unbiased}
    \mathbb{E}_{u \sim Q}\Big[ \widehat{\mathbb{H}}(u) \Big]
    = \frac{1}{m}\, \mathbb{H}(Y) .
\end{equation}
\end{proposition}

\begin{proof}[Proof of \cref{prop:stability-unbiased}]
Marginalizing $Q$ over every continuation past $s-1$,
\begin{equation*}
    \sum_{u_{\geq s} \in \mathcal{V}^{m-s+1}} Q(u) = Q\big( u_{<s} \big) ,
\end{equation*}
hence $u_{<s} \sim Q$ whenever $u \sim Q$, and
\begin{equation*}
    \mathbb{E}_{u \sim Q}\Big[ \mathbb{H}\big( q(\cdot \mid u_{<s}) \big) \Big]
    = \mathbb{E}_{u_{<s} \sim Q}\Big[ \mathbb{H}\big( q(\cdot \mid u_{<s}) \big) \Big]
    \qquad \text{at every position } s .
\end{equation*}
Averaging these $m$ equalities and applying \cref{prop:stability-factorization},
\begin{equation*}
    \mathbb{E}_{u \sim Q}\Big[ \widehat{\mathbb{H}}(u) \Big]
    = \frac{1}{m} \sum_{s=1}^{m} \mathbb{E}_{u_{<s} \sim Q}
      \Big[ \mathbb{H}\big( q(\cdot \mid u_{<s}) \big) \Big]
    = \frac{1}{m}\, \mathbb{H}(Y) .
\end{equation*}
\end{proof}

A single output supplies one prefix at every position, and \cref{prop:stability-unbiased} states
that these $m$ prefixes already follow the distribution \cref{eq:stability-factorization} requires.
No enumeration of $\mathcal{V}^{m}$ is needed, and the second difference reduces to the origin of
that single output.

\begin{assumption}[\textbf{Producer reference faithfulness}]\label{ass:producer-faithfulness}
Let $D_{\mathrm{prod}}$ denote the distribution from which the training data draws the producer's output
$u$, conditioned on the producer's prompt.
There exists $\varepsilon_{\mathrm{prod}} \in [0, 1]$ with
$\TV\big(D_{\mathrm{prod}}, Q\big) \leq \varepsilon_{\mathrm{prod}}$.
\end{assumption}

\Cref{ass:producer-faithfulness} transposes \cref{ass:reference-faithfulness} from the consumer to
the producer, requiring that the producer, given its own prompt, reproduce what the training data records
at its position of the pipeline.

\begin{proposition}[\textbf{The estimate tracks the axiom}]\label{prop:stability-tracking}
Under \cref{ass:producer-faithfulness},
\begin{equation}\label{eq:stability-tracking}
    \Big|\, \mathbb{E}_{u \sim D_{\mathrm{prod}}}\Big[ \widehat{\mathbb{H}}(u) \Big]
    \;-\; \frac{1}{m}\, \mathbb{H}(Y) \Big|
    \;\leq\; \varepsilon_{\mathrm{prod}} \log|\mathcal{V}| .
\end{equation}
\end{proposition}

\begin{proof}[Proof of \cref{prop:stability-tracking}]
The entropy of a distribution on $\mathcal{V}$ takes values in $[0, \log|\mathcal{V}|]$.
Averaging $m$ of them in \cref{eq:entropy-rate-estimate},
\begin{equation*}
    0 \;\leq\; \widehat{\mathbb{H}}(u)
    = \frac{1}{m} \sum_{s=1}^{m} \mathbb{H}\big( q(\cdot \mid u_{<s}) \big)
    \;\leq\; \frac{1}{m} \sum_{s=1}^{m} \log|\mathcal{V}|
    \;=\; \log|\mathcal{V}| ,
\end{equation*}
gives $\widehat{\mathbb{H}} : \mathcal{V}^{m} \to [0, \log|\mathcal{V}|]$.
Therefore, \cref{lem:tv-averaging} applied to $\widehat{\mathbb{H}}$ and to the two distributions
$D_{\mathrm{prod}}$ and $Q$ on $\mathcal{V}^{m}$ gives
\begin{equation*}
    \Big| \mathbb{E}_{D_{\mathrm{prod}}}\big[ \widehat{\mathbb{H}} \big]
        - \mathbb{E}_{Q}\big[ \widehat{\mathbb{H}} \big] \Big|
    \;\leq\; \log|\mathcal{V}| \cdot \TV\big(D_{\mathrm{prod}}, Q\big)
    \;\leq\; \varepsilon_{\mathrm{prod}} \log|\mathcal{V}| ,
\end{equation*}
and \cref{prop:stability-unbiased} identifies the second expectation as
$\mathbb{H}(Y) / m$.
\end{proof}

The division by $m$ in \cref{eq:entropy-rate-estimate} renders \cref{eq:stability-tracking}
independent of the output length and makes the target an entropy rate.
Without it the target would grow with $m$, and a long output the producer was certain about would be
indistinguishable from a short one it was uncertain about.

\begin{corollary}[\textbf{The surrogate regresses onto the axiom}]\label{cor:stability-target}
Let $\eta(u) = \widehat{\mathbb{H}}(u) - \mathbb{H}(Y) / m$, under which \cref{eq:loss-stability}
takes the form
\begin{align}\label{eq:stability-substituted}
    \mathcal{L}_{\mathrm{stab}}(\psi, \omega)
    &= \Big( g_\omega(\mathbf{T}) - \widehat{\mathbb{H}}(u) \Big)^{2} \nonumber \\
    &= \Big( g_\omega(\mathbf{T}) - \frac{1}{m}\, \mathbb{H}(Y) - \eta(u) \Big)^{2} .
\end{align}
Then $\eta$ is independent of $\psi$ and $\omega$, and under
\cref{ass:producer-faithfulness}
\begin{equation*}
    \mathbb{E}_{u \sim Q}\big[ \eta(u) \big] = 0 ,
    \qquad
    \Big| \mathbb{E}_{u \sim D_{\mathrm{prod}}}\big[ \eta(u) \big] \Big|
    \leq \varepsilon_{\mathrm{prod}} \log|\mathcal{V}| .
\end{equation*}
Consequently every $(\psi, \omega)$ at which \cref{eq:loss-stability} equals zero recovers
$\mathbb{H}(Y) / m$ from $\mathbf{T}$ up to $|\eta(u)|$.
\end{corollary}

\begin{proof}[Proof of \cref{cor:stability-target}]
Substituting $\widehat{\mathbb{H}}(u) = \mathbb{H}(Y) / m + \eta(u)$ into
\cref{eq:loss-stability} yields \cref{eq:stability-substituted}.
Since $\eta$ depends only on the frozen producer and on $u$, it depends on neither $\psi$ nor
$\omega$.
Subtracting $\mathbb{H}(Y) / m$ from \cref{eq:stability-unbiased} gives
\begin{equation*}
    \mathbb{E}_{u \sim Q}\big[ \eta(u) \big] = 0 ,
\end{equation*}
and the same subtraction inside \cref{eq:stability-tracking} gives
\begin{equation*}
    \Big| \mathbb{E}_{u \sim D_{\mathrm{prod}}}\big[ \eta(u) \big] \Big|
    \leq \varepsilon_{\mathrm{prod}} \log|\mathcal{V}| .
\end{equation*}
Finally the right side of \cref{eq:stability-substituted} equals zero if and only if
\begin{equation*}
    g_\omega(\mathbf{T}) = \frac{1}{m}\, \mathbb{H}(Y) + \eta(u) .
\end{equation*}
\end{proof}

\Cref{cor:stability-target} shows that a single term accounts for all of the imprecision remaining
in \cref{eq:loss-stability}.
That term, $\eta(u)$, comes from the one output that was actually sampled, and it does not depend
on $\psi$ or $\omega$.
\Cref{ass:producer-faithfulness} keeps its average under the training distribution within
$\varepsilon_{\mathrm{prod}} \log|\mathcal{V}|$ of zero.
Therefore at $\mathcal{L}_{\mathrm{stab}}(\psi, \omega) = 0$ the probe $g_\omega$ recovers the
entropy rate $\mathbb{H}(Y) / m$ from $\mathbf{T}$.

\subsection{Effect on the Consumer's Cross-Entropy}
\label{app:ce}

This subsection connects each loss term stated in \cref{sec:method} to the empirical performance gains observed in \cref{sec:experiments}.

\subsubsection{Causality}
\label{app:ce:causality}

Let $\mathrm{CE}_{\mathrm{lat}}$ and $\mathrm{CE}_{\mathrm{txt}}$ denote the per-token
cross-entropy that the consumer incurs on the target $v$ under the two transfers of
\cref{def:transfer},
\begin{equation}\label{eq:ce-pair}
    \mathrm{CE}_{\mathrm{lat}}(v) = - \frac{1}{n} \log P\big( v \mid \mathbf{T} \big) ,
    \qquad
    \mathrm{CE}_{\mathrm{txt}}(v) = - \frac{1}{n} \log P\big( v \mid E(u) \big) ,
\end{equation}
where \cref{eq:seq-law} supplies each at the corresponding block $Z$.
The difference between them is the cost of substituting the thought for the producer's text,
\begin{equation}\label{eq:ce-causality-gap}
    \Lambda(v)
    \;\triangleq\; \mathrm{CE}_{\mathrm{lat}}(v) - \mathrm{CE}_{\mathrm{txt}}(v)
    \;=\; \frac{1}{n} \log \frac{P\big( v \mid E(u) \big)}{P\big( v \mid \mathbf{T} \big)} .
\end{equation}
The average of \cref{eq:ce-causality-gap} recovers \cref{eq:causality-axiom}.

\begin{proposition}[\textbf{The axiom's divergence is an excess cross-entropy}]\label{prop:ce-causality-identity}
For every transfer,
\begin{equation}\label{eq:ce-causality-identity}
    \mathbb{E}_{v \sim P(\cdot \mid E(u))}\big[ \Lambda(v) \big]
    = \frac{1}{n}\, \KL\big( P(\cdot \mid E(u)) \,\big\|\, P(\cdot \mid \mathbf{T}) \big)
    \;\geq\; 0 .
\end{equation}
\end{proposition}

\begin{proof}[Proof of \cref{prop:ce-causality-identity}]
Averaging \cref{eq:ce-causality-gap} under $P(\cdot \mid E(u))$,
\begin{align*}
    \mathbb{E}_{v \sim P(\cdot \mid E(u))}\big[ \Lambda(v) \big]
    &= \sum_{v \in \mathcal{V}^{n}} P\big( v \mid E(u) \big) \cdot
       \frac{1}{n} \log \frac{P\big( v \mid E(u) \big)}{P\big( v \mid \mathbf{T} \big)} \\
    &= \frac{1}{n} \sum_{v \in \mathcal{V}^{n}} P\big( v \mid E(u) \big)
       \log \frac{P\big( v \mid E(u) \big)}{P\big( v \mid \mathbf{T} \big)} \\
    &= \frac{1}{n}\, \KL\big( P(\cdot \mid E(u)) \,\big\|\, P(\cdot \mid \mathbf{T}) \big)
     \;\geq\; 0 ,
\end{align*}
where Gibbs' inequality gives the final inequality.
\end{proof}

\Cref{prop:ce-causality-identity} averages over the targets that the textual transfer induces,
whereas training averages over $D$.
\Cref{ass:reference-faithfulness} bounds the distance between the two distributions and
\cref{ass:bounded-disagreement} bounds the quantity being averaged, hence \cref{eq:loss-causality}
controls the two cross-entropies of \cref{eq:ce-pair} under the conditions that
\cref{app:theory:causality} already states.

\begin{proposition}[\textbf{The surrogate bounds the excess cross-entropy}]\label{prop:ce-causality-bound}
Under \cref{ass:bounded-disagreement,ass:reference-faithfulness}, at every $\psi$ in the optimization
domain,
\begin{equation}\label{eq:ce-causality-bound}
    - 2 B \varepsilon
    \;\leq\; \mathbb{E}_{v \sim D}\big[ \Lambda(v) \big]
    \;\leq\; \mathbb{E}_{v \sim D}\big[ \mathcal{L}_{\mathrm{caus}}(\psi) \big] + 3 B \varepsilon .
\end{equation}
\end{proposition}

\begin{proof}[Proof of \cref{prop:ce-causality-bound}]
Expanding \cref{eq:ce-causality-gap} through \cref{eq:seq-law} and bounding each of its $n$
log-ratios by \cref{eq:bounded-disagreement},
\begin{equation*}
    \big| \Lambda(v) \big|
    = \bigg| \frac{1}{n} \sum_{t=1}^{n}
      \log \frac{p\big( v_t \mid v_{<t}, E(u) \big)}{p\big( v_t \mid v_{<t}, \mathbf{T} \big)} \bigg|
    \;\leq\; \frac{1}{n} \sum_{t=1}^{n} B
    \;=\; B ,
\end{equation*}
hence $\Lambda : \mathcal{V}^{n} \to [-B, B]$ and its range fits in an interval of length $2B$.
\Cref{lem:tv-averaging}, applied to
$\Lambda$ and to the two distributions $D$ and $P(\cdot \mid E(u))$ on $\mathcal{V}^{n}$, therefore
gives
\begin{equation}\label{eq:ce-causality-transport}
    \Big| \mathbb{E}_{D}\big[ \Lambda \big]
      - \mathbb{E}_{P(\cdot \mid E(u))}\big[ \Lambda \big] \Big|
    \;\leq\; 2 B \cdot \TV\big( D, P(\cdot \mid E(u)) \big)
    \;\leq\; 2 B \varepsilon .
\end{equation}
Combining \cref{eq:ce-causality-transport} with \cref{eq:ce-causality-identity} on the left and
with \cref{prop:causality-tracking} on the right yields
\begin{align*}
    \mathbb{E}_{D}\big[ \Lambda \big]
    &\;\geq\; \mathbb{E}_{P(\cdot \mid E(u))}\big[ \Lambda \big] - 2 B \varepsilon
     \;\geq\; - 2 B \varepsilon , \\
    \mathbb{E}_{D}\big[ \Lambda \big]
    &\;\leq\; \frac{1}{n}\, \KL\big( P(\cdot \mid E(u)) \,\big\|\, P(\cdot \mid \mathbf{T}) \big)
       + 2 B \varepsilon
     \;\leq\; \mathbb{E}_{v \sim D}\big[ \mathcal{L}_{\mathrm{caus}}(\psi) \big] + 3 B \varepsilon . \qedhere
\end{align*}
\end{proof}

\begin{corollary}[\textbf{Parity at zero surrogate}]\label{cor:ce-causality-exact}
If $\mathcal{L}_{\mathrm{caus}}(\psi) = 0$, then $\Lambda(v) = 0$.
\end{corollary}

\begin{proof}[Proof of \cref{cor:ce-causality-exact}]
\Cref{eq:loss-causality} averages $n$ nonnegative divergences, hence
\begin{equation*}
    \mathcal{L}_{\mathrm{caus}}(\psi) = 0
    \qquad\Longrightarrow\qquad
    \KL\big( p(\cdot \mid v_{<t}, E(u)) \,\big\|\, p(\cdot \mid v_{<t}, \mathbf{T}) \big) = 0
    \quad \text{at every } t ,
\end{equation*}
and a divergence vanishes only at equal distributions,
\begin{equation*}
    p\big( \cdot \mid v_{<t}, E(u) \big) = p\big( \cdot \mid v_{<t}, \mathbf{T} \big)
    \qquad \text{at every } t .
\end{equation*}
Substituting into \cref{eq:seq-law},
\begin{equation*}
    P\big( v \mid E(u) \big)
    = \prod_{t=1}^{n} p\big( v_t \mid v_{<t}, E(u) \big)
    = \prod_{t=1}^{n} p\big( v_t \mid v_{<t}, \mathbf{T} \big)
    = P\big( v \mid \mathbf{T} \big) ,
\end{equation*}
at which \cref{eq:ce-causality-gap} gives $\Lambda(v) = 0$.
\end{proof}

\Cref{cor:ce-causality-exact} holds at a single $v$ and therefore requires neither assumption.
Exponentiating \cref{eq:ce-causality-gap} restates \cref{eq:ce-causality-bound} in terms of the
probability that the consumer assigns to the correct answer.

\begin{corollary}[\textbf{Discount on the correct answer}]\label{cor:ce-causality-discount}
For every $v$,
\begin{equation}\label{eq:ce-causality-discount}
    P\big( v \mid \mathbf{T} \big)
    = P\big( v \mid E(u) \big) \cdot \exp\big( - n \Lambda(v) \big) ,
\end{equation}
and under \cref{ass:bounded-disagreement,ass:reference-faithfulness} the geometric mean of the
discount factor under $D$ satisfies
\begin{equation}\label{eq:ce-causality-discount-bound}
    \exp\Big( - n \big( \mathbb{E}_{v \sim D}[ \mathcal{L}_{\mathrm{caus}}(\psi) ]
      + 3 B \varepsilon \big) \Big)
    \;\leq\; \exp\Big( - n\, \mathbb{E}_{v \sim D}\big[ \Lambda(v) \big] \Big)
    \;\leq\; \exp\big( 2 n B \varepsilon \big) .
\end{equation}
\end{corollary}

\begin{proof}[Proof of \cref{cor:ce-causality-discount}]
Multiplying \cref{eq:ce-causality-gap} by $n$ and exponentiating,
\begin{equation*}
    \exp\big( n \Lambda(v) \big)
    = \frac{P\big( v \mid E(u) \big)}{P\big( v \mid \mathbf{T} \big)} ,
\end{equation*}
and rearranging yields \cref{eq:ce-causality-discount}.
The geometric mean of its discount factor under $D$ is
\begin{equation*}
    \exp\Big( \mathbb{E}_{v \sim D}\big[ \log \exp\big( - n \Lambda(v) \big) \big] \Big)
    = \exp\Big( - n\, \mathbb{E}_{v \sim D}\big[ \Lambda(v) \big] \Big) ,
\end{equation*}
and the exponential is increasing, hence applying it to the two bounds of
\cref{eq:ce-causality-bound} yields \cref{eq:ce-causality-discount-bound}.
\end{proof}

Within \cref{eq:ce-causality-bound}, the data and the frozen consumer fix every quantity except
$\mathbb{E}_{D}[ \mathcal{L}_{\mathrm{caus}} ]$.
\Cref{eq:loss-causality} therefore determines the substitution cost, and
\cref{cor:ce-causality-exact} establishes that reducing it to zero recovers the textual transfer
exactly.

\subsubsection{Minimality}
\label{app:ce:minimality}

Let $V$ denote the random variable realized by the consumer's target $v$ of \cref{def:transfer}, and
condition every entropy below on the blocks $E_{\mathrm{pre}}$ and $E_{\mathrm{post}}$
that \cref{def:transfer} holds independent of $\psi$, left implicit in the notation.
The smallest cross-entropy any decoder incurs on $V$ under a transferred block is the conditional
entropy of $V$ given that block, hence
\begin{equation}\label{eq:ce-minimality-excess}
    \mathbb{H}\big( V \mid \mathbf{T} \big) - \mathbb{H}\big( V \mid Y \big)
\end{equation}
is the cost of substituting the thought for the producer's output.
We establish that \cref{eq:loss-minimality} upper-bounds \cref{eq:ce-minimality-excess} up to a
constant, hence lowering the loss lowers an upper bound on that cost.

\begin{assumption}[\textbf{Nested input probes}]\label{ass:nested-probes}
Every probe that predicts $X$ from $Y$ is available to the probe that predicts $X$ from
$(Y, \mathbf{T})$, which recovers it by discarding $\mathbf{T}$.
\end{assumption}

Comparing the two probe classes at their optima\footnote{\Cref{ass:nested-probes} also requires
$\mathrm{CE}(X \mid Y, \mathbf{T}) \leq \mathrm{CE}(X \mid Y)$ for the frozen consumer, that is,
adding the thought never makes it worse at reconstructing $X$.} under \cref{ass:nested-probes}
bounds one cross-entropy of \cref{eq:loss-minimality},
\begin{equation}\label{eq:ce-minimality-nesting}
    \mathrm{CE}\big( X \mid Y, \mathbf{T} \big) \;\leq\; \mathrm{CE}\big( X \mid Y \big) ,
\end{equation}
and their difference estimates the residual mutual information $I(X ; \mathbf{T} \mid Y)$,
\begin{equation}\label{eq:ce-minimality-leakage}
    \widehat{I}\big( X ; \mathbf{T} \mid Y \big)
    \;\triangleq\; \mathrm{CE}\big( X \mid Y \big) - \mathrm{CE}\big( X \mid Y, \mathbf{T} \big)
    \;\geq\; 0 .
\end{equation}
Only $\mathrm{CE}(X \mid Y, \mathbf{T})$ of \cref{eq:ce-minimality-leakage} is in terms of
$\mathbf{T}$, hence $\mathrm{CE}(X \mid Y)$ is constant in $\psi$.

\begin{proposition}[\textbf{The loss bounds the reconstruction cross-entropy}]\label{prop:ce-minimality-ceiling}
Under \cref{ass:nested-probes}, at every $\psi$ in the optimization domain,
\begin{equation}\label{eq:ce-minimality-ceiling}
    \mathrm{CE}\big( Y \mid \mathbf{T} \big)
    \;\leq\; \frac{1}{\lambda_1} \Big( \mathcal{L}_{\mathrm{min}}(\psi)
      + \lambda_2\, \mathrm{CE}\big( X \mid Y \big) \Big) ,
\end{equation}
and the difference between the two sides of \cref{eq:ce-minimality-ceiling} is
\begin{equation}\label{eq:ce-minimality-slack}
    \frac{1}{\lambda_1} \Big( \mathcal{L}_{\mathrm{min}}(\psi)
      + \lambda_2\, \mathrm{CE}\big( X \mid Y \big) \Big)
    - \mathrm{CE}\big( Y \mid \mathbf{T} \big)
    = \frac{\lambda_2}{\lambda_1}\, \widehat{I}\big( X ; \mathbf{T} \mid Y \big) .
\end{equation}
\end{proposition}

\begin{proof}[Proof of \cref{prop:ce-minimality-ceiling}]
Solving \cref{eq:loss-minimality} for its first term, which $\lambda_1 > 0$ permits,
\begin{equation*}
    \mathrm{CE}\big( Y \mid \mathbf{T} \big)
    = \frac{1}{\lambda_1} \Big( \mathcal{L}_{\mathrm{min}}(\psi)
      + \lambda_2\, \mathrm{CE}\big( X \mid Y, \mathbf{T} \big) \Big) .
\end{equation*}
The remaining cross-entropy appears with the positive coefficient $\lambda_2 / \lambda_1$, hence
replacing
it by the larger quantity of \cref{eq:ce-minimality-nesting} preserves the inequality and yields
\cref{eq:ce-minimality-ceiling}.
Subtracting $\mathrm{CE}(Y \mid \mathbf{T})$ from the right side of \cref{eq:ce-minimality-ceiling}
and applying the same identity,
\begin{equation*}
    \frac{1}{\lambda_1} \Big( \mathcal{L}_{\mathrm{min}}(\psi)
      + \lambda_2\, \mathrm{CE}\big( X \mid Y \big) \Big)
    - \mathrm{CE}\big( Y \mid \mathbf{T} \big)
    = \frac{\lambda_2}{\lambda_1}
      \Big( \mathrm{CE}\big( X \mid Y \big)
        - \mathrm{CE}\big( X \mid Y, \mathbf{T} \big) \Big) ,
\end{equation*}
which is \cref{eq:ce-minimality-slack} by \cref{eq:ce-minimality-leakage}.
\end{proof}

\begin{proposition}[\textbf{Reconstruction uncertainty bounds answer uncertainty}]\label{prop:ce-minimality-chain}
For every transfer,
\begin{equation}\label{eq:ce-minimality-chain}
    \mathbb{H}\big( V \mid \mathbf{T} \big) - \mathbb{H}\big( V \mid Y \big)
    \;\leq\; \mathbb{H}\big( Y \mid \mathbf{T} \big) .
\end{equation}
\end{proposition}

\begin{proof}[Proof of \cref{prop:ce-minimality-chain}]
Adjoining a variable cannot lower a joint entropy, the chain rule splits the result, and conditioning
on more cannot raise an entropy,
\begin{align*}
    \mathbb{H}\big( V \mid \mathbf{T} \big)
    &\;\leq\; \mathbb{H}\big( V, Y \mid \mathbf{T} \big) \\
    &\;=\; \mathbb{H}\big( Y \mid \mathbf{T} \big)
       + \mathbb{H}\big( V \mid Y, \mathbf{T} \big) \\
    &\;\leq\; \mathbb{H}\big( Y \mid \mathbf{T} \big) + \mathbb{H}\big( V \mid Y \big) ,
\end{align*}
and subtracting $\mathbb{H}(V \mid Y)$ yields \cref{eq:ce-minimality-chain}.
\end{proof}

\begin{proposition}[\textbf{The surrogate bounds the substitution cost}]\label{prop:ce-minimality-bound}
Under \cref{ass:nested-probes}, at every $\psi$ in the optimization domain,
\begin{equation}\label{eq:ce-minimality-bound}
    \mathbb{H}\big( V \mid \mathbf{T} \big) - \mathbb{H}\big( V \mid Y \big)
    \;\leq\; \frac{1}{\lambda_1} \Big( \mathcal{L}_{\mathrm{min}}(\psi)
      + \lambda_2\, \mathrm{CE}\big( X \mid Y \big) \Big) .
\end{equation}
\end{proposition}

\begin{proof}[Proof of \cref{prop:ce-minimality-bound}]
Gibbs' inequality bounds an entropy by the cross-entropy that approximates it,
\begin{equation*}
    \mathbb{H}\big( Y \mid \mathbf{T} \big) \;\leq\; \mathrm{CE}\big( Y \mid \mathbf{T} \big) ,
\end{equation*}
and chaining \cref{eq:ce-minimality-chain}, this inequality, and \cref{eq:ce-minimality-ceiling}
gives
\begin{equation*}
    \mathbb{H}\big( V \mid \mathbf{T} \big) - \mathbb{H}\big( V \mid Y \big)
    \;\leq\; \mathbb{H}\big( Y \mid \mathbf{T} \big)
    \;\leq\; \mathrm{CE}\big( Y \mid \mathbf{T} \big)
    \;\leq\; \frac{1}{\lambda_1} \Big( \mathcal{L}_{\mathrm{min}}(\psi)
      + \lambda_2\, \mathrm{CE}\big( X \mid Y \big) \Big) . \qedhere
\end{equation*}
\end{proof}

The data and the frozen models fix $\mathrm{CE}(X \mid Y)$ and $\mathbb{H}(V \mid Y)$.
We can restate \cref{eq:ce-minimality-bound} in terms of the probability that the consumer
assigns to its target.

\begin{corollary}[\textbf{Bound on the probability of the target}]\label{cor:ce-minimality-discount}
Under \cref{ass:nested-probes}, the geometric mean probability that an optimal decoder assigns to
the consumer's target satisfies
\begin{equation}\label{eq:ce-minimality-discount}
    \exp\Big( - \mathbb{H}\big( V \mid \mathbf{T} \big) \Big)
    \;\geq\; \exp\Big( - \mathbb{H}\big( V \mid Y \big) \Big)
      \cdot \exp\!\bigg( - \frac{1}{\lambda_1} \Big( \mathcal{L}_{\mathrm{min}}(\psi)
        + \lambda_2\, \mathrm{CE}\big( X \mid Y \big) \Big) \bigg) .
\end{equation}
\end{corollary}

\begin{proof}[Proof of \cref{cor:ce-minimality-discount}]
The exponential is increasing, hence negating \cref{eq:ce-minimality-bound} and applying it gives
\begin{equation*}
    \exp\Big( - \mathbb{H}\big( V \mid \mathbf{T} \big) \Big)
    \;\geq\; \exp\bigg( - \mathbb{H}\big( V \mid Y \big)
      - \frac{1}{\lambda_1} \Big( \mathcal{L}_{\mathrm{min}}(\psi)
        + \lambda_2\, \mathrm{CE}\big( X \mid Y \big) \Big) \bigg) ,
\end{equation*}
and splitting the exponential of a sum yields \cref{eq:ce-minimality-discount}.
The left side is the reciprocal perplexity of $V$ under $\mathbf{T}$, which is the geometric mean of
the probability assigned to the target.
\end{proof}

The exponential of \cref{eq:ce-minimality-discount} does not exceed one, and it attains one exactly
when $\mathrm{CE}(Y \mid \mathbf{T})$ and $\widehat{I}(X ; \mathbf{T} \mid Y)$ both vanish, at which
the minimality loss reaches its minimum $-\lambda_2 \mathrm{CE}(X \mid Y)$ and the probability under
the latent transfer attains the probability under the textual one.

\subsubsection{Separability}
\label{app:ce:separability}

Let $v$ denote the consumer's target under $\mathbf{T}$ and let $v_{k}$ denote the target of the
$k$-th preceding example under $\mathbf{T}_{k}$, with
\begin{equation}\label{eq:ce-separability-scores}
    a = P\big( v \mid \mathbf{T} \big) ,
    \qquad
    a_{k} = P\big( v_{k} \mid \mathbf{T}_{k} \big) ,
\end{equation}
supplied by \cref{eq:seq-law} at the corresponding block.

\begin{proposition}[\textbf{The loss bounds the nearest distance}]\label{prop:ce-separability-proximity}
Let $k^\star$ attain $\max_{k} \mathrm{sim}(\mathbf{t}, \mathbf{t}_k)$ and let
\begin{equation}\label{eq:ce-separability-delta}
    \delta \;\triangleq\; \sqrt{\, 2 - 2\tau\big( \mathcal{L}_{\mathrm{sep}}(\psi, \omega) - \log K \big) \,} .
\end{equation}
Then $\delta$ is real and
\begin{equation}\label{eq:ce-separability-proximity}
    \big\| \varphi_\omega(\mathbf{T}) - \varphi_\omega(\mathbf{T}_{k^\star}) \big\|
    \;\leq\; \delta .
\end{equation}
\end{proposition}

\begin{proof}[Proof of \cref{prop:ce-separability-proximity}]
Bounding each term of \cref{eq:loss-separability} by the largest,
\begin{align}
    \mathcal{L}_{\mathrm{sep}}(\psi, \omega)
    &= \log \sum_{k=1}^{K}
       \exp\!\Big( \mathrm{sim}\big( \mathbf{t}, \mathbf{t}_k \big) \big/ \tau \Big) \nonumber \\
    &\leq \log \bigg( K \exp\!\Big( \mathrm{sim}\big( \mathbf{t}, \mathbf{t}_{k^\star} \big)
       \big/ \tau \Big) \bigg) \nonumber \\
    &= \log K + \frac{1}{\tau}\, \mathrm{sim}\big( \mathbf{t}, \mathbf{t}_{k^\star} \big) ,
    \label{eq:ce-separability-lse}
\end{align}
and rearranging \cref{eq:ce-separability-lse},
\begin{equation}\label{eq:ce-separability-sim-lower}
    \mathrm{sim}\big( \mathbf{t}, \mathbf{t}_{k^\star} \big)
    \;\geq\; \tau \big( \mathcal{L}_{\mathrm{sep}}(\psi, \omega) - \log K \big) .
\end{equation}
Cosine similarity does not exceed one, hence \cref{eq:ce-separability-sim-lower} lower-bounds the
radicand of \cref{eq:ce-separability-delta} by zero and $\delta$ is real.
Substituting \cref{eq:ce-separability-sim-lower} into \cref{eq:sphere-identity},
\begin{align*}
    \big\| \varphi_\omega(\mathbf{T}) - \varphi_\omega(\mathbf{T}_{k^\star}) \big\|^{2}
    &= 2 - 2\, \mathrm{sim}\big( \mathbf{t}, \mathbf{t}_{k^\star} \big) \\
    &\leq 2 - 2\tau \big( \mathcal{L}_{\mathrm{sep}}(\psi, \omega) - \log K \big)
     \;=\; \delta^{2} ,
\end{align*}
and taking square roots yields \cref{eq:ce-separability-proximity}.
\end{proof}

\begin{assumption}[\textbf{Smooth consumer}]\label{ass:smooth-consumer}
There exists a finite $L$ with
\begin{equation}\label{eq:smooth-consumer}
    \TV\Big( P\big( \cdot \mid \mathbf{T} \big), P\big( \cdot \mid \mathbf{T}' \big) \Big)
    \;\leq\; L \, \big\| \varphi_\omega(\mathbf{T}) - \varphi_\omega(\mathbf{T}') \big\|
\end{equation}
at every pair of thoughts and every $\psi$ in the optimization domain.
\end{assumption}

\Cref{ass:smooth-consumer} requires that the consumer's target distribution vary continuously with
the thought it receives, and it identifies that distribution's dependence on $\mathbf{T}$ with the
pooled direction $\varphi_\omega(\mathbf{T})$ of \cref{eq:sphere-map}.

\begin{proposition}[\textbf{Colliding thoughts share one budget}]\label{prop:ce-separability-budget}
Under \cref{ass:preceding-disjointness,ass:smooth-consumer},
\begin{equation}\label{eq:ce-separability-budget}
    a + a_{k^\star} \;\leq\; 1 + L \delta .
\end{equation}
\end{proposition}

\begin{proof}[Proof of \cref{prop:ce-separability-budget}]
Applying \cref{eq:smooth-consumer} to the single event $\{v\}$ and then
\cref{eq:ce-separability-proximity},
\begin{equation}\label{eq:ce-separability-transport}
    P\big( v \mid \mathbf{T}_{k^\star} \big)
    \;\geq\; P\big( v \mid \mathbf{T} \big)
      - L \big\| \varphi_\omega(\mathbf{T}) - \varphi_\omega(\mathbf{T}_{k^\star}) \big\|
    \;\geq\; a - L \delta .
\end{equation}
\Cref{ass:preceding-disjointness} makes the two examples semantically disjoint, hence $v$ and
$v_{k^\star}$ are distinct outcomes of the single distribution $P(\cdot \mid \mathbf{T}_{k^\star})$
and
\begin{equation}\label{eq:ce-separability-disjoint}
    P\big( v \mid \mathbf{T}_{k^\star} \big)
      + P\big( v_{k^\star} \mid \mathbf{T}_{k^\star} \big) \;\leq\; 1 .
\end{equation}
Substituting \cref{eq:ce-separability-transport} into \cref{eq:ce-separability-disjoint},
\begin{equation*}
    \big( a - L \delta \big) + a_{k^\star}
    \;\leq\; P\big( v \mid \mathbf{T}_{k^\star} \big)
      + P\big( v_{k^\star} \mid \mathbf{T}_{k^\star} \big)
    \;\leq\; 1 ,
\end{equation*}
and adding $L \delta$ to both ends yields \cref{eq:ce-separability-budget}.
\end{proof}

\begin{proposition}[\textbf{The loss floors the cross-entropy}]\label{prop:ce-separability-floor}
Under \cref{ass:preceding-disjointness,ass:smooth-consumer}, at every $\psi$ in the optimization
domain,
\begin{equation}\label{eq:ce-separability-floor}
    \frac{1}{2} \Big( - \log a - \log a_{k^\star} \Big)
    \;\geq\; \log 2 - \log\big( 1 + L \delta \big) .
\end{equation}
\end{proposition}

\begin{proof}[Proof of \cref{prop:ce-separability-floor}]
The inequality of arithmetic and geometric means bounds the product by the square of the mean, and
\cref{eq:ce-separability-budget} bounds that mean,
\begin{equation}\label{eq:ce-separability-amgm}
    a \, a_{k^\star}
    \;\leq\; \bigg( \frac{a + a_{k^\star}}{2} \bigg)^{2}
    \;\leq\; \bigg( \frac{1 + L \delta}{2} \bigg)^{2} .
\end{equation}
The logarithm is increasing, hence applying $-\tfrac{1}{2}\log(\cdot)$ to
\cref{eq:ce-separability-amgm} reverses it,
\begin{align*}
    \frac{1}{2} \Big( - \log a - \log a_{k^\star} \Big)
    &= - \frac{1}{2} \log\big( a \, a_{k^\star} \big) \\
    &\geq - \frac{1}{2} \log \bigg( \frac{1 + L \delta}{2} \bigg)^{2}
     \;=\; - \log \bigg( \frac{1 + L \delta}{2} \bigg)
     \;=\; \log 2 - \log\big( 1 + L \delta \big) . \qedhere
\end{align*}
\end{proof}

Substituting \cref{eq:ce-separability-delta} states \cref{eq:ce-separability-floor} through the loss,
\begin{equation}\label{eq:ce-separability-floor-loss}
    \frac{1}{2} \Big( - \log a - \log a_{k^\star} \Big)
    \;\geq\; \log 2 - \log\bigg( 1
      + L \sqrt{\, 2 - 2\tau\big( \mathcal{L}_{\mathrm{sep}}(\psi, \omega) - \log K \big) \,} \bigg) .
\end{equation}

\begin{corollary}[\textbf{Threshold above which the objective is obstructed}]\label{cor:ce-separability-threshold}
Under \cref{ass:preceding-disjointness,ass:smooth-consumer}, the right side of
\cref{eq:ce-separability-floor-loss} is positive if and only if
\begin{equation}\label{eq:ce-separability-threshold}
    \mathcal{L}_{\mathrm{sep}}(\psi, \omega)
    \;>\; \log K + \frac{1}{\tau} \bigg( 1 - \frac{1}{2 L^{2}} \bigg) .
\end{equation}
\end{corollary}

\begin{proof}[Proof of \cref{cor:ce-separability-threshold}]
Positivity of the right side of \cref{eq:ce-separability-floor} requires
\begin{equation*}
    \log 2 > \log\big( 1 + L \delta \big)
    \qquad\Longleftrightarrow\qquad
    L \delta < 1
    \qquad\Longleftrightarrow\qquad
    \delta^{2} < \frac{1}{L^{2}} ,
\end{equation*}
and substituting \cref{eq:ce-separability-delta},
\begin{align*}
    2 - 2\tau\big( \mathcal{L}_{\mathrm{sep}}(\psi, \omega) - \log K \big)
    &< \frac{1}{L^{2}} \\
    \mathcal{L}_{\mathrm{sep}}(\psi, \omega)
    &> \log K + \frac{1}{\tau} \bigg( 1 - \frac{1}{2 L^{2}} \bigg) . \qedhere
\end{align*}
\end{proof}

If training resulted in two different examples having colliding thoughts, then a consumer agent must
then answer both with nearly the same distribution even if the two target texts are different.
This imposes a lower-bound on their average cross-entropy (\cref{eq:ce-separability-floor-loss}).
Under total collapse, where $\delta = 0$, the optimal solution would be assigning equal probability
to both targets which will make the lower-bound at its largest value ($\log 2$, random chance between two targets).
CE loss does not prevent such collisions through its objective.

\subsubsection{Stability}
\label{app:ce:stability}

A textual transfer delivers one output $u$ drawn from the producer's output distribution $Q$ of
\cref{eq:producer-law}, whereas $\mathbf{T}$ is a trained function and is not bound to that sample.
Let $s(u)$ denote the probability the consumer assigns to its target $v$ under the transfer of $u$,
and let $\bar{P}$ denote the mixture over the producer's samples,
\begin{equation}\label{eq:ce-stability-scores}
    s(u) = P\big( v \mid E(u) \big) \in (0, 1] ,
    \qquad
    \bar{s} = \mathbb{E}_{u \sim Q}\big[ s(u) \big] ,
    \qquad
    \bar{P} = \mathbb{E}_{u \sim Q}\big[ P( \cdot \mid E(u) ) \big] .
\end{equation}
Evaluating $\bar{P}$ at $v$,
\begin{equation}\label{eq:ce-stability-mixture-at-v}
    \bar{P}(v)
    = \mathbb{E}_{u \sim Q}\big[ P( v \mid E(u) ) \big]
    = \bar{s} .
\end{equation}
Both proofs below use
\begin{equation}\label{eq:ce-stability-convexity}
    f(x) = - \log x ,
    \qquad
    f''(x) = \frac{1}{x^{2}} \;\geq\; 1
    \qquad \text{on } (0, 1] .
\end{equation}

\begin{proposition}[\textbf{The mixture bounds the sampled transfer}]\label{prop:ce-stability-jensen}
For every producer,
\begin{equation}\label{eq:ce-stability-jensen}
    G \;\triangleq\; \mathbb{E}_{u \sim Q}\big[ - \log s(u) \big] - \big( - \log \bar{s} \big)
    \;\geq\; 0 ,
\end{equation}
with equality if and only if
\begin{equation}\label{eq:ce-stability-degenerate}
    s(u) = \bar{s}
    \qquad \text{for } Q\text{-almost every } u .
\end{equation}
\end{proposition}

\begin{proof}[Proof of \cref{prop:ce-stability-jensen}]
\Cref{eq:ce-stability-convexity} makes $f$ strictly convex on $(0, 1]$, hence Jensen's inequality
gives
\begin{equation*}
    \mathbb{E}_{u \sim Q}\big[ - \log s(u) \big]
    \;\geq\; - \log \mathbb{E}_{u \sim Q}\big[ s(u) \big]
    \;=\; - \log \bar{s} ,
\end{equation*}
and strict convexity attains equality only at \cref{eq:ce-stability-degenerate}.
\end{proof}

\begin{proposition}[\textbf{The advantage is the sampling variance}]\label{prop:ce-stability-variance}
For every producer,
\begin{equation}\label{eq:ce-stability-variance}
    G \;\geq\; \tfrac{1}{2}\, \mathrm{Var}_{u \sim Q}\big( s(u) \big) .
\end{equation}
\end{proposition}

\begin{proof}[Proof of \cref{prop:ce-stability-variance}]
Taylor's theorem at $\bar{s}$ with Lagrange remainder places $\xi$ between $s(u)$ and $\bar{s}$,
hence in $(0, 1]$, and \cref{eq:ce-stability-convexity} bounds $f''(\xi)$,
\begin{align*}
    - \log s(u)
    &= - \log \bar{s} - \frac{s(u) - \bar{s}}{\bar{s}}
       + \tfrac{1}{2} f''(\xi) \big( s(u) - \bar{s} \big)^{2} \\
    &\geq - \log \bar{s} - \frac{s(u) - \bar{s}}{\bar{s}}
       + \tfrac{1}{2} \big( s(u) - \bar{s} \big)^{2} .
\end{align*}
The linear term vanishes under $Q$,
\begin{equation*}
    \mathbb{E}_{u \sim Q}\big[ s(u) - \bar{s} \big] = 0 ,
\end{equation*}
hence averaging the previous display over $u \sim Q$,
\begin{equation*}
    \mathbb{E}_{u \sim Q}\big[ - \log s(u) \big]
    \;\geq\; - \log \bar{s}
      + \tfrac{1}{2}\, \mathbb{E}_{u \sim Q}\Big[ \big( s(u) - \bar{s} \big)^{2} \Big]
    \;=\; - \log \bar{s} + \tfrac{1}{2}\, \mathrm{Var}_{u \sim Q}\big( s(u) \big) ,
\end{equation*}
and subtracting $- \log \bar{s}$ yields \cref{eq:ce-stability-variance}.
\end{proof}

\begin{assumption}[\textbf{Realizable mixture}]\label{ass:realizable-mixture}
Some $\psi$ in the optimization domain satisfies
\begin{equation}\label{eq:realizable-mixture}
    P\big( \cdot \mid \mathbf{T} \big) = \bar{P} .
\end{equation}
\end{assumption}

\begin{corollary}[\textbf{The latent transfer attains the advantage}]\label{cor:ce-stability-advantage}
Under \cref{ass:realizable-mixture}, at that $\psi$,
\begin{equation}\label{eq:ce-stability-advantage}
    - \log P\big( v \mid \mathbf{T} \big)
    \;=\; \mathbb{E}_{u \sim Q}\big[ - \log P\big( v \mid E(u) \big) \big]
      \;-\; G
    \;\leq\; \mathbb{E}_{u \sim Q}\big[ - \log P\big( v \mid E(u) \big) \big] .
\end{equation}
\end{corollary}

\begin{proof}[Proof of \cref{cor:ce-stability-advantage}]
Evaluating \cref{eq:realizable-mixture} at the single outcome $v$ and applying
\cref{eq:ce-stability-mixture-at-v},
\begin{equation*}
    - \log P\big( v \mid \mathbf{T} \big)
    \;=\; - \log \bar{P}(v)
    \;=\; - \log \bar{s} ,
\end{equation*}
and \cref{eq:ce-stability-jensen} rearranges this to
\begin{equation*}
    - \log P\big( v \mid \mathbf{T} \big)
    \;=\; \mathbb{E}_{u \sim Q}\big[ - \log s(u) \big] - G ,
\end{equation*}
where $G \geq 0$ by \cref{prop:ce-stability-jensen}.
\end{proof}

A textual transfer cannot attain the left side of \cref{eq:ce-stability-advantage}, since $E(u)$ is
the sampled output and its target distribution is fixed at $P(\cdot \mid E(u))$.
\Cref{eq:ce-stability-variance} makes the advantage strictly positive whenever the producer's sample
moves the consumer's target probability.
\Cref{eq:loss-stability} encodes a scalar summary of $Q$, since encoding $Q$ requires
multiple producer samples per training example.

\section{Implementation Details}
\label{app:hyperparams}

\subsection{Benchmarks}
\label{app:datasets}

\begin{itemize}[leftmargin=*]
    \item \textbf{MATH500} \citep{lightman2024stepverify} contains 500 test problems drawn uniformly at random from MATH, and its mix of subjects and difficulty levels is representative of the full test set.
    \item \textbf{AIME2025} \citep{zhang2025aimeexam} is the 2025 American Invitational Mathematics Examination, 30 problems whose answers are single integers.
    \item \textbf{AIME2026} \citep{dekoninck2026matharena} is the 2026 examination in the same format.
    \item \textbf{GPQA-Diamond} \citep{rein2024gpqabench} is the most strictly filtered GPQA subset, 198 graduate-level multiple-choice questions in biology, chemistry and physics that both expert validators answered correctly and most non-experts missed.
    \item \textbf{MedQA} \citep{jin2021disease} draws multiple-choice questions from medical licensing examinations.
    \item \textbf{LiveCodeBench-v6} \citep{jain2025livecodebench} collects competition programming problems with release dates, which lets evaluation be restricted to a window after a model's training cutoff.
    \item \textbf{MBPP+} \citep{liu2023evalplus} re-evaluates the Python programming tasks of MBPP against many more test cases, and a program counts as correct only when it passes every test.
\end{itemize}

We report Pass@10 for AIME2025 and AIME2026, and Pass@1 for every other benchmark.
The \textbf{Code Gen.} column reports MBPP+ on the Light system and LiveCodeBench-v6 on the Scaled system.

\subsection{Compared Baselines}
\label{app:baseline}
We compare REST against the following alternative auxiliary-loss baselines:

\begin{itemize}[leftmargin=*]
    \item \textbf{CODI}~\citep{shen2025selfdistillcot} trains continuous latent thoughts as the student in a self-distillation setup, matching the student's hidden states at one distillation token to those of a teacher pass that reads the explicit chain of thought. We adapt its distillation loss as an auxiliary term on the outer-link cross-entropy objective, taking the consumer's text transfer of \cref{def:transfer} as the teacher.

    \item \textbf{SIM-CoT}~\citep{wei2025simcot} supervises latent reasoning with an auxiliary decoder that reconstructs reasoning-step text from the implicit latents and is discarded at inference. We adapt its reconstruction loss as an auxiliary term on the outer-link cross-entropy objective, training a per-stage decoder to reconstruct the producer's output text from the transferred thought.
\end{itemize}

\subsection{Hyperparameters}
\label{app:hp}

\begin{wraptable}{R}{0.325\linewidth}
\centering
\vspace{-1.0\baselineskip}
\caption{Generation budget.}
\label{tab:hyperparams}
\begin{tabular}{l|cc}
\toprule
\textbf{Benchmark} & \textbf{Light} & \textbf{Scaled} \\
\midrule
MATH500  & 1000 & 2000  \\
AIME2025 & 8192 & 16000 \\
AIME2026 & 8192 & 16000 \\
GPQA-D   & 4000 & 4000  \\
MedQA    & 4000 & 4000  \\
MBPP+    & 4000 & 4000  \\
LCB-v6   & 4096 & 4096  \\
\bottomrule
\end{tabular}
\vspace{-1.0\baselineskip}
\end{wraptable}

\textbf{Optimization.}
We train the outer links for 20{,}000 steps at batch size 4 and a maximum sequence length of 4096, in bfloat16.
AdamW runs at a learning rate of $5 \times 10^{-4}$ under a cosine schedule with 10 warmup steps, no weight decay, and gradients clipped at 1.0.

\textbf{Loss Terms.}
Every property term is swept over $\beta \in \{0.1, 0.3, 1.0, 3.0\}$.
For minimality, we set $\lambda_1 = 1.0$ and $\lambda_2 = 0.3$ and $0$ for composition. For separability, we set $K = 64$ at $\tau = 0.1$.
The CODI baseline is swept over $\{1, 10, 20\}$ and SIM-CoT over $\{0.3, 1, 3\}$.

\textbf{Inference.}
A thought spans at most 80 positions in training and a fixed 32 latent steps at inference.
We sample at temperature 0.6 and top-$p$ 0.95, lowered to 0.2 on MBPP+ and LiveCodeBench-v6.
\Cref{tab:hyperparams} provides max length for generated text.

\section{Extended Related Work}
\label{app:related}

\textbf{Reasoning in Latent Space.}
A parallel family internalizes explicit rationales into the weights or the architecture, leaving the input stream unchanged.
Stepwise removal of thinking tokens moves an explicit chain of thought into the forward pass \citep{deng2025implicitcot}, dense embeddings stand in for a compressed chain \citep{cheng2024compressedcot}, and latent and text tokens are mixed within one sequence \citep{su2025tokenmixing}.
Architectural variants add depth or recurrence at fixed points of the network, through dynamic depth scaling \citep{chen2025innerthinking} and through middle-layer recurrence across decoding steps \citep{cai2026tmlr}.
Recurrence can also carry a persistent latent state across refinement steps, each supervised against a ground-truth intermediate target \citep{li2026persistentlatent}.
\citet{huang2026transformers} prove that a transformer internalizes a chain of thought under a staged curriculum, where the schedule for removing thinking tokens sets the number of training stages.
Latent reasoning has also been applied to chemistry \citep{ye2026latentchem}, built on hybrid state-space backbones \citep{wang2026tiny}, and cast as an exchange between a fast and a slow process \citep{codaforno2025dualsystem}.

\textbf{Supervising Latent Thoughts.}
Two objectives compress reasoning over decoded traces rather than over latent representations.
\citet{massoli2026reasoning} model a reasoning trace under a conditional information bottleneck, retaining only the information about the response that the prompt does not already supply, and optimize it as a reinforcement learning objective with a surprisal prior over traces.
\citet{conklin2026learning} treat training itself as lossy compression and report that pretrained models approach the information bottleneck bound, which links how much a model compresses to downstream performance.
Neither objective constrains a thought transferred between agents.
REST applies its terms at that transfer.
Other work measures internal states without training on them, through uncertainty over the meanings a model could generate \citep{kuhn2023semantic} and through the geometry of reasoning trajectories in representation space \citep{zhou2026geometryreasoning}.
\citet{li2026dynamics} intervene on individual latent steps and find that the steps carry different functions and route information non-locally.
A larger latent budget does not act as uniform extra depth.

REST differs in the object it constrains and in the form of the constraint. Its terms apply to the latent thought at the point where one agent or one round passes it to the next, and each term is a differentiable addition to the cross-entropy objective that leaves the architecture unchanged and adds no parameters at inference.

\section{Additional Results}
\label{app:results}

\subsection{Base Model Accuracy}
\label{app:base-model-accuracy}

As shown in \cref{tab:base-model}, each frozen model answers at a level that varies substantially across domains, and no single model leads on every benchmark within either agent composition. The latent system combines agents of complementary strength, and the values here give the accuracy of each agent before any system is applied.

\begin{table}[h!]
\caption{Frozen base LLMs, no system.}
\label{tab:base-model}
\begin{center}
\small
\setlength{\tabcolsep}{3.5pt}
\begin{tabular}{l|c|ccccccc}
\toprule
Model & Metric & MATH500 & GPQA-D & MedQA & AIME25 & AIME26 & MBPP+ & LCB-v6 \\
\midrule
\multirow{2}{*}{Qwen3-1.7B} & Acc.  & 68.4 & 33.5 & 46.2 & 20.0 & 22.2 & 57.3 & 22.8 \\
                             & Token & 546  & 1009 & 472  & 1962 & 2695 & 80  & 608 \\
\midrule
\multirow{2}{*}{Llama-3.2-1B} & Acc.  & 25.0 & 27.4 & 37.0 & 1.1  & 2.2  & 41.7 & 4.3 \\
                               & Token & 417  & 629  & 369  & 2412 & 2225 & 309 & 288 \\
\midrule
\multirow{2}{*}{Qwen2.5-Math-1.5B} & Acc.  & 76.3 & 29.6 & 28.4 & 27.8 & 22.2 & 33.1 & 2.9 \\
                                    & Token & 530  & 807  & 1061 & 908 & 987 & 693 & 669 \\
\specialrule{1pt}{1pt}{1pt}
\multirow{2}{*}{Gemma-3-4B} & Acc.  & 76.3 & 31.0 & 52.9 & 27.8 & 23.3 & 70.8 & 18.5 \\
                             & Token & 829  & 807  & 341  & 1543 & 1623 & 136 & 278 \\
\midrule
\multirow{2}{*}{Llama-3.2-3B} & Acc.  & 49.6 & 31.0 & 58.4 & 10.0 & 7.8  & 59.5 & 12.0 \\
                               & Token & 509  & 967  & 385  & 3389 & 3033 & 287 & 324 \\
\midrule
\multirow{2}{*}{Qwen3.5-4B} & Acc.  & 79.1 & 62.5 & 83.8 & 81.1 & 90.0 & 70.9 & 41.1 \\
                             & Token & 993  & 2526 & 1086 & 10752 & 10888 & 272 & 1883 \\
\bottomrule
\end{tabular}
\end{center}
\end{table}

\subsection{Per-Round Results}
\label{app:per-round}

\Cref{tab:single-agent-main} and \cref{tab:multi-agent-main} in the main text configure each system at its own recursion round, and compare every method against the CE-only baseline at that same round. The tables below give both rounds in full for each setting. Across training seeds, the average standard error of the Avg.\ Change column is $\pm0.7$ points in accuracy and $\pm2.1\%$ in tokens.

As shown in \cref{tab:single-agent,tab:single-agent-r3,tab:multi-agent,tab:multi-agent-r3} and summarised in \cref{fig:depth-scale}, recursion depth affects the two scales in opposite directions. In the Light setting, every property term improves over the CE-only objective at $r=1$, and the improvement is diminished at $r=3$. In the Scaled setting the ordering is reversed. REST is relatively close to CE-only at $r=1$, while at $r=3$ every term except multi-agent separability improves accuracy. At $r=1$ our CE-only reproduction is close to the accuracy reported by \citet{zou2026recursivemultiagentsystems} on the Scaled system and lower on the Light system, while at $r=3$ it is lower on both, most of all on AIME, where seed variance is largest. Avg.\ Change at $r=3$ is measured against this reproduction of the experiment.

\begin{table}[h!]
\centering \small
\caption{Single-agent (round $r=1$), Light vs Scaled.}
\label{tab:single-agent}
\resizebox{\linewidth}{!}{%
\arrayrulecolor{black!30}
\begin{tabular}{l|c|c>{\columncolor{tolGrey!40}}cc>{\columncolor{tolGrey!40}}cc>{\columncolor{tolGrey!40}}c|c>{\columncolor{tolGrey!40}}c|c>{\columncolor{tolGrey!40}}c|c>{\columncolor{tolGrey!40}}c|cc}
\arrayrulecolor{restcoral}\specialrule{2.0pt}{0pt}{2pt}\arrayrulecolor{black!30}
\textbf{Method} & \textbf{Metric} & \multicolumn{2}{c}{\textbf{Math500}} & \multicolumn{2}{c}{\textbf{AIME2025}} & \multicolumn{2}{c|}{\textbf{AIME2026}} & \multicolumn{2}{c|}{\textbf{GPQA-D}} & \multicolumn{2}{c|}{\textbf{MedQA}} & \multicolumn{2}{c|}{\textbf{Code Gen.}} & \multicolumn{2}{c}{\textbf{Avg.\ Change}} \\
\cmidrule(lr){3-4}\cmidrule(lr){5-6}\cmidrule(lr){7-8}\cmidrule(lr){9-10}\cmidrule(lr){11-12}\cmidrule(lr){13-14}\cmidrule(lr){15-16}
 & & \textbf{Light} & \textbf{Scaled} & \textbf{Light} & \textbf{Scaled} & \textbf{Light} & \textbf{Scaled} & \textbf{Light} & \textbf{Scaled} & \textbf{Light} & \textbf{Scaled} & \textbf{Light} & \textbf{Scaled} & \textbf{Light} & \textbf{Scaled} \\
\midrule
 & Acc. & 70.6 & 81.1 & 27.8 & 80.0 & 12.2 & 80.0 & 27.3 & 61.8 & 27.1 & 79.8 & 27.7 & 37.3 & \basetxt & \basetxt \\
\multirow{-2}{*}{CE only} & Token & 557 & 880 & 905 & 8272 & 1016 & 8468 & 911 & 2081 & 1177 & 767 & 477 & 1353 & \basetxt & \basetxt \\
\arrayrulecolor{restcoral}\specialrule{0.8pt}{0pt}{0pt}\arrayrulecolor{black!30}
\rowcolor{restcoral!40}\multicolumn{16}{c}{\textbf{\textsc{REST (ours), single property}}} \\
\arrayrulecolor{restcoral}\specialrule{0.8pt}{0pt}{0pt}\arrayrulecolor{black!30}
 & Acc. & 72.4 & 83.3 & 23.3 & 68.9 & 16.7 & 76.7 & 23.9 & 62.8 & 29.7 & 79.9 & 32.9 & 38.1 & \gaincell{$\uparrow 1.0$} & \losscell{$\downarrow 1.7$} \\
\multirow{-2}{*}{Causality ($\beta=0.3$)} & Token & 551 & 842 & 907 & 8335 & 1008 & 7292 & 851 & 1999 & 1113 & 739 & 570 & 1369 & \savecell{$-0.9\%$} & \savecell{$-5.7\%$} \\
\midrule
 & Acc. & 72.1 & 79.3 & 30.0 & 77.8 & 20.0 & 83.3 & 27.6 & 59.8 & 29.7 & 80.8 & 31.5 & 40.9 & \gaincell{$\uparrow 3.0$} & \gaincell{$\uparrow 0.3$} \\
\multirow{-2}{*}{Minimality ($\beta=1.0$)} & Token & 546 & 1000 & 906 & 9839 & 936 & 9602 & 872 & 2532 & 1075 & 1108 & 561 & 1797 & \savecell{$-2.9\%$} & $+18.6\%$ \\
\midrule
 & Acc. & 73.1 & 79.9 & 26.7 & 81.1 & 17.8 & 83.3 & 28.8 & 64.8 & 27.7 & 78.8 & 30.2 & 41.3 & \gaincell{$\uparrow 1.9$} & \gaincell{$\uparrow 1.5$} \\
\multirow{-2}{*}{Separability ($\beta=0.1$)} & Token & 551 & 974 & 918 & 9324 & 987 & 8872 & 901 & 2408 & 1083 & 904 & 488 & 1799 & \savecell{$-2.3\%$} & $+11.3\%$ \\
\midrule
 & Acc. & 72.1 & 79.9 & 27.8 & 72.2 & 15.6 & 77.8 & 28.1 & 63.5 & 29.8 & 80.1 & 29.1 & 37.8 & \gaincell{$\uparrow 1.6$} & \losscell{$\downarrow 1.5$} \\
\multirow{-2}{*}{Stability ($\beta=1.0$)} & Token & 538 & 894 & 921 & 8334 & 944 & 7250 & 845 & 2197 & 1051 & 804 & 645 & 1367 & \savecell{$-2.0\%$} & \savecell{$-4.5\%$} \\
\arrayrulecolor{restcoral}\specialrule{0.8pt}{0pt}{0pt}\arrayrulecolor{black!30}
\rowcolor{restcoral!40}\multicolumn{16}{c}{\textbf{\textsc{REST (ours), composition of properties}}} \\
\arrayrulecolor{restcoral}\specialrule{0.8pt}{0pt}{0pt}\arrayrulecolor{black!30}
 & Acc. & 69.7 & 80.9 & 23.3 & 81.1 & 16.7 & 81.1 & 26.1 & 59.9 & 29.6 & 82.0 & 33.5 & 40.9 & \gaincell{$\uparrow 1.0$} & \gaincell{$\uparrow 1.0$} \\
\multirow{-2}{*}{Best pair} & Token & 542 & 951 & 876 & 9160 & 1039 & 8758 & 826 & 2433 & 1083 & 1013 & 553 & 1780 & \savecell{$-2.4\%$} & $+10.4\%$ \\
\midrule
 & Acc. & 72.4 & 82.2 & 33.3 & 73.3 & 16.7 & 83.3 & 23.7 & 59.6 & 31.0 & 81.7 & 32.8 & 39.1 & \gaincell{$\uparrow 2.9$} & \losscell{$\downarrow 0.1$} \\
\multirow{-2}{*}{All properties} & Token & 569 & 923 & 934 & 8749 & 984 & 8642 & 953 & 2210 & 1130 & 833 & 555 & 1628 & $+1.6\%$ & $+5.3\%$ \\
\arrayrulecolor{restcoral}\specialrule{2.0pt}{2pt}{0pt}\arrayrulecolor{black}
\end{tabular}%
}
\end{table}

\begin{table}[h!]
\centering \small
\caption{Single-agent (round $r=3$), Light vs Scaled.}
\label{tab:single-agent-r3}
\resizebox{\linewidth}{!}{%
\arrayrulecolor{black!30}
\begin{tabular}{l|c|c>{\columncolor{tolGrey!40}}cc>{\columncolor{tolGrey!40}}cc>{\columncolor{tolGrey!40}}c|c>{\columncolor{tolGrey!40}}c|c>{\columncolor{tolGrey!40}}c|c>{\columncolor{tolGrey!40}}c|cc}
\arrayrulecolor{restcoral}\specialrule{2.0pt}{0pt}{2pt}\arrayrulecolor{black!30}
\textbf{Method} & \textbf{Metric} & \multicolumn{2}{c}{\textbf{Math500}} & \multicolumn{2}{c}{\textbf{AIME2025}} & \multicolumn{2}{c|}{\textbf{AIME2026}} & \multicolumn{2}{c|}{\textbf{GPQA-D}} & \multicolumn{2}{c|}{\textbf{MedQA}} & \multicolumn{2}{c|}{\textbf{Code Gen.}} & \multicolumn{2}{c}{\textbf{Avg.\ Change}} \\
\cmidrule(lr){3-4}\cmidrule(lr){5-6}\cmidrule(lr){7-8}\cmidrule(lr){9-10}\cmidrule(lr){11-12}\cmidrule(lr){13-14}\cmidrule(lr){15-16}
 & & \textbf{Light} & \textbf{Scaled} & \textbf{Light} & \textbf{Scaled} & \textbf{Light} & \textbf{Scaled} & \textbf{Light} & \textbf{Scaled} & \textbf{Light} & \textbf{Scaled} & \textbf{Light} & \textbf{Scaled} & \textbf{Light} & \textbf{Scaled} \\
\midrule
 & Acc. & 70.9 & 80.9 & 26.7 & 65.6 & 20.0 & 72.2 & 26.9 & 63.3 & 27.8 & 79.2 & 28.1 & 35.5 & \basetxt & \basetxt \\
\multirow{-2}{*}{CE only} & Token & 613 & 862 & 912 & 7226 & 1020 & 6967 & 941 & 1709 & 1312 & 739 & 667 & 1289 & \basetxt & \basetxt \\
\arrayrulecolor{restcoral}\specialrule{0.8pt}{0pt}{0pt}\arrayrulecolor{black!30}
\rowcolor{restcoral!40}\multicolumn{16}{c}{\textbf{\textsc{REST (ours), single property}}} \\
\arrayrulecolor{restcoral}\specialrule{0.8pt}{0pt}{0pt}\arrayrulecolor{black!30}
 & Acc. & 74.9 & 80.4 & 26.7 & 78.9 & 18.9 & 83.3 & 29.1 & 63.5 & 27.6 & 79.4 & 30.0 & 40.2 & \gaincell{$\uparrow 1.1$} & \gaincell{$\uparrow 4.8$} \\
\multirow{-2}{*}{Causality ($\beta=0.3$)} & Token & 619 & 1008 & 910 & 8677 & 952 & 8165 & 899 & 2216 & 1112 & 955 & 654 & 1662 & \savecell{$-5.8\%$} & $+20.7\%$ \\
\midrule
 & Acc. & 72.1 & 81.1 & 24.4 & 78.9 & 16.7 & 86.7 & 26.1 & 63.8 & 29.2 & 83.0 & 32.8 & 42.3 & \gaincell{$\uparrow 0.1$} & \gaincell{$\uparrow 6.5$} \\
\multirow{-2}{*}{Minimality ($\beta=1.0$)} & Token & 627 & 1053 & 939 & 10256 & 976 & 10390 & 888 & 2526 & 1236 & 1162 & 681 & 1850 & \savecell{$-2.1\%$} & $+44.9\%$ \\
\midrule
 & Acc. & 62.0 & 80.2 & 26.7 & 76.7 & 20.0 & 83.3 & 30.8 & 61.1 & 31.0 & 81.7 & 24.9 & 39.4 & \losscell{$\downarrow 0.8$} & \gaincell{$\uparrow 4.3$} \\
\multirow{-2}{*}{Separability ($\beta=0.1$)} & Token & 589 & 1056 & 1160 & 8657 & 1631 & 8033 & 1204 & 2556 & 1362 & 1087 & 1387 & 1764 & $+34.2\%$ & $+23.2\%$ \\
\midrule
 & Acc. & 73.3 & 79.1 & 26.7 & 83.3 & 12.2 & 88.3 & 24.7 & 64.4 & 31.4 & 79.8 & 34.4 & 40.4 & \gaincell{$\uparrow 0.4$} & \gaincell{$\uparrow 6.4$} \\
\multirow{-2}{*}{Stability ($\beta=1.0$)} & Token & 594 & 1049 & 902 & 9626 & 1076 & 9174 & 865 & 2464 & 1043 & 1071 & 616 & 1783 & \savecell{$-6.7\%$} & $+33.9\%$ \\
\arrayrulecolor{restcoral}\specialrule{0.8pt}{0pt}{0pt}\arrayrulecolor{black!30}
\rowcolor{restcoral!40}\multicolumn{16}{c}{\textbf{\textsc{REST (ours), composition of properties}}} \\
\arrayrulecolor{restcoral}\specialrule{0.8pt}{0pt}{0pt}\arrayrulecolor{black!30}
 & Acc. & 72.5 & 79.4 & 26.7 & 80.0 & 18.9 & 86.7 & 25.4 & 57.6 & 30.2 & 79.0 & 31.8 & 38.6 & \gaincell{$\uparrow 0.8$} & \gaincell{$\uparrow 4.1$} \\
\multirow{-2}{*}{Best pair} & Token & 618 & 1065 & 901 & 10503 & 1010 & 10728 & 850 & 2605 & 995 & 1179 & 656 & 1632 & \savecell{$-7.9\%$} & $+47.5\%$ \\
\midrule
 & Acc. & 74.9 & 78.4 & 26.7 & 73.3 & 18.9 & 76.7 & 26.6 & 60.6 & 30.0 & 79.3 & 25.7 & 37.0 & \gaincell{$\uparrow 0.4$} & \gaincell{$\uparrow 1.4$} \\
\multirow{-2}{*}{All properties} & Token & 611 & 981 & 925 & 11053 & 1030 & 10338 & 964 & 1744 & 1114 & 1001 & 671 & 1375 & \savecell{$-2.7\%$} & $+41.0\%$ \\
\arrayrulecolor{restcoral}\specialrule{2.0pt}{2pt}{0pt}\arrayrulecolor{black}
\end{tabular}%
}
\end{table}

\begin{table}[h!]
\centering \small
\caption{Multi-agent (round $r=1$), Light vs Scaled.}
\label{tab:multi-agent}
\resizebox{\linewidth}{!}{%
\arrayrulecolor{black!30}
\begin{tabular}{l|c|c>{\columncolor{tolGrey!40}}cc>{\columncolor{tolGrey!40}}cc>{\columncolor{tolGrey!40}}c|c>{\columncolor{tolGrey!40}}c|c>{\columncolor{tolGrey!40}}c|c>{\columncolor{tolGrey!40}}c|cc}
\arrayrulecolor{restcoral}\specialrule{2.0pt}{0pt}{2pt}\arrayrulecolor{black!30}
\textbf{Method} & \textbf{Metric} & \multicolumn{2}{c}{\textbf{Math500}} & \multicolumn{2}{c}{\textbf{AIME2025}} & \multicolumn{2}{c|}{\textbf{AIME2026}} & \multicolumn{2}{c|}{\textbf{GPQA-D}} & \multicolumn{2}{c|}{\textbf{MedQA}} & \multicolumn{2}{c|}{\textbf{Code Gen.}} & \multicolumn{2}{c}{\textbf{Avg.\ Change}} \\
\cmidrule(lr){3-4}\cmidrule(lr){5-6}\cmidrule(lr){7-8}\cmidrule(lr){9-10}\cmidrule(lr){11-12}\cmidrule(lr){13-14}\cmidrule(lr){15-16}
 & & \textbf{Light} & \textbf{Scaled} & \textbf{Light} & \textbf{Scaled} & \textbf{Light} & \textbf{Scaled} & \textbf{Light} & \textbf{Scaled} & \textbf{Light} & \textbf{Scaled} & \textbf{Light} & \textbf{Scaled} & \textbf{Light} & \textbf{Scaled} \\
\midrule
 & Acc. & 71.1 & 87.7 & 22.2 & 76.7 & 17.8 & 86.7 & 24.9 & 63.1 & 28.9 & 80.2 & 32.5 & 39.0 & \basetxt & \basetxt \\
\multirow{-2}{*}{CE only} & Token & 550 & 1018 & 849 & 9009 & 884 & 8488 & 749 & 2379 & 868 & 1137 & 604 & 1507 & \basetxt & \basetxt \\
\arrayrulecolor{restcoral}\specialrule{0.8pt}{0pt}{0pt}\arrayrulecolor{black!30}
\rowcolor{restcoral!40}\multicolumn{16}{c}{\textbf{\textsc{REST (ours), single property}}} \\
\arrayrulecolor{restcoral}\specialrule{0.8pt}{0pt}{0pt}\arrayrulecolor{black!30}
 & Acc. & 77.8 & 86.4 & 27.8 & 80.0 & 22.2 & 85.6 & 28.6 & 63.1 & 30.0 & 82.7 & 33.9 & 39.5 & \gaincell{$\uparrow 3.8$} & \gaincell{$\uparrow 0.7$} \\
\multirow{-2}{*}{Causality} & Token & 601 & 1041 & 894 & 9538 & 978 & 9430 & 888 & 2285 & 1049 & 1115 & 643 & 1545 & $+12.2\%$ & $+6.0\%$ \\
\midrule
 & Acc. & 77.1 & 88.0 & 30.0 & 83.3 & 22.2 & 90.0 & 27.6 & 67.2 & 26.8 & 81.3 & 32.0 & 42.3 & \gaincell{$\uparrow 3.1$} & \gaincell{$\uparrow 3.1$} \\
\multirow{-2}{*}{Minimality} & Token & 605 & 1044 & 920 & 10309 & 1052 & 9956 & 849 & 2287 & 1158 & 1120 & 575 & 2036 & $+14.6\%$ & $+13.7\%$ \\
\midrule
 & Acc. & 74.7 & 85.6 & 25.6 & 56.7 & 17.8 & 60.0 & 25.6 & 53.5 & 30.9 & 77.7 & 31.1 & 30.0 & \gaincell{$\uparrow 1.4$} & \losscell{$\downarrow 11.6$} \\
\multirow{-2}{*}{Separability} & Token & 564 & 790 & 838 & 6648 & 923 & 6180 & 758 & 1084 & 1116 & 582 & 1650 & 744 & $+29.9\%$ & \savecell{$-31.9\%$} \\
\midrule
 & Acc. & 74.0 & 79.0 & 27.8 & 40.0 & 16.7 & 46.7 & 26.8 & 52.0 & 31.0 & 77.3 & 35.1 & 31.6 & \gaincell{$\uparrow 2.3$} & \losscell{$\downarrow 17.8$} \\
\multirow{-2}{*}{Stability} & Token & 551 & 581 & 881 & 5220 & 976 & 5369 & 785 & 836 & 1205 & 445 & 781 & 882 & $+15.0\%$ & \savecell{$-43.4\%$} \\
\arrayrulecolor{restcoral}\specialrule{0.8pt}{0pt}{0pt}\arrayrulecolor{black!30}
\rowcolor{restcoral!40}\multicolumn{16}{c}{\textbf{\textsc{REST (ours), composition of properties}}} \\
\arrayrulecolor{restcoral}\specialrule{0.8pt}{0pt}{0pt}\arrayrulecolor{black!30}
 & Acc. & 77.6 & 87.2 & 25.6 & 83.3 & 15.6 & 86.7 & 29.0 & 67.2 & 30.1 & 84.3 & 33.9 & 40.0 & \gaincell{$\uparrow 2.4$} & \gaincell{$\uparrow 2.6$} \\
\multirow{-2}{*}{Best pair} & Token & 604 & 1018 & 916 & 10111 & 978 & 7694 & 886 & 2206 & 1100 & 1103 & 613 & 1495 & $+13.2\%$ & $+0.4\%$ \\
\midrule
 & Acc. & 76.2 & 86.6 & 28.9 & 73.3 & 21.1 & 86.7 & 27.6 & 63.6 & 28.7 & 80.0 & 33.3 & 33.7 & \gaincell{$\uparrow 3.1$} & \losscell{$\downarrow 1.6$} \\
\multirow{-2}{*}{All properties} & Token & 608 & 832 & 900 & 7634 & 1022 & 7466 & 879 & 2406 & 1042 & 1024 & 569 & 962 & $+11.5\%$ & \savecell{$-13.7\%$} \\
\arrayrulecolor{restcoral}\specialrule{2.0pt}{2pt}{0pt}\arrayrulecolor{black}
\end{tabular}%
}
\end{table}

\begin{table}[h!]
\centering \small
\caption{Multi-agent (round $r=3$), Light vs Scaled.}
\label{tab:multi-agent-r3}
\resizebox{\linewidth}{!}{%
\arrayrulecolor{black!30}
\begin{tabular}{l|c|c>{\columncolor{tolGrey!40}}cc>{\columncolor{tolGrey!40}}cc>{\columncolor{tolGrey!40}}c|c>{\columncolor{tolGrey!40}}c|c>{\columncolor{tolGrey!40}}c|c>{\columncolor{tolGrey!40}}c|cc}
\arrayrulecolor{restcoral}\specialrule{2.0pt}{0pt}{2pt}\arrayrulecolor{black!30}
\textbf{Method} & \textbf{Metric} & \multicolumn{2}{c}{\textbf{Math500}} & \multicolumn{2}{c}{\textbf{AIME2025}} & \multicolumn{2}{c|}{\textbf{AIME2026}} & \multicolumn{2}{c|}{\textbf{GPQA-D}} & \multicolumn{2}{c|}{\textbf{MedQA}} & \multicolumn{2}{c|}{\textbf{Code Gen.}} & \multicolumn{2}{c}{\textbf{Avg.\ Change}} \\
\cmidrule(lr){3-4}\cmidrule(lr){5-6}\cmidrule(lr){7-8}\cmidrule(lr){9-10}\cmidrule(lr){11-12}\cmidrule(lr){13-14}\cmidrule(lr){15-16}
 & & \textbf{Light} & \textbf{Scaled} & \textbf{Light} & \textbf{Scaled} & \textbf{Light} & \textbf{Scaled} & \textbf{Light} & \textbf{Scaled} & \textbf{Light} & \textbf{Scaled} & \textbf{Light} & \textbf{Scaled} & \textbf{Light} & \textbf{Scaled} \\
\midrule
 & Acc. & 70.9 & 86.6 & 22.2 & 80.0 & 15.6 & 60.0 & 27.3 & 59.1 & 28.9 & 79.7 & 30.3 & 33.5 & \basetxt & \basetxt \\
\multirow{-2}{*}{CE only} & Token & 716 & 1252 & 796 & 10485 & 891 & 5420 & 973 & 2572 & 1145 & 829 & 857 & 1029 & \basetxt & \basetxt \\
\arrayrulecolor{restcoral}\specialrule{0.8pt}{0pt}{0pt}\arrayrulecolor{black!30}
\rowcolor{restcoral!40}\multicolumn{16}{c}{\textbf{\textsc{REST (ours), single property}}} \\
\arrayrulecolor{restcoral}\specialrule{0.8pt}{0pt}{0pt}\arrayrulecolor{black!30}
 & Acc. & 75.4 & 86.8 & 30.0 & 80.0 & 14.4 & 86.7 & 28.3 & 54.9 & 28.1 & 78.7 & 34.1 & 39.1 & \gaincell{$\uparrow 2.5$} & \gaincell{$\uparrow 4.6$} \\
\multirow{-2}{*}{Causality} & Token & 792 & 1158 & 933 & 8576 & 1024 & 8278 & 1087 & 1568 & 1304 & 996 & 777 & 1924 & $+10.0\%$ & $+4.2\%$ \\
\midrule
 & Acc. & 67.2 & 85.8 & 33.3 & 86.7 & 15.0 & 83.3 & 27.3 & 62.6 & 28.3 & 83.0 & 30.4 & 39.0 & \gaincell{$\uparrow 1.0$} & \gaincell{$\uparrow 6.9$} \\
\multirow{-2}{*}{Minimality} & Token & 800 & 1204 & 1072 & 9873 & 1148 & 9558 & 1322 & 2528 & 1440 & 1275 & 889 & 1875 & $+24.0\%$ & $+21.9\%$ \\
\midrule
 & Acc. & 75.8 & 82.2 & 25.6 & 73.3 & 16.7 & 65.6 & 27.4 & 52.5 & 25.4 & 78.7 & 27.9 & 31.6 & \gaincell{$\uparrow 0.6$} & \losscell{$\downarrow 2.5$} \\
\multirow{-2}{*}{Separability} & Token & 753 & 1191 & 885 & 7039 & 980 & 6940 & 1086 & 1940 & 1374 & 1189 & 867 & 1280 & $+10.6\%$ & \savecell{$-9.3\%$} \\
\midrule
 & Acc. & 72.1 & 85.6 & 21.1 & 80.0 & 17.8 & 80.0 & 27.9 & 62.6 & 28.2 & 75.0 & 33.8 & 34.2 & \gaincell{$\uparrow 0.9$} & \gaincell{$\uparrow 3.1$} \\
\multirow{-2}{*}{Stability} & Token & 699 & 1105 & 850 & 7588 & 896 & 6497 & 958 & 1908 & 1458 & 941 & 743 & 1215 & $+4.2\%$ & \savecell{$-10.8\%$} \\
\arrayrulecolor{restcoral}\specialrule{0.8pt}{0pt}{0pt}\arrayrulecolor{black!30}
\rowcolor{restcoral!40}\multicolumn{16}{c}{\textbf{\textsc{REST (ours), composition of properties}}} \\
\arrayrulecolor{restcoral}\specialrule{0.8pt}{0pt}{0pt}\arrayrulecolor{black!30}
 & Acc. & 75.5 & 87.0 & 27.8 & 83.3 & 21.1 & 86.7 & 29.5 & 60.1 & 28.2 & 85.0 & 32.3 & 41.7 & \gaincell{$\uparrow 3.2$} & \gaincell{$\uparrow 7.5$} \\
\multirow{-2}{*}{Best pair} & Token & 794 & 1227 & 926 & 10281 & 981 & 9828 & 1039 & 2629 & 1330 & 1259 & 780 & 2134 & $+8.8\%$ & $+26.7\%$ \\
\midrule
 & Acc. & 76.1 & 86.6 & 32.2 & 78.9 & 16.7 & 86.7 & 27.8 & 63.6 & 30.9 & 81.7 & 34.0 & 42.1 & \gaincell{$\uparrow 3.7$} & \gaincell{$\uparrow 6.8$} \\
\multirow{-2}{*}{All properties} & Token & 783 & 1253 & 941 & 11033 & 1023 & 11275 & 1086 & 2673 & 1377 & 1373 & 755 & 2151 & $+10.9\%$ & $+37.9\%$ \\
\arrayrulecolor{restcoral}\specialrule{2.0pt}{2pt}{0pt}\arrayrulecolor{black}
\end{tabular}%
}
\end{table}

\subsection{Full Property Sweep}
\label{app:property-sweep}

As shown in \cref{tab:property-sweep}, the four property terms separate into two groups. Causality and minimality improve over the CE-only objective at every weight in both settings, indicating that neither term depends on a tuned value of $\beta$. Separability and stability remain close to the CE-only objective in the Light setting and fall below it in the Scaled setting at every weight, and the loss of accuracy is accompanied by a substantial reduction in decoded tokens.

\begin{table}[h!]
\centering \small
\caption{Full property sweep, multi-agent (round $r=1$), Light vs Scaled.}
\label{tab:property-sweep}
\resizebox{\linewidth}{!}{%
\arrayrulecolor{black!30}
\begin{tabular}{l|c|c>{\columncolor{tolGrey!40}}cc>{\columncolor{tolGrey!40}}cc>{\columncolor{tolGrey!40}}c|c>{\columncolor{tolGrey!40}}c|c>{\columncolor{tolGrey!40}}c|c>{\columncolor{tolGrey!40}}c|cc}
\arrayrulecolor{restcoral}\specialrule{2.0pt}{0pt}{2pt}\arrayrulecolor{black!30}
\textbf{Method} & \textbf{Metric} & \multicolumn{2}{c}{\textbf{Math500}} & \multicolumn{2}{c}{\textbf{AIME2025}} & \multicolumn{2}{c|}{\textbf{AIME2026}} & \multicolumn{2}{c|}{\textbf{GPQA-D}} & \multicolumn{2}{c|}{\textbf{MedQA}} & \multicolumn{2}{c|}{\textbf{Code Gen.}} & \multicolumn{2}{c}{\textbf{Avg.\ Change}} \\
\cmidrule(lr){3-4}\cmidrule(lr){5-6}\cmidrule(lr){7-8}\cmidrule(lr){9-10}\cmidrule(lr){11-12}\cmidrule(lr){13-14}\cmidrule(lr){15-16}
 & & \textbf{Light} & \textbf{Scaled} & \textbf{Light} & \textbf{Scaled} & \textbf{Light} & \textbf{Scaled} & \textbf{Light} & \textbf{Scaled} & \textbf{Light} & \textbf{Scaled} & \textbf{Light} & \textbf{Scaled} & \textbf{Light} & \textbf{Scaled} \\
\midrule
 & Acc. & 71.1 & 87.7 & 22.2 & 76.7 & 17.8 & 86.7 & 24.9 & 63.1 & 28.9 & 80.2 & 32.5 & 39.0 & \basetxt & \basetxt \\
\multirow{-2}{*}{CE only} & Token & 550 & 1018 & 849 & 9009 & 884 & 8488 & 749 & 2379 & 868 & 1137 & 604 & 1507 & \basetxt & \basetxt \\
\arrayrulecolor{restcoral}\specialrule{0.8pt}{0pt}{0pt}\arrayrulecolor{black!30}
\rowcolor{restcoral!40}\multicolumn{16}{c}{\textbf{\textsc{REST (ours), causality}}} \\
\arrayrulecolor{restcoral}\specialrule{0.8pt}{0pt}{0pt}\arrayrulecolor{black!30}
 & Acc. & 76.1 & 86.6 & 31.1 & 86.7 & 20.0 & 83.3 & 26.6 & 65.2 & 28.2 & 83.3 & 33.3 & 41.5 & \gaincell{$\uparrow 3.0$} & \gaincell{$\uparrow 2.2$} \\
\multirow{-2}{*}{$\beta=0.1$} & Token & 608 & 999 & 940 & 10437 & 1009 & 10161 & 941 & 2187 & 1151 & 1066 & 613 & 1882 & $+16.9\%$ & $+13.6\%$ \\
\midrule
 & Acc. & 77.8 & 86.4 & 27.8 & 80.0 & 22.2 & 85.6 & 28.6 & 63.1 & 30.0 & 82.7 & 33.9 & 39.5 & \gaincell{$\uparrow 3.8$} & \gaincell{$\uparrow 0.7$} \\
\multirow{-2}{*}{$\beta=0.3$} & Token & 601 & 1041 & 894 & 9538 & 978 & 9430 & 888 & 2285 & 1049 & 1115 & 643 & 1545 & $+12.2\%$ & $+6.0\%$ \\
\midrule
 & Acc. & 75.1 & 87.6 & 27.8 & 83.3 & 18.9 & 86.7 & 29.0 & 62.6 & 30.1 & 82.0 & 32.6 & 42.1 & \gaincell{$\uparrow 2.7$} & \gaincell{$\uparrow 1.8$} \\
\multirow{-2}{*}{$\beta=1.0$} & Token & 601 & 1034 & 880 & 10156 & 1063 & 9942 & 963 & 2089 & 1127 & 1119 & 685 & 1868 & $+18.1\%$ & $+11.3\%$ \\
\midrule
 & Acc. & 77.0 & 86.8 & 27.8 & 78.9 & 17.8 & 90.0 & 26.9 & 64.1 & 29.1 & 84.0 & 33.4 & 42.7 & \gaincell{$\uparrow 2.4$} & \gaincell{$\uparrow 2.2$} \\
\multirow{-2}{*}{$\beta=3.0$} & Token & 598 & 1022 & 915 & 9946 & 982 & 9436 & 955 & 2069 & 1106 & 1052 & 622 & 1968 & $+15.0\%$ & $+8.3\%$ \\
\arrayrulecolor{restcoral}\specialrule{0.8pt}{0pt}{0pt}\arrayrulecolor{black!30}
\rowcolor{restcoral!40}\multicolumn{16}{c}{\textbf{\textsc{REST (ours), minimality}}} \\
\arrayrulecolor{restcoral}\specialrule{0.8pt}{0pt}{0pt}\arrayrulecolor{black!30}
 & Acc. & 74.1 & 87.2 & 26.7 & 83.3 & 16.7 & 90.0 & 27.1 & 64.6 & 27.9 & 83.3 & 34.1 & 42.6 & \gaincell{$\uparrow 1.5$} & \gaincell{$\uparrow 2.9$} \\
\multirow{-2}{*}{$\beta=0.1$} & Token & 610 & 1052 & 922 & 10181 & 990 & 9858 & 865 & 2292 & 1087 & 1079 & 626 & 1695 & $+13.3\%$ & $+11.1\%$ \\
\midrule
 & Acc. & 77.1 & 88.0 & 30.0 & 83.3 & 22.2 & 90.0 & 27.6 & 67.2 & 26.8 & 81.3 & 32.0 & 42.3 & \gaincell{$\uparrow 3.1$} & \gaincell{$\uparrow 3.1$} \\
\multirow{-2}{*}{$\beta=0.3$} & Token & 605 & 1044 & 920 & 10309 & 1052 & 9956 & 849 & 2287 & 1158 & 1120 & 575 & 2036 & $+14.6\%$ & $+13.7\%$ \\
\midrule
 & Acc. & 76.1 & 87.0 & 32.2 & 86.7 & 21.1 & 84.4 & 28.5 & 65.7 & 27.3 & 83.7 & 26.6 & 43.7 & \gaincell{$\uparrow 2.4$} & \gaincell{$\uparrow 3.0$} \\
\multirow{-2}{*}{$\beta=1.0$} & Token & 596 & 1072 & 902 & 10099 & 995 & 11072 & 871 & 1996 & 1093 & 1093 & 785 & 2075 & $+16.4\%$ & $+16.4\%$ \\
\midrule
 & Acc. & 75.3 & 87.2 & 25.6 & 80.0 & 20.0 & 86.7 & 25.1 & 60.6 & 28.8 & 83.3 & 31.0 & 43.1 & \gaincell{$\uparrow 1.4$} & \gaincell{$\uparrow 1.2$} \\
\multirow{-2}{*}{$\beta=3.0$} & Token & 593 & 1053 & 906 & 10527 & 1006 & 10111 & 867 & 2257 & 1162 & 1080 & 613 & 1955 & $+14.3\%$ & $+14.6\%$ \\
\arrayrulecolor{restcoral}\specialrule{0.8pt}{0pt}{0pt}\arrayrulecolor{black!30}
\rowcolor{restcoral!40}\multicolumn{16}{c}{\textbf{\textsc{REST (ours), separability}}} \\
\arrayrulecolor{restcoral}\specialrule{0.8pt}{0pt}{0pt}\arrayrulecolor{black!30}
 & Acc. & 74.7 & 85.6 & 25.6 & 56.7 & 17.8 & 60.0 & 25.6 & 53.5 & 30.9 & 77.7 & 31.1 & 30.0 & \gaincell{$\uparrow 1.4$} & \losscell{$\downarrow 11.6$} \\
\multirow{-2}{*}{$\beta=0.1$} & Token & 564 & 790 & 838 & 6648 & 923 & 6180 & 758 & 1084 & 1116 & 582 & 1650 & 744 & $+29.9\%$ & \savecell{$-31.9\%$} \\
\midrule
 & Acc. & 73.6 & 84.8 & 24.4 & 63.3 & 16.7 & 40.0 & 26.9 & 56.1 & 27.7 & 78.3 & 28.8 & 29.3 & \gaincell{$\uparrow 0.1$} & \losscell{$\downarrow 13.6$} \\
\multirow{-2}{*}{$\beta=0.3$} & Token & 548 & 728 & 881 & 5972 & 982 & 4337 & 825 & 971 & 1169 & 598 & 1756 & 795 & $+36.8\%$ & \savecell{$-43.1\%$} \\
\midrule
 & Acc. & 71.3 & 80.2 & 26.7 & 60.0 & 16.7 & 40.0 & 24.2 & 58.6 & 26.3 & 78.3 & 32.2 & 32.1 & $0.0$ & \losscell{$\downarrow 14.0$} \\
\multirow{-2}{*}{$\beta=1.0$} & Token & 549 & 718 & 948 & 4209 & 961 & 4255 & 807 & 893 & 1168 & 567 & 790 & 1034 & $+16.0\%$ & \savecell{$-50.4\%$} \\
\midrule
 & Acc. & 72.3 & 84.6 & 26.7 & 70.0 & 17.8 & 70.0 & 29.1 & 57.1 & 29.4 & 79.0 & 29.3 & 35.0 & \gaincell{$\uparrow 1.2$} & \losscell{$\downarrow 6.3$} \\
\multirow{-2}{*}{$\beta=3.0$} & Token & 570 & 824 & 930 & 5987 & 1014 & 6944 & 907 & 1316 & 1234 & 708 & 656 & 1307 & $+17.9\%$ & \savecell{$-27.4\%$} \\
\arrayrulecolor{restcoral}\specialrule{0.8pt}{0pt}{0pt}\arrayrulecolor{black!30}
\rowcolor{restcoral!40}\multicolumn{16}{c}{\textbf{\textsc{REST (ours), stability}}} \\
\arrayrulecolor{restcoral}\specialrule{0.8pt}{0pt}{0pt}\arrayrulecolor{black!30}
 & Acc. & 70.1 & 85.8 & 22.2 & 73.3 & 17.8 & 56.7 & 28.3 & 61.6 & 26.0 & 81.0 & 32.3 & 31.8 & \losscell{$\downarrow 0.1$} & \losscell{$\downarrow 7.2$} \\
\multirow{-2}{*}{$\beta=0.1$} & Token & 495 & 983 & 773 & 7608 & 888 & 6961 & 729 & 1380 & 913 & 867 & 1219 & 797 & $+11.4\%$ & \savecell{$-21.0\%$} \\
\midrule
 & Acc. & 74.7 & 83.4 & 28.9 & 53.3 & 17.8 & 70.0 & 27.4 & 59.6 & 26.4 & 79.3 & 28.7 & 33.6 & \gaincell{$\uparrow 1.1$} & \losscell{$\downarrow 9.0$} \\
\multirow{-2}{*}{$\beta=0.3$} & Token & 586 & 806 & 904 & 6763 & 992 & 7196 & 929 & 1602 & 1247 & 509 & 1565 & 993 & $+38.2\%$ & \savecell{$-24.1\%$} \\
\midrule
 & Acc. & 74.0 & 79.0 & 27.8 & 40.0 & 16.7 & 46.7 & 26.8 & 52.0 & 31.0 & 77.3 & 35.1 & 31.6 & \gaincell{$\uparrow 2.3$} & \losscell{$\downarrow 17.8$} \\
\multirow{-2}{*}{$\beta=1.0$} & Token & 551 & 581 & 881 & 5220 & 976 & 5369 & 785 & 836 & 1205 & 445 & 781 & 882 & $+15.0\%$ & \savecell{$-43.4\%$} \\
\midrule
 & Acc. & 72.1 & 79.8 & 24.4 & 46.7 & 18.9 & 50.0 & 24.2 & 53.0 & 29.8 & 77.7 & 34.0 & 32.3 & \gaincell{$\uparrow 1.0$} & \losscell{$\downarrow 15.6$} \\
\multirow{-2}{*}{$\beta=3.0$} & Token & 546 & 586 & 884 & 4005 & 964 & 4327 & 825 & 701 & 1274 & 468 & 813 & 1092 & $+17.8\%$ & \savecell{$-52.5\%$} \\
\arrayrulecolor{restcoral}\specialrule{2.0pt}{2pt}{0pt}\arrayrulecolor{black}
\end{tabular}%
}
\end{table}

\subsection{Auxiliary-Loss Baseline Sweeps}
\label{app:baseline-sweeps}

\Cref{tab:baseline-comparison} in the main text reports each auxiliary-loss baseline at its best available weight. Both CODI and SIM-CoT were run as full weight sweeps rather than a single point. We report the full grids here.

As shown in \cref{tab:codi-beta-sweep} and \cref{tab:simcot-lambda-sweep}, neither baseline is sensitive to its auxiliary weight. Accuracy varies by less than a point across the full range of the CODI distillation weight and of the SIM-CoT step weight on the larger benchmarks, and the wider variation on the two AIME sets corresponds to one or two problems. The operating point reported in \cref{tab:baseline-comparison} is therefore representative of each method rather than an artifact of tuning.

\begin{table}[h!]
\centering \small
\caption{CODI $\mathcal{L}_{\text{KD}}$ weight sweep (round $r=1$), Light.}
\label{tab:codi-beta-sweep}
\begin{tabular}{c|c|cccccc}
\toprule
\textbf{$\beta$} & \textbf{Metric} &
\textbf{Math500} & \textbf{AIME2025} & \textbf{AIME2026} & \textbf{GPQA-D} & \textbf{MedQA} & \textbf{Code Gen.} \\
\midrule

\multirow{2}{*}{1} & Acc.  & 76.2 & 30.0 & 15.6 & 26.4 & 28.9 & 32.7 \\
                    & Token & 598 & 9179 & 10298 & 931 & 1215 & 661 \\
\midrule
\multirow{2}{*}{10} & Acc.  & 76.9 & 26.7 & 17.8 & 24.2 & 28.9 & 32.9 \\
                     & Token & 596 & 9371 & 10308 & 877 & 1158 & 699 \\
\midrule
\multirow{2}{*}{20} & Acc.  & 76.9 & 30.0 & 22.2 & 28.8 & 27.3 & 32.0 \\
                     & Token & 598 & 9328 & 9867 & 861 & 1101 & 682 \\
\bottomrule
\end{tabular}
\end{table}

\begin{table}[h!]
\centering \small
\caption{SIM-CoT $\lambda_{\text{step}}$ weight sweep (round $r=1$), Light.}
\label{tab:simcot-lambda-sweep}
\begin{tabular}{c|c|ccccc}
\toprule
\textbf{$\lambda_{\text{step}}$} & \textbf{Metric} &
\textbf{Math500} & \textbf{AIME2025} & \textbf{AIME2026} & \textbf{GPQA-D} & \textbf{MedQA} \\
\midrule

\multirow{2}{*}{0.3} & Acc.  & 76.7 & 27.8 & 21.1 & 26.9 & 28.9 \\
                      & Token & 598 & 9287 & 10042 & 914 & 1202 \\
\midrule
\multirow{2}{*}{1} & Acc.  & 74.9 & 28.9 & 18.9 & 27.9 & 27.9 \\
                    & Token & 574 & 8696 & 9377 & 866 & 1055 \\
\midrule
\multirow{2}{*}{3} & Acc.  & 76.4 & 28.3 & 18.3 & 27.0 & 29.0 \\
                    & Token & 596 & 8945 & 10209 & 902 & 1159 \\
\bottomrule
\end{tabular}
\end{table}

\subsection{Representation Analysis}

Every analysis of this subsection and of \cref{sec:analysis} inspects a representative run of each
arm rather than the full seed and weight grid of the tables above, since the quantities it reports
describe the geometry of a trained thought rather than a benchmark score.

\begin{wrapfigure}{R}{0.425\linewidth}
\centering
\vspace{-2.0\baselineskip}
\includegraphics[width=\linewidth]{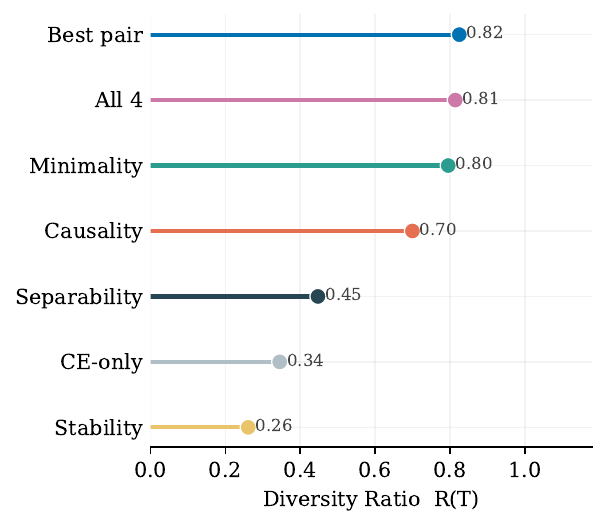}
\caption{Vector diversity within $\mathbf{T}$.}
\label{fig:collapse-diversity}
\vspace{-1.0\baselineskip}
\end{wrapfigure}

\textbf{Diversity Ratio.} We check whether $\mathbf{T}$'s own positions are diverse from one another. We measure diversity ratio
\begin{equation}\label{eq:diversity-ratio}
    R(\mathbf{T}) = \frac{\mathrm{Dist}(\mathbf{T})}{\sqrt{2}\,\mathrm{DistVC}(\mathbf{T})},
\end{equation}
where
\begin{equation}
\mathrm{Dist}(\mathbf{T}) = \frac{2}{m'(m'-1)} \sum_{i<j} \| \mathbf{T}_i - \mathbf{T}_j \|_2
\end{equation}
and
\begin{equation}
\mathrm{DistVC}(\mathbf{T}) = \frac{1}{m'} \sum_{i=1}^{m'} \| \mathbf{T}_i - \mu \|_2
\end{equation}
are the mean pairwise distance and the mean distance to the consumer's vocabulary centroid $\mu$ among $\mathbf{T}$'s $m'$ positions. The $\sqrt{2}$ normalizer follows from concentration of measure, hence $R(\mathbf{T}) = 1$ matches a thought as diverse as a random sample of real vocabulary tokens and $R(\mathbf{T}) = 0$ is total collapse. \Cref{fig:collapse-diversity} shows Minimality, Causality, and their combinations have higher ratios than the CE-only baseline, Separability is close to it, and Stability has a lower ratio. This is consistent with Minimality and Causality's losses spanning many positions of $\mathbf{T}$ while Separability and Stability reduce it to one pooled vector before applying any loss.

\textbf{Effective Superposition.} We measure how many candidate reasoning paths $\mathbf{T}$ supports at once, adapting the effective global parallelism of \citet{deng2025latentsft}.
For a question, let $x^{(1)}, \dots, x^{(M)}$ denote $M$ candidate refined plans of lengths $L_1, \dots, L_M$, namely the reference refined plan and the producer's sampled plans that lead the consumer to the correct answer and remain distinct from one another.
Let $W$ denote the consumer's unembedding matrix, and let $p_i = \mathrm{softmax}(W \mathbf{T}_i)$ denote the token distribution decoded from position $i$ of $\mathbf{T}$.
The decoded mass that $\mathbf{T}$ places on a token $x$, pooled over its $m'$ positions, is
\begin{equation}\label{eq:pooled-mass}
    \mu(x) = \sum_{i=1}^{m'} p_i[x] \, \mathbb{1}\big[ x \in \mathcal{K}(p_i) \big],
\end{equation}
where $\mathcal{K}(p_i)$ is the set of the $K$ most probable tokens under $p_i$.
Each candidate receives the coverage score
\begin{equation}\label{eq:coverage}
    c_m = \frac{1}{L_m} \sum_{j=1}^{L_m} \mu\big( x^{(m)}_j \big),
\end{equation}
which a temperatured softmax converts into a posterior over the candidates,
\begin{equation}\label{eq:candidate-posterior}
    P_m = \frac{\exp\big( \log(c_m + \varepsilon) / \tau \big)}{\sum_{k=1}^{M} \exp\big( \log(c_k + \varepsilon) / \tau \big)} .
\end{equation}
The effective superposition is the exponentiated entropy of that posterior,
\begin{equation}\label{eq:neff}
    N_{\mathrm{eff}} = \exp\Big( - \sum_{m=1}^{M} P_m \log P_m \Big) ,
\end{equation}
which equals one when a single candidate holds all of the mass and $M$ when every candidate holds the same.
We fix the same values for $K = 100$ and $\tau = 1$ as introduced in \citet{deng2025latentsft}.

\citet{deng2025latentsft} align each latent step with a segment of every candidate.
A freely generated $\mathbf{T}$ has no such correspondence.
We therefore pool the decoded mass over all positions of $\mathbf{T}$ in \cref{eq:pooled-mass}, normalize by the length of the candidate in \cref{eq:coverage}, and admit a top-$K$ token at any position.
A question on which every candidate has $c_m = 0$ yields a uniform posterior by construction, which would record an absence of evidence as maximal superposition, and such a question is excluded.

\begin{wrapfigure}{R}{0.36\linewidth}
\centering
\vspace{-2.0\baselineskip}
\includegraphics[width=\linewidth]{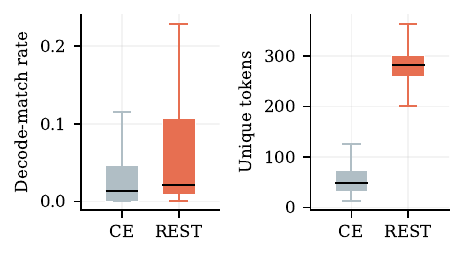}
\caption{Decoded thoughts compared to the producer's output.}
\label{fig:causality-decode-match}
\vspace{-1.0\baselineskip}
\end{wrapfigure}

\textbf{Decoded Thought Content.}
\Cref{fig:causality-decode-match} decodes each transferred thought through the consumer's vocabulary and compares the top-10 decoded tokens against the producer's output.
REST thoughts cover a wider set of distinct tokens, while CE-only concentrates on a small repeated set.
The rate at which the decoded tokens match the producer's output is also higher under REST.
This is the behaviour \cref{eq:loss-causality} targets.

\textbf{Training CE.}
As shown in \Cref{fig:ce-training}(a), CE-only memorizes the training set, while causality retains a higher training CE and a higher accuracy.
Stopping CE-only early, at the training CE where REST plateaus, does not recover the accuracy of REST in the Scaled system at $r=1$. The gain of REST is therefore not explained by avoiding memorization.
\Cref{fig:ce-training}(b) showcases the relation between training CE relative to CE-only and the change in accuracy. A lower training CE does not indicate a higher accuracy, which is the limitation of answer-level supervision in \cref{sec:price}.

\textbf{Questions Inside One Cluster.}
\Cref{fig:separability-pca-full} projects the thoughts of a fixed subsample of the training distribution, for which no held-out split exists.
At the refiner to solver transfer on math, CE-only thoughts spread along a one-dimensional curve rather than a point, although most questions on it still share a near-identical thought with another question.
We take the tightest cluster the CE-only objective produces at the planner to refiner transfer and inspect the questions behind it.
The three below belong to that cluster, shortened to their statements, and each is labelled by its index in the training dataset.

\begin{figure}[h!]
    \centering
    \includegraphics[width=\linewidth]{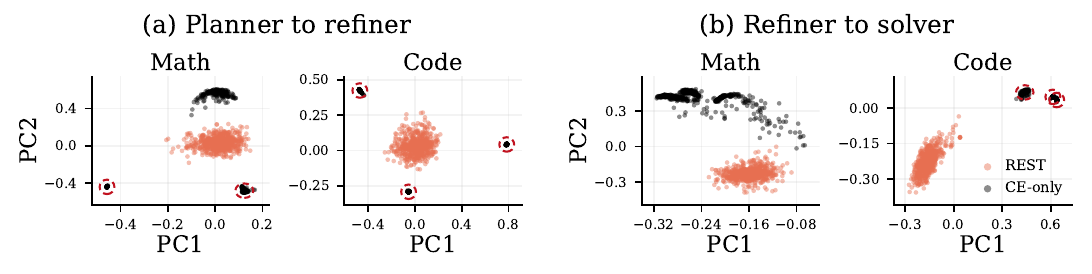}
    \caption{\textbf{PCA for $\mathbf{T}$.} CE thoughts collapse into dense clusters, where each dashed ring holds all questions of one training run, while REST spreads thoughts apart.}
    \label{fig:separability-pca-full}
\end{figure}

\begin{figure}[!h]
    \centering
    \includegraphics[width=0.85\linewidth]{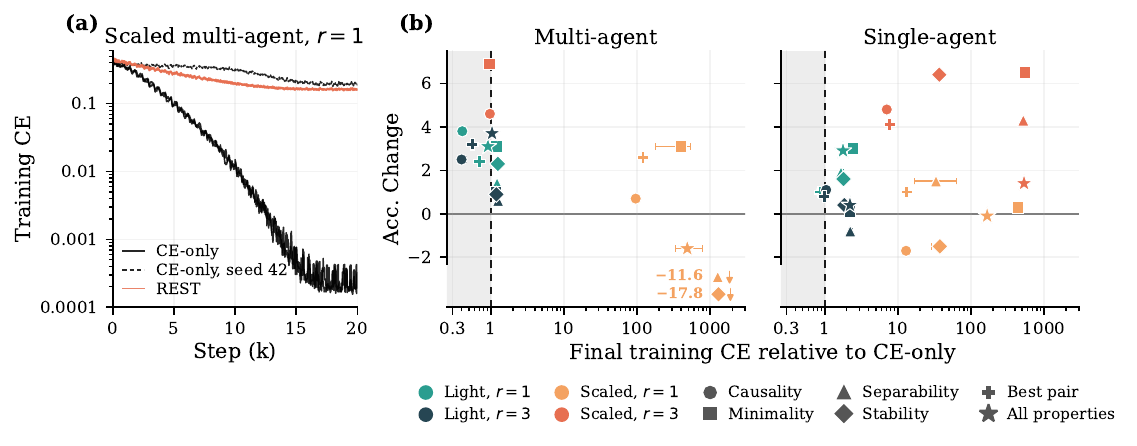}
    \caption{\textbf{Training CE.} (a) Training CE loss for causality at $\beta=3.0$. (b) Mean of last CE 200 steps relative to CE-only against Acc.\ Change. Bars span seeds.}
    \label{fig:ce-training}
\end{figure}

Their topics are a trigonometric series, refraction through a prism, and a bound state in quantum mechanics, and the cluster holds three further questions on combinatorial game theory, a recurrence relation, and a construction over the real numbers.
The questions do not share a domain nor a method of solution. Therefore, we conclude that the proximity of their thoughts represents collapse.

Under the CE-only objective the six thoughts of this cluster span $2.8\%$ of the projection, measured as the largest distance between any two of them as a fraction of the largest distance between any two of the $200$ thoughts.
The same six questions span $56.4\%$ under separability at $\beta=3.0$.
The questions grouped by the CE-only objective share no more content than unrelated ones.
Under a six-cluster partition of the projection, taken across both transfers and all three seeds, the questions inside a cluster have a mean pairwise TF-IDF cosine similarity of $0.03$, against $0.02$ for groups of the same sizes drawn at random from the same pool.

\begin{tcolorbox}[
  casebox,
  title={Three Questions From One CE-only Cluster (Planner to Refiner Transfer, seed 42)},
  left=2mm,
  right=2mm
]
\footnotesize
\textbf{Sample 1832, trigonometric series.}\\
Given that $A_k = \frac{k(k-1)}{2}\cos\frac{k(k-1)\pi}{2}$, find $|A_{19} + A_{20} + \cdots + A_{98}|$.

\vspace{4pt}\hrule\vspace{4pt}
\textbf{Sample 967, refraction through a prism.}\\
For an isosceles prism of angle $A$ and refractive index $\mu$, the angle of minimum deviation is
$\delta_m = A$. Which of the following options is or are correct? [A] For the angle of incidence
$i_1 = A$, the ray inside the prism is parallel to the base of the prism. (Three further options
omitted.)

\vspace{4pt}\hrule\vspace{4pt}
\textbf{Sample 1553, bound state in quantum mechanics.}\\
A particle of mass $m$ moves in a one-dimensional potential $V(x) = -\alpha\,\delta(x)$, where
$\delta(x)$ is the Dirac delta function and $\alpha$ is a positive constant. The particle is bound.
Find the value of $x_0$ such that the probability of finding the particle with $|x| < x_0$ is
exactly $1/2$.
\end{tcolorbox}

\begin{figure}[h!]
    \centering
    \includegraphics[width=0.7\columnwidth]{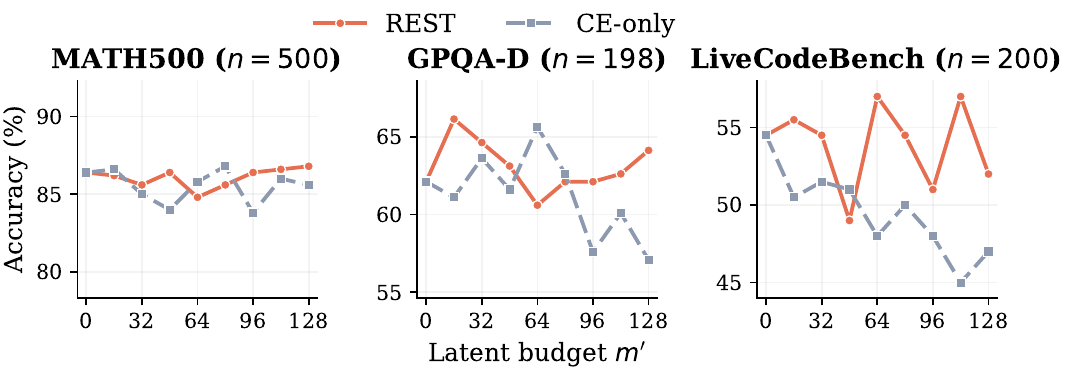}
    \caption{Accuracy against the latent budget $m'$ for the Scaled system at $r=1$.}
    \label{fig:latent-budget}
\end{figure}

\textbf{Latent Budget.}
\Cref{fig:latent-budget} varies the latent budget $m$ of the Scaled system at a fixed recursion round and compares REST against the CE-only objective.
The two objectives stay close at the smallest budgets.
As the budget grows, the CE-only objective loses accuracy on GPQA-D and on LiveCodeBench while REST holds its level, and neither objective separates on MATH500.
Each cell is a single run per arm, and the variation between adjacent budgets is therefore wider than in the tables above.

\needspace{18\baselineskip}
\section{Case Study: Effect on Generation}
\label{app:study}

\begin{wraptable}{r}{0.425\linewidth}
\centering
\small
\setlength{\tabcolsep}{4pt}
\vspace{-2.0\baselineskip}
\caption{Single-agent repeated text.}
\label{tab:generation-stalling}
\begin{tabular}{l|c|cc}
\toprule
\multirow{2}{*}{\textbf{Benchmark}} & \textbf{Answers $\geq$200} & \multicolumn{2}{c}{\textbf{Repeated text}} \\
\cmidrule(lr){3-4}
 & \textbf{tokens shorter} & CE & REST \\
\midrule
Math500  & 10--12\% & 4.5\% & \textbf{2.2\%} \\
AIME2025 & 21--23\% & 9.2\% & \textbf{3.4\%} \\
AIME2026 & 27--32\% & 8.7\% & \textbf{4.0\%} \\
GPQA-D   & 24--27\% & 28.1\% & \textbf{6.7\%} \\
MedQA    & 24--33\% & 51.5\% & \textbf{8.9\%} \\
\bottomrule
\end{tabular}
\vspace{-1.0\baselineskip}
\end{wraptable}

To study the difference in tokens when using REST compared to CE, we first analyzed the single
agent system.
Since the single agent system contained shorter outputs on average, we looked at the samples that
are shorter by at least 200 tokens.
On those answers we measured how much of each answer repeats a token span. We sweep the span width from 4 to 24 tokens and the result stays the same. We report the midpoint percentage drops across span widths in \cref{tab:generation-stalling}, averaged over the four property terms, which agree on every benchmark. The tokens REST saves are repeated text.

The two boxes below follow one of these questions in full.
The CE-only objective commits to polynomial long division and rewrites the full tableau at every
elimination step, exhausting the token budget before reaching an answer.
REST applies synthetic division and answers correctly in less than half the tokens. The shorter answer follows a different method (not a truncation of the other).

\begin{tcolorbox}[
  casebox,
  title={Output of the Single-Agent System, CE-only Objective (MATH500, sample 51, seed 42)},
  left=1.5mm,
  right=1.5mm
]
\footnotesize
\textbf{Question:}\\
Find the quotient when $x^6 - 3$ is divided by $x + 1$.
\\

\hrule
\vspace{8pt}
\textbf{Final Answer of the CE-only Objective:}\\
To find the quotient when $x^6-3$ is divided by $x+1$, we can use polynomial long division. Here are the steps:

\textbf{1. Set up the division:}
\(
\begin{array}{r|rrrrrrr} x+1 & x^6 & 0 & 0 & 0 & 0 & 0 & -3 \end{array}
\)

\textbf{2. Divide the leading term of the dividend by the leading term of the divisor:}
\(
\frac{x^6}{x} = x^5
\). Write $x^5$ above the division bar.

\textbf{3. Multiply $x^5$ by $x+1$:} $x^5 \cdot (x+1) = x^6 + x^5$. Write this product under the dividend and subtract:
\(
\begin{array}{r|rrrrrrr} x+1 & x^6 & 0 & 0 & 0 & 0 & 0 & -3 \\ & -(x^6+x^5) & & & & & & \\ \hline & 0 & -x^5 & 0 & 0 & 0 & 0 & -3 \end{array}
\)

\textbf{4. Repeat the process with the new polynomial $-x^5$:} $\frac{-x^5}{x} = -x^4$. Write $-x^4$ above the division bar.

\textbf{5. Multiply $-x^4$ by $x+1$:} $-x^4 \cdot (x+1) = -x^5 - x^4$. Write this product under the current polynomial and subtract:

\textit{[steps 6 through 8 continue the same pattern, each restating the whole tableau]}

\textbf{9. Multiply $-x^2$ by $x+1$:} $-x^2 \cdot (x+1) = -x^3 - x^2$. Write this product under the current polynomial and subtract:
\verb|\[ \begin{array}{r|
\textcolor{worsetxt}{\textit{Budget exhausted, no quotient stated}~\ding{55}}

\end{tcolorbox}

\begin{tcolorbox}[
  casebox,
  title={Output of the Single-Agent System, REST (MATH500, sample 51, seed 42)},
  left=1.5mm,
  right=1.5mm
]
\footnotesize
\textbf{Question:}\\
Find the quotient when $x^6 - 3$ is divided by $x + 1$.
\\

\hrule
\vspace{8pt}
\textbf{Final Answer of REST:}\\
To find the quotient when $x^6-3$ is divided by $x+1$, we can use polynomial long division or synthetic division. Here, I will use synthetic division for efficiency.

\textbf{1. Set up the synthetic division:}
The divisor is $x+1$, so the root is $-1$.
The coefficients of the dividend $x^6-3$ are $1, 0, 0, 0, 0, 0, -3$ (we include $0$ for the missing $x^5$, $x^4$, $x^3$, and $x^2$ terms).

\textbf{2. Perform the synthetic division:}
Bring down the first coefficient: $1$.
Multiply $1$ by $-1$ and add to the next coefficient: $0 + (-1) = -1$.
Multiply $-1$ by $-1$ and add to the next coefficient: $0 + 1 = 1$.
Multiply $1$ by $-1$ and add to the next coefficient: $0 + (-1) = -1$.
Multiply $-1$ by $-1$ and add to the next coefficient: $0 + 1 = 1$.
Multiply $1$ by $-1$ and add to the next coefficient: $0 + (-1) = -1$.
Multiply $-1$ by $-1$ and add to the next coefficient: $-3 + 1 = -2$.

The quotient is the polynomial formed by the numbers above the line, starting from the second coefficient: $x^5 - x^4 + x^3 - x^2 + x - 1$. The remainder is the last number: $-2$.

Therefore, the quotient when $x^6-3$ is divided by $x+1$ is
\(
\boxed{x^5 - x^4 + x^3 - x^2 + x - 1}
\) \textcolor{improvetxt}{\mbox{\textit{Correct Answer}~\ding{51}}}

\end{tcolorbox}

The multi-agent system shows the opposite pattern on some questions. There REST spends more
tokens than CE-only, and the extra length is genuine derivation rather than repeated text. The
boxes below follow one such question from GPQA-D. The answer options appear
in the order the run assigns, since the benchmark options are permuted per run. The CE-only
objective states an answer with no derivation. REST works through the relevant physics before answering, and only REST reaches
the correct option. The derivation carries a slight arithmetic mistake at one intermediate step, and
the answer is still correct.

\begin{tcolorbox}[
  casebox,
  title={Output of the Multi-Agent System, CE-only Objective (GPQA-D, sample 0, seed 42)},
  left=1.5mm,
  right=1.5mm
]
\footnotesize
\textbf{Question:}\\
Two quantum states with energies E1 and E2 have a lifetime of $10^{-9}$ sec and $10^{-8}$ sec,
respectively. We want to clearly distinguish these two energy levels. Which one of the following
options could be their energy difference so that they can be clearly resolved?\\
A. $10^{-8}$ eV \quad B. $10^{-11}$ eV \quad C. $10^{-4}$ eV \quad D. $10^{-9}$ eV

\hrule
\vspace{8pt}
\textbf{Final Answer of the CE-only Objective:}\\
The correct answer is \boxed{B}. \textcolor{worsetxt}{\mbox{\textit{Wrong Answer}~\ding{55}}}

\end{tcolorbox}

\begin{tcolorbox}[
  casebox,
  title={Output of the Multi-Agent System, REST (GPQA-D, sample 0, seed 42)},
  left=1.5mm,
  right=1.5mm
]
\footnotesize
\textbf{Question:}\\
Two quantum states with energies E1 and E2 have a lifetime of $10^{-9}$ sec and $10^{-8}$ sec,
respectively. We want to clearly distinguish these two energy levels. Which one of the following
options could be their energy difference so that they can be clearly resolved?\\
A. $10^{-8}$ eV \quad B. $10^{-11}$ eV \quad C. $10^{-4}$ eV \quad D. $10^{-9}$ eV

\hrule
\vspace{8pt}
\textbf{Final Answer of REST:}\\
To determine the energy difference between the two states that allows them to be clearly
resolved, we need to use the relationship between energy and temperature in statistical physics,
specifically the Fermi-golden rule and the linewidth of the energy levels.

The linewidth (width of the energy level) is given by the formula
\(
\Delta E = \hbar \frac{\pi}{\Delta t}
\),
where $\hbar$ is the reduced Planck's constant and $\Delta t$ is the lifetime of the level.

Given the lifetimes of the energy levels, $\Delta t_1 = 10^{-9}$ s and $\Delta t_2 = 10^{-8}$ s,
the corresponding linewidths are
\(
\Delta E_1 = \hbar \pi / \Delta t_1 \approx 3.30 \times 10^{-25}
\) J and
\(
\Delta E_2 = \hbar \pi / \Delta t_2 \approx 3.30 \times 10^{-24}
\) J. Converting to eV gives $\Delta E_1 \approx 2.06 \times 10^{-6}$ eV and
$\Delta E_2 \approx 2.06 \times 10^{-5}$ eV.

To clearly distinguish the two energy levels, the energy difference between the two states
should be greater than the (larger) linewidth of the levels, so greater than roughly
$2.06 \times 10^{-5}$ eV. Comparing against the options, only $10^{-4}$ eV clears this bound.

The correct option is
\(
\boxed{C}
\) \textcolor{improvetxt}{\mbox{\textit{Correct Answer}~\ding{51}}}

\end{tcolorbox}

\end{document}